\documentclass[11pt]{article}
\usepackage[utf8]{inputenc}
\usepackage[margin=1in]{geometry}
\usepackage[T1]{fontenc} %
\usepackage{graphicx}
\usepackage{mathpazo}
\usepackage{hyperref}
\hypersetup{
  colorlinks = true,  
  urlcolor = {blueGrotto},
  linkcolor = {royalBlue},
  citecolor = {navyBlue}
}
\usepackage{booktabs} 
\usepackage{multirow} 
\usepackage{multicol}

\usepackage{xcolor}
\definecolor{niceRed}{RGB}{190,38,38}
\definecolor{blueGrotto}{HTML}{059DC0}
\definecolor{royalBlue}{HTML}{057DCD}
\definecolor{navyBlue}{HTML}{0B579C}
\definecolor{limeGreen}{HTML}{81B622}
\definecolor{nicePurple}{HTML}{9c27b0}
\definecolor{lightRoyalBlue}{HTML}{def2ff}

\usepackage[
  backref=true,
  backend=biber,
  natbib=true,
  style=alphabetic,
  sorting=alphabeticlabel,
  sortcites=true,
  minbibnames=3,
  maxbibnames=999,
  mincitenames=1,
  maxcitenames=3,
  minalphanames=3,
  maxalphanames=3,
  doi=false, url=false, eprint=false, isbn=false, 
]{biblatex}

\DeclareSortingTemplate{alphabeticlabel}{
  \sort[final]{%
    \field{labelalpha} %
  }
  \sort{%
    \field{year}   %
  }
  \sort{%
    \field{title}  %
  }
}

\AtBeginRefsection{\GenRefcontextData{sorting=ynt}}
\AtEveryCite{\localrefcontext[sorting=ynt]}
\usepackage[utf8]{inputenc}

\usepackage{amsmath}
\usepackage{amssymb}
\usepackage{amsthm}
\usepackage{bbm}
\usepackage{blkarray}
\usepackage{color}
\usepackage{enumerate}
\usepackage{float}
\usepackage{subfigure}
\usepackage{tikz}
\usetikzlibrary{arrows.meta, positioning, calc, shapes, fit}
\usepackage{hyperref}
\usepackage[capitalise,noabbrev,nameinlink]{cleveref} %
\usepackage{mathtools}
\usepackage[inline]{enumitem} 
\usepackage{tcolorbox}
\usepackage{nicefrac}
\usepackage{dsfont}
\usepackage{mdframed}
\usepackage{todonotes}
\usepackage{changepage} %
\usepackage{pifont} %
\usepackage{thm-restate} %

\tikzset{
  myNodeFlex/.style={
    draw,
    rectangle,
    rounded corners,
    text centered,
    minimum height=1.5em,  
  }
}

\tikzset{
  myNode/.style={
    draw,
    rectangle,
    rounded corners,
    text centered,
    minimum height=1.5em,
    minimum width=3cm,
    text width=5cm,     
  }
}

\tikzset{
  myNodeNarrow/.style={
    draw,
    rectangle,
    rounded corners,
    text centered,
    minimum height=1.5em,
    minimum width=1cm, 
  }
}

\tikzset{
  myNodeWide/.style={
    draw,
    rectangle,
    rounded corners,
    text centered,
    minimum height=1.5em,
    minimum width=6cm, 
  }
}

\usepackage{algorithm}
\usepackage{algpseudocode}[1]

\usepackage{mathrsfs} %
\usepackage{soul} %

\theoremstyle{plain} %
\newtheorem{theorem}{Theorem}[section]
\newtheorem{corollary}[theorem]{Corollary}
\newtheorem{proposition}[theorem]{Proposition}
\newtheorem{lemma}[theorem]{Lemma}

\newtheorem{inftheorem}{Informal Theorem}

\newtheorem{definition}{Definition}
\newtheorem{infdefinition}{Informal Definition}
\newtheorem*{definition*}{Definition}

\theoremstyle{definition} %
\newtheorem{example}{Example}

\theoremstyle{remark} %

\newtheorem{remark}{Remark}

\AfterEndEnvironment{definition}{\noindent\ignorespaces}
\AfterEndEnvironment{infdefinition}{\noindent\ignorespaces}
\AfterEndEnvironment{example}{\noindent\ignorespaces}
\AfterEndEnvironment{assumption}{\noindent\ignorespaces}
\AfterEndEnvironment{lemma}{\noindent\ignorespaces}
\AfterEndEnvironment{theorem}{\noindent\ignorespaces}
\AfterEndEnvironment{proposition}{\noindent\ignorespaces}
\AfterEndEnvironment{fact}{\noindent\ignorespaces}
\AfterEndEnvironment{question}{\noindent\ignorespaces}
\AfterEndEnvironment{corollary}{\noindent\ignorespaces}
\AfterEndEnvironment{model}{\noindent\ignorespaces}
\AfterEndEnvironment{remark}{\noindent\ignorespaces}
\AfterEndEnvironment{proof}{\noindent\ignorespaces}
\AfterEndEnvironment{fact}{\noindent\ignorespaces}
\AfterEndEnvironment{minftheorem}{\noindent\ignorespaces}
\AfterEndEnvironment{inftheorem}{\noindent\ignorespaces}
\AfterEndEnvironment{maintheorem}{\noindent\ignorespaces}
\AfterEndEnvironment{restatable}{\noindent\ignorespaces}

\crefname{section}{Section}{Sections}
\crefname{theorem}{Theorem}{Theorems}
\crefname{inftheorem}{Informal Theorem}{Informal Theorems}
\crefname{assumption}{Assumption}{Assumptions}
\crefname{lemma}{Lemma}{Lemmas}
\crefname{definition}{Definition}{Definitions}
\crefname{infdefinition}{Informal Definition}{Informal Definitions}
\crefname{conjecture}{Conjecture}{Conjectures}
\crefname{corollary}{Corollary}{Corollaries}
\crefname{construction}{Construction}{Constructions}
\crefname{conjecture}{Conjecture}{Conjectures}
\crefname{claim}{Claim}{Claims}
\crefname{observation}{Observation}{Observations}
\crefname{proposition}{Proposition}{Propositions}
\crefname{fact}{Fact}{Facts}
\crefname{question}{Question}{Questions}
\crefname{problem}{Problem}{Problems}
\crefname{remark}{Remark}{Remarks}
\crefname{example}{Example}{Examples}
\crefname{equation}{Equation}{Equations}
\crefname{appendix}{Section}{Sections}
\crefname{algorithm}{Algorithm}{Algorithms}
\crefname{model}{Model}{Models}
\crefname{figure}{Figure}{Figures}
\crefname{condition}{Condition}{Conditions}

\usepackage{etoolbox}

\BeforeBeginEnvironment{subenvironment}{\white{.}\\ \vspace{-10mm}}
\AfterEndEnvironment{subenvironment}{\vspace{4mm}}

\newcommand{\Stackrel}[2]{\stackrel{\mathmakebox[\widthof{\ensuremath{#2}}]{#1}}{#2}}
\newcommand{\eat}[1]{}
\newcommand{\yesnum}{\addtocounter{equation}{1}\tag{\theequation}} 

\makeatletter
\newcommand{\tagnum}[2]{%
    \refstepcounter{equation}%
    \tag{#1) \ (\theequation}%
    \protected@write \@auxout {}{%
        \string \newlabel {#2}{{\theequation}{\thepage}{}{equation.\theequation}{}}%
    }%
}
\makeatother

\newcommand{\quadtext}[1]{\quad\text{#1}\quad}
\newcommand{\qquadtext}[1]{\qquad\text{#1}\qquad}

\newcommand{\qquadand}{\qquadtext{and}}

\newcommand{\white}[1]{\textcolor{white}{#1}}

\def\abs#1{\left| #1 \right|}

\newcommand{\inbrace}[1]{\left\{#1\right\}}

\newcommand{\inparen}[1]{\left(#1\right)}
\newcommand{\insquare}[1]{\left[#1\right]}
\newcommand{\inangle}[1]{\left\langle#1\right\rangle}

\newcommand{\N}{\mathbb{N}}
\newcommand{\R}{\mathbb{R}}
\newcommand{\Q}{\mathbb{Q}}

\newcommand{\evE}{\ensuremath{\mathscr{E}}}

\newcommand{\zo}{\ensuremath{\inbrace{0, 1}}}

\newcommand{\Ex}{\operatornamewithlimits{\mathbb{E}}}
\newcommand{\E}{\operatornamewithlimits{\mathbb{E}}}

\newcommand\ind{\mathds{1}}

\newcommand{\nfrac}[2]{\nicefrac{#1}{#2}}
\newcommand{\sfrac}[2]{{#1/#2}}

\newcommand{\eps}{\varepsilon}
\renewcommand{\epsilon}{\varepsilon}

\makeatletter
\newcommand*{\tran}{{\mathpalette\@tran{}}}
\newcommand*{\@tran}[2]{\raisebox{\depth}{$\m@th#1\intercal$}}
\makeatother

\newcommand{\wh}[1]{\widehat{#1}}

\renewcommand{\tilde}{\widetilde}

\def\<{\langle}
\def\>{\rangle}

\DeclareMathAlphabet{\mathpzc}{OT1}{pzc}{m}{it}

\newcommand{\customcal}[1]{\euscr{#1}}
\newcommand{\cA}{\customcal{A}}

\newcommand{\cC}{\customcal{C}}
\newcommand{\cD}{\customcal{D}}
\newcommand{\cE}{\customcal{E}}

\newcommand{\cG}{\customcal{G}}

\newcommand{\cI}{\customcal{I}}

\newcommand{\cK}{\customcal{K}}
\newcommand{\cL}{\customcal{L}}

\newcommand{\cP}{\customcal{P}}

\newcommand{\cT}{\customcal{T}}

\newcommand{\cV}{\customcal{V}}

\newcommand{\cX}{\customcal{X}}

\newcommand{\er}{\mathrm{er}}

\DeclareMathAlphabet{\mathdutchcal}{U}{dutchcal}{m}{n}
\SetMathAlphabet{\mathdutchcal}{bold}{U}{dutchcal}{b}{n}
\DeclareMathAlphabet{\mathdutchbcal}{U}{dutchcal}{b}{n}
 
\DeclareMathAlphabet\urwscr{U}{urwchancal}{b}{n}%
\DeclareMathAlphabet\rsfscr{U}{rsfso}{m}{n}
\DeclareMathAlphabet\euscr{U}{eus}{m}{n}
\DeclareFontEncoding{LS2}{}{}
\DeclareFontSubstitution{LS2}{stix}{m}{n}
\DeclareMathAlphabet\stixcal{LS2}{stixcal}{m} {n}

\newcommand{\itparagraph}[1]{\bigskip \noindent\textit{#1}~~}
\renewcommand{\paragraph}[1]{\bigskip \noindent\textbf{#1}~~}

\newcommand{\ie}{\textit{i.e.}}
\newcommand{\eg}{\textit{e.g.}}

\newcommand{\supp}{\operatorname{supp}}

\newcommand{\kl}[2]{\operatornamewithlimits{\mathsf{KL}}\inparen{#1\|#2}}

\newcommand{\iid}{i.i.d.}

\renewcommand{\d}{{\rm d}} 

        \usepackage{tikz}
\usetikzlibrary{patterns}

\newcommand{\algo}[1]{\mathpzc{#1}}
\newcommand{\generator}{\algo{G}}
\newcommand{\mop}{\ensuremath{\mathsf{MOP}}}

\newcommand{\eos}{\textsf{EOS}}
\newcommand{\identifierPN}{\algo{I}_{\rm PN}}

\renewcommand{\hat}{\widehat}
\renewcommand{\tilde}{\widetilde}
\renewcommand{\cE}{\evE}

\usepackage{array}
\newcolumntype{L}[1]{>{\raggedright\let\newline\\\arraybackslash\hspace{0pt}}m{#1}}
\newcolumntype{C}[1]{>{\centering\let\newline\\\arraybackslash\hspace{0pt}}m{#1}}
\newcolumntype{R}[1]{>{\raggedleft\let\newline\\\arraybackslash\hspace{0pt}}m{#1}}

\newcommand{\citecustom}[1]{\citet{#1}}

\newcommand\blfootnote[1]{%
  \begingroup
  \renewcommand\thefootnote{}\footnote{#1}%
  \addtocounter{footnote}{-1}%
  \endgroup
}

\title{ 
    \vspace{-5mm}
    On the Limits of Language Generation:\\
    Trade-Offs Between Hallucination and Mode Collapse
}

\author{    
        \vspace{-5mm}
        \begin{tabular}{C{4.8cm}C{4.8cm}C{4.8cm}}
        {\bf Alkis Kalavasis} 
            & {\bf Anay Mehrotra} 
                & {\bf Grigoris Velegkas}\\[2mm]
        {Yale University} 
            & {Yale University} 
                & {Yale University}\\[1mm]
        \mbox{\small\texttt{\href{mailto:alkis.kalavasis@yale.edu}{alkis.kalavasis@yale.edu}}} 
            & \mbox{\small\texttt{\href{mailto:anaymehrotra1@gmail.com}{anaymehrotra1@gmail.com}}}
                & \mbox{\small\texttt{\href{mailto:grigoris.velegkas@yale.edu}{grigoris.velegkas@yale.edu}}}
        \\
        \end{tabular}
}

\date{}

\begin{document}

\maketitle
\thispagestyle{empty}

\begin{abstract} 
    Specifying all desirable properties of a language model is challenging, but certain requirements seem essential for any good model.
    Given samples drawn from an unknown language, the trained model should (1) produce valid strings that have not been seen in the training data, and (2) be expressive enough to capture the full richness of the language. 
    Otherwise, if the language model outputs invalid strings, it ``hallucinates,'' and if it fails to capture the full range of the language, it suffers from ``mode collapse.'' 
    In this paper, we ask whether it is possible for a language model to meet both of these requirements.
    
    We investigate this question within a statistical setting of language generation, building on the seminal works of \citet[Inf.~Control]{gold1967language}, \citet[STOC]{angluin1979finding}, and \citet[Tech.~Report]{angluin1988identifying}.
    In this setting, the language model is presented with randomly sampled strings from a distribution supported on an unknown language $K$, which is only known to belong to a possibly infinite collection of candidate languages. 
    The goal of the model is to generate unseen strings from this target language.
    We say that the language model generates from $K$ with consistency and breadth if, as the size of the training set increases, the set of strings it can output converges to the set of all unseen strings in $K$.

    \citet[NeurIPS]{kleinberg2024language} posed an open question of whether 
    consistency and breadth in language generation {are both} possible. We answer this question negatively: for a large class of language models -- including next-token-prediction-based models -- this is impossible {for most collections of candidate languages}.
    This contrasts with the recent positive result of \citet[NeurIPS]{kleinberg2024language}, which demonstrated that consistent generation, without requiring breadth, is possible for any countable collection of candidate languages.
    Our finding highlights that generation with breadth is fundamentally different from generation without breadth. 
   
    As a byproduct of our result, we also examine how many samples are required for generation with or without breadth, establishing near-tight bounds on the ``learning curves'' for generation in the statistical framework of \citet*[STOC]{bousquet2021theory}.

    Finally, our results also give some hope for consistent generation with breadth: it is achievable for any countable collection of languages when negative examples -- in the form of strings outside of $K$ -- are available in addition to strings inside of $K$.
    This suggests that feedback in post-training, which encodes negative examples, can be crucial in reducing hallucinations while also limiting mode collapse.

    \blfootnote{Accepted for presentation at the 57th Annual ACM Symposium on Theory of Computing (STOC 2025)}

\end{abstract}

\newpage

\thispagestyle{empty}
\tableofcontents
\thispagestyle{empty}

\clearpage
\pagenumbering{arabic}

\section{Introduction} \label{sec:intro:new}

Language acquisition is a fundamental 
mystery across multiple scientific fields, ranging from Biology and Neuroscience to Sociology \citep{bresnan2007syntactic,saffran1996statistical,clark2014distributional,mahowald2024dissociating}. Theoretical Computer Scientists have been fascinated by language since the early days of the field: in the 1950s,
\citet{turing1950computing} introduced his famous test using language as an interface to cognition, 
 \citet{shannon1951prediction} studied statistics of printed English aiming at understanding its entropy and the extent to which it could be compressed, and \citet{mandelbrot1953informational} designed a statistical model to capture connections between language and the brain. 
 
Over the years, language modeling has advanced through simple models, such as the word $n$-gram model introduced by \citet{shannon1951redundancy} and widely used in natural language processing \citep{brown1992class}. In the early 2000s, neural networks achieved a significant breakthrough in the field \citep{bengio2000neural}, leading to fascinating deep learning {systems} \citep{mikolov2010recurrent,goldberg2016primer,lecun2015deep} built using traditional architectures like Recurrent Neural Networks \citep{rumelhart1986learning} and Long Short-Term Memory \citep{hochreiter1997long}.  %
In 2017, the field of language modeling was revolutionized by the introduction of the Transformer architecture \citep{sutskever2014sequence,bahdanau2014neural,vaswani2017attention}, which led to the development of Large Language Models (LLMs). The achievements of LLMs have been groundbreaking; recent models can perform well on tasks far beyond natural language processing \citep{bubeck2023sparks,touvron2023llama}. 
Despite their impressive performance, their extensive use has revealed that LLMs exhibit various bizarre behaviors even in seemingly mundane tasks \citep{borji2023categorical}.

Perhaps the most well-known issue with 
current LLMs is \emph{hallucinations}: the models generate false but plausible-sounding text with surprising frequency \citep{zhang2023siren,ji2023survey}.\footnote{We stress that LLMs outputting wrong facts based on errors in training data (\eg{}, ``The Earth is flat'') or miscalculations (\eg{}, ``1+1 = 3'') do not constitute hallucinations. A hallucination is a plausible but false text with unclear origin (\eg{}, ``Barack Obama was the president of the US and was born on January 1, 1958'').}
Such hallucinations, highlighted by popular media \citep{weise2023ai}, could significantly impact safety, reliability, and user trust as the adoption of these systems extends to new tasks \citep{hendrycks2021unsolved,amodei2016concrete}. The importance of this problem, among other concerns, led both the US \citep{biden2023executive} and the EU \citep{satariano2023eu} to issue calls for safeguards against misleading outputs generated by LLMs. %
In this direction, designing LLMs that generate responses 
\emph{consistent} with
the ground truth is an effort that has gained a lot
of attention from Machine Learning (ML) practitioners \citep{andriopoulos2023augmenting,gunasekar2023textbooks,wei2022chain,huang2023survey,feng2024don,kang2024unfamiliar,ji2023survey}, policymakers \citep{biden2023executive,satariano2023eu,satariano2023nations}, and theorists \citep{hanneke2018actively,kalai2024calibrated,kleinberg2024language}.
  
    If the sole goal is to avoid hallucinations, then, of course, one could simply limit the range of outputs generated by the language model.
    As an extreme example, consider a language model that only outputs ``I am a language model'' and, therefore, never hallucinates.
    However, modern LLMs do not just aim to generate \emph{a few} valid outputs; their goal is to obtain the ability to express \emph{a wide range} of plausible outputs, thus capturing the richness of human language.
    The key challenge lies in avoiding hallucinations while achieving \emph{breadth}.
    The problem of achieving consistent generation with breadth is not new in the ML community, dating back at least to the era of Generative Adversarial Networks (GANs) \citep{goodfellow2020generative}. In this line of work, \emph{mode collapse} \citep{goodfellow2020generative} is the analog of lack of breadth; it refers to the phenomenon where the GAN assigns non-zero mass only to a few modes of the true data distribution, thus producing a limited variety of samples and becoming repetitive \citep{arjovsky2017towards,bau2019seeing,shmelkov2018good}. \emph{The starting point of our work is exactly this puzzling tension between consistent generation and breadth in language generation}.  

We start with a mathematical specification inspired by classical work on learning theory, tracing back to the seminal work of \citet{angluin1988identifying},
and the recent formulation of \citet{kleinberg2024language}: the domain $\cX$ is
a countable collection of strings, and there is
an unknown target language 
$K$ which is a subset of this domain. 
We know that the true language lies within a collection of possibly infinite but countably many languages $\cL = \{L_1, L_2,\dots\}.$ 
    There exists an unknown distribution $\cP$ over strings in $K\in \cL$ that satisfies $\supp(\cP)=K$; any distribution with this property is said to be \emph{valid} for $K$.
The algorithm observes \iid{} samples from $\cP$ and aims to learn how to generate unseen strings from the target language $K$ -- this, at a high level, is the language generation problem. 
{Intuitively, the target language
$K$ is capturing ``facts'' of the world; everything
that belongs to $K$ is correct, 
whereas everything outside of $K$ is unambiguously incorrect and can be thought of as a ``hallucination.''
}
Observe that $K$ has to be infinite for the problem to be well-defined as, otherwise, at some point, the algorithm will see all possible strings of $K$ and, from then on, would have no unseen strings to generate from.

Let us explore language generation further, with the immediate aim of quantifying an algorithm's progress toward becoming a useful generator.
Consider a generating algorithm $\generator_n$\footnote{Formally, a generating algorithm is a sequence of mappings $(\generator_n)_{n \in \N}$: for each $n$, it is a computable mapping from a training dataset of size $n$ to a (computable) distribution (\ie{}, a sampling algorithm) over $\cX.$ We will use the notation $(\generator_n)_n$ to refer to the generating algorithm and the notation $\generator_n$ or simply $\generator$ for the induced distribution (\emph{generator}) after training; hence when we write $x \sim \generator_n$ or $\supp(\generator_n)$, we refer to the distribution obtained after training.} 
that is trained on a set $S$ of $n$ \iid{} examples from $\cP$.
To quantify the inconsistency of $\generator_n$, 
    we need an objective.
    As discussed above, this objective should penalize $\generator_n$ for outputting strings outside of $K$ and for repeating examples already seen in the training data $S$.\footnote{{When we require generating algorithm to achieve breadth, it is not important to enforce that the support does not contain $S$. We will elaborate after the formal statement of \Cref{def:breadth}.}}
    For a target language $K$ and a model $\generator_n$ trained on $S$, we consider the following generation error
\begin{equation}
   \mathrm{gen\_er}(\generator_n) 
   \coloneqq
    \Pr_{S \sim \cP^n}[
        \supp(\generator_n)\not\subseteq K\setminus S
    ]
   \,.
   \label{eq:gener}
\end{equation}
In words, a model errs according to $\mathrm{gen\_er}(\cdot)$ if it either hallucinates by outputting strings from $\cX \setminus K$ or if it outputs something already contained in the training set $S$. 
    This is inspired by the notion of generation considered by \citet{kleinberg2024language}; they call an algorithm a \textit{consistent} generator if its support becomes a subset of $K\setminus S$ after seeing finitely many training examples $S$.
    We relax this definition and call an algorithm a \emph{consistent generator for the collection $\cL$} if its error, as defined in \cref{eq:gener},  asymptotically goes to zero  for any valid distribution $\cP$.

{Let us now review how prior work has approached issues with language generation algorithms -- foremost, hallucination.}
Under the above {statistical} setting, \citet{kalai2024calibrated} made important progress
showing that {calibrated models must hallucinate by lower bounding the hallucination rate by the model's calibration. For a detailed comparison with our work, we refer to \Cref{sec:relatedWork}. Closer to our paper, the work of} 
\citet{kleinberg2024language} explored language generators that must not hallucinate, \ie{}, they must be consistent.
    They studied language generation in an online setting where the data are not drawn from $\cP$ but are given as a stream to the learner, \ie{}, as an \emph{adversarial enumeration} of the strings of the true language $K$.
In their setting, $\generator_n$ {is said to} \emph{generate in the limit} from $K$ if, after some finite time $n_0$ in the enumeration of $K$, $\generator_n$ is able to generate new unseen strings from $K$ for 
all subsequent times $n \geq n_0$.
They showed that there exists an algorithm that can generate in the limit from \textit{every} countable list of candidate languages.

This result is surprising because it contrasts with strong negative results  for the well-studied problem of \emph{language identification in the limit}
(where one wants to identify $K$ in the limit and not simply generate from it;\footnote{
Very briefly,  
a language collection $\cL = \{L_1, L_2,\dots\}$ is called identifiable in the limit if
there exists an algorithm $(\cA_n \colon \cX^n \to \N)_n$  such that
for any $K \in \cL$ and any enumeration $x_1, x_2,\dots$ of the strings of $K$ appearing as a stream to $(\cA_n)$, there is a finite time $n_0 \in \N$ after which 
the algorithm predicts the correct index of the true language, \ie{},
$L_{\cA_n(x_1,\dots,x_n)} = K$ for any $n \geq n_0$.
} see also \Cref{def:Identification}).
The family of languages identifiable in the limit is very limited: the results of \citet{gold1967language,angluin1979finding} showed that language identification is a very difficult problem and most collections of languages are non-identifiable (in fact, there is a {tight} characterization due to \citet{angluin1980inductive} which we state in \Cref{def:angliun-criterion}). Hence, the algorithm of \citet{kleinberg2024language} shows that language generation in the limit is much more tractable than identification. We note that while their algorithm operates in a non-statistical setting, it will be an important building block for our results. 

\citet{kleinberg2024language} observed that their algorithm eventually becomes a consistent generator but \emph{suffers from mode collapse}: initially, it generates with breadth while being inconsistent with the target language; later on, as a larger part of the stream is seen, it starts sacrificing breadth in order to generate valid outputs. 
This behavior led them to leave the existence of a consistent generator that achieves breadth as an interesting open question. 
In this work, we will formally introduce {a notion of} breadth for language generation in our statistical setting (\Cref{sec:setup}). For now, we mention that our definition roots in the notion of mode collapse from Generative Adversarial Networks (GANs) \citep{goodfellow2020generative,arjovsky2017towards} and, roughly speaking, states that an algorithm $(\generator_n)$ generates with breadth from $K$ if the probability that its support contains all the unseen examples from the target language goes to 1, as the training samples from a valid distribution go to infinity. 
Now it is a good point to contrast breadth with consistency: consistent generators {aim at avoiding any elements outside of $K$} while {generators achieving breadth try to cover all unseen elements of $K$}. The question of \citet{kleinberg2024language} is asking whether the equilibrium condition that the support of the generator \emph{exactly} matches the unseen elements of $K$ can eventually be achieved by some algorithm. This is the main question we aim to address in this paper. 
\begin{center}
    \emph{Is it possible to achieve consistent language generation with breadth or\\
    is there some inherent trade-off between consistency and breadth?}
\end{center}
    \subsection{Informal Results}\label{sec:intro:results}
    Our main results confirm the tension between consistent generation and breadth for language models, conjectured by \citet{kleinberg2024language}, in a strong way: \emph{informally}, we show that
    \begin{center}
        \emph{A language model that generates with breadth must be inconsistent, \ie{}, it must hallucinate.}
    \end{center}
    We focus on the probabilistic setting of \citet{angluin1988identifying} which we have already introduced informally.
    En route to our results in the probabilistic setting, we also obtain results in the online setting of \citet{gold1967language}, \citet{angluin1979finding}, and \citet{kleinberg2024language}, as we will see later.
    To facilitate a formal discussion of our contributions, we need to introduce some further definitions.

    \subsubsection{Setup and Definitions}\label{sec:setup}
    A generating (or learning) algorithm is a sequence of computable mappings $(\generator_n) = (\generator_n)_{n \in \N}$ from samples $S \subseteq \cX^n$ to generators, which are simply distributions over the domain $\cX$. More formally, a generating algorithm is a sequence of mappings from samples to Turing machines that generate samples from an (explicitly or implicitly) defined distribution over strings.

    In the statistical setting we consider, the learner observes samples from an unknown distribution which is valid for some unknown language $K$ in the collection $\cL = \{L_1, L_2, \dots\}$.
    \begin{definition}[Valid Distribution \citep{angluin1988identifying}]\label{def:valid-language-distribution}
    A distribution $\cP$ over a countable domain $\cX$ is valid with respect to a countable language collection $\cL$ if its support is the same as some language
     $K \in \cL.$
     In this case, when we want to be specific about the language that $\cP$ draws samples from, we say $\cP$ is valid for $K$.
\end{definition}
If the collection $\cL$ is clear from context,
we will simply say that $\cP$ is valid.
Based on this definition and building on the model studied by \citet{kleinberg2024language}, we give the following adaptation for consistent generation from a collection $\cL$ in the statistical setting.
\begin{definition}
[Consistency]\label{def:consistency:intro}
A generating algorithm $(\generator_n)$ for a language collection $\cL$ is consistent if for any valid distribution $\cP$, it holds that $\lim_{n \to \infty} \mathrm{gen\_er}(\generator_n) = 0$.
Otherwise, the algorithm is said to be inconsistent.
\end{definition}
Hence, an algorithm is said to be consistent if the generators it produces by training on any valid distribution $\cP$ converge to generating examples from the unseen part of $\cP.$
Some of our results explore when asymptotic consistency is achievable. However, the main focus of our work is on understanding the \emph{rates} at which consistency (and other desirable properties) can be attained -- if possible at all.
In particular, we want to study the rate at which the generation error $\mathrm{gen\_er}(\generator_n)$ decreases as the number of samples $n$ goes to infinity -- that is, we want to study the \textit{learning curve} of consistent generation (and other tasks that we introduce later in this section). Bousquet, Hanneke, Moran, van Handel, and Yehudayoff \citep{bousquet2021theory} characterized learning curves for binary classification, formalizing the \emph{universal rates} framework, earlier
explored by \citet{schuurmans1997characterizing} and \citet{antos1996strong}. To this end, we borrow their definition of universal rates.

\begin{definition}
[Informal, Universal Rates; \citep{bousquet2021theory}, see \Cref{def:achievable-rates}]
A generating algorithm $(\generator_n)$ has rate $R(\cdot)$, where $\lim_{n\rightarrow \infty} R(n) = 0$, for a language collection $\cL$ if
    \[
    \forall \cP \in \mathrm{Val}(\cL)~~ \exists C ,c  > 0 \quadtext{such that}
    \mathrm{gen\_er}(\generator_n) \leq C \cdot R(c \cdot n)\quad \forall n \in \N \,,
    \]
    where $\mathrm{Val}(\cL)$ is the class of valid (realizable) distributions for $\cL$. 
    \label{def:universal}
\end{definition}
Observe that these learning curves are distribution-dependent since the constants $c$ and $C$ are allowed to depend on $\cP$.
This difference turns out to be crucial and can, sometimes, lead to significant differences between universal rates and the corresponding distribution-independent rates \citep{bousquet2021theory}.
Among different universal rates, exponential universal rates are of specific interest as they are often the best possible rate, as we will see later.
We say that the algorithm $(\generator_n)$ generates with an exponential universal rate if $R(n) = \exp(-n)$ in the above definition. Next, we turn to language generation with breadth.

\begin{definition}
[Breadth]
A generating algorithm $(\generator_n)$ for a language collection $\cL$ is said to achieve breadth if, 
for any 
valid distribution $\cP,$ it holds that
$\lim_{n \to \infty} \Pr[\supp(\generator_n) \supseteq K \setminus S_n] = 1$, where 
$S_n$ is the dataset used to train $\generator_n,$ \iid{} from $\cP.$
Otherwise, the algorithm suffers from mode collapse. 
\label{def:breadth}
\end{definition}
\Cref{def:breadth} is inspired by the literature on GANs (see \eg{}, \citep{goodfellow2020generative,arjovsky2017towards}). For instance, consider the work of \citet{arjovsky2017towards}, which studies distributions $\generator$ and $\cP$ induced by the generator and nature, respectively, and says that mode collapse occurs when the KL divergence $\kl{\cP}{\generator} \coloneqq \int \log\inparen{\nfrac{\cP(x)}{\generator(x)}}~ \d \cP(x) \to \infty$.
In particular, mode collapse happens when there is some string $x \in \supp(\cP)$ for which $\generator(x) = 0$. 
In other words, the generator has breadth when $\supp(\generator) \cup S_n\supseteq \supp(\cP)$, which recovers our definition for breadth by noting that $\supp(\cP)=K$ since $\cP$ is valid for $K$ and that, to be compatible with the definition of consistency (\cref{def:consistency:intro}), we bar a generator from repeating strings it has already seen.
(It is worth mentioning that one can modify the definition of breadth to require $\supp(\generator_n)\supseteq K$ without changing any of our results; see \cref{rem:generatingFromTrainingSet}.)
{We also note that the definition of consistency we use can also be derived in an analogous fashion by requiring the reverse KL divergence (\ie{}, $\kl{\generator}{\cP}$) to be finite.}

Putting the definitions of consistency and breadth together implies that an algorithm generates with consistency and breadth if, eventually, its support \emph{matches} the set of unseen strings in $K$, \ie{},  $K\setminus S_n$ at the $n$-th step.
After presenting our main results, in \cref{sec:intro:relaxation}, we discuss relaxations of this notion of consistent generation with breadth.

A last ingredient for our results concerns the decidability of a folklore Theoretical Computer Science problem, which we call the \emph{membership oracle problem}, that has motivated extensive work in formal languages and complexity theory \citep{sipser2012introduction,soare1999recursively}. A generator $\generator$, which is the output of some generating algorithm,
corresponds to some Turing machine, as is standard in the language inference literature,
that samples according to a distribution
over $\cX$
\citep{angluin1979finding,blum1975toward,angluin1983inductive,adleman1991inductive}.

\begin{restatable}[Membership Oracle Problem]{definition}{defMOP}\label{def:membership access}\label{def:mop}
    Given a generator $\generator$,
    the membership oracle problem for $\generator$, denoted as $\mathsf{MOP}(\generator)$, is defined as follows: given the description of $\generator$ and a string $x$, output $\textsf{Yes}$ if $x\in \supp(\generator)$ and output $\textsf{No}$ otherwise.
\end{restatable}
This problem is, in general, undecidable due to a reduction to the halting problem (\cref{sec:undeciable}); nevertheless, its decidability depends on the structure of the Turing machine as we will see shortly. The above definition naturally extends to generating algorithms.

\begin{definition}
[MOP for Generating Algorithms]
\label{def:mopAlgo}
The membership oracle problem is decidable for a generating algorithm $(\generator_n)$ if, for any $n \in \N$ and any $S \subseteq \cX^n$,
$\mop(\cdot)$ is decidable for the induced generator $\generator = \generator_n(S)$.
\end{definition}
We note that the above definitions implicitly assume that the generator $\generator_n(S)$ depends only on the randomness of $S$; we could extend this by allowing $\generator_n(S)$ to be a distribution over generators.

\subsubsection{Main Results}
We now have all the ingredients to state our first result, which establishes that, for all generating algorithms for which $\mop(\cdot)$ is decidable, (consistent) generation with breadth is as hard as language identification in the statistical setting. 

As in \Cref{def:universal}, we will say that the generating algorithm $(\generator_n)$ generates with breadth from $\cL$ at some rate $R(\cdot)$ if, for any $K \in \cL$, valid distribution $\cP$, and $n \in \N,$
\[
\E_{S \sim \cP^n} \ind
\inbrace{\supp(\generator_n) \neq K \setminus S} \leq C \cdot R(c \cdot n)\,,
\]
 for some distribution-dependent constants $C,c > 0.$ If no rate $R(\cdot)$ satisfying $\lim_{n \to \infty} R(n) = 0$ exists, we will say that $(\generator_n)$ does not generate with breadth \emph{at any rate}.

    \begin{inftheorem}[see \cref{mainthm:gen:mop}]
   \label{infthm:breadth:mop}
            For every language collection $\cL$ that is not identifiable in the limit, no generating algorithm $(\generator_n)$, for which $\mop{}(\cdot)$ is decidable, can generate from $\cL$ with breadth at any rate.
        \end{inftheorem}
        Recall that the family of languages non-identifiable in the limit is quite broad.
        Based on the results of \citet{gold1967language,angluin1979finding,angluin1980inductive} on the problem of language identification in the limit, our impossibility result holds for most interesting collections of languages. For \Cref{infthm:breadth:mop} to be valuable and meaningful though, we further need to show that there exists an algorithm that generates without breadth for the collections of languages for which our impossibility result is true. Our next result states that this is indeed possible: there exists an algorithm that generates with 
        (almost)
        exponential universal rates for \emph{any} countable language collection $\cL$.
    \begin{inftheorem}
    [see \Cref{mainthm:gen:mop}]
    \label{infthm:generation}
    For every language collection $\cL$ that is not identifiable in the limit, there exists a generating algorithm $(\generator_n),$ for which $\mop(\cdot)$ is decidable, that generates (possibly) without breadth from $\cL$ at exponential rates. Further, if $\cL$ is identifiable in the limit, then there exists a generating algorithm $(\generator_n)$, for which $\mop(\cdot)$ is decidable, that generates with breadth from $\cL$ at (almost) exponential rates.
    \end{inftheorem}
    \Cref{infthm:generation} shows that \emph{any} countable collection of languages not only admits a consistent generator in the limit under an adversarial enumeration of the target language (as shown by \citet{kleinberg2024language}), but the statistical rate at which consistency (as per {\cref{def:consistency:intro}}) is achieved is exponential in the number of samples. Further, for identifiable collections of languages, we give an algorithm that generates with breadth at an (almost) exponential rate.
    
    The combination of \Cref{infthm:breadth:mop} and \ref{infthm:generation}, reveals a strong separation between generation with and without breadth for any generating algorithm for which $\mop(\cdot)$ is decidable. 
    What is missing is an answer to:
    how large is the class of generators for which the membership oracle problem $\mop(\cdot)$ is decidable?
    It turns out there is a very broad class of language generators for which this is the case and which also captures modern LLMs, as we show next. 

    \paragraph{A Family of Generators for which $\mop{}(\cdot)$ Is Decidable.}
    Motivated by the structure of modern language models \citep{bahl1983maximum,brown1990statistical,touvron2023llama,bubeck2023sparks,achiam2023gpt}, we consider a family of \textit{iterative} generators.
    A generator is said to be iterative if it generates text one alphabet or ``token'' at a time (see \Cref{def:tokenByToken}).
    To generate each token, the generator can perform an arbitrary (but finite) amount of computation and, possibly, use randomness.
    {For this to make sense, one has to imagine strings of $\cX$ as strings over some finite alphabet $\Sigma$.
    This holds without loss of generality as $\cX$ is countably infinite and, hence, there is a one-to-one mapping from $\cX$ to strings over $\Sigma$ (due to which $\cX$ can be thought of as a set of strings over $\Sigma$).}\footnote{
        {In a bit more detail, since $\cX$ and $\Sigma^*$ are countably infinite, they have enumerations $x_1,x_2,\dots$ and $s_1,s_2,\dots$.
        Therefore, given any string $s_i\in\Sigma^*$ generated by an iterative generator, one can map it to a string $x_i\in\cX$, thereby getting a generator for $\cX$.}
    }   
    We show that for any iterative generator, the membership oracle problem is decidable and our \Cref{infthm:breadth:mop} is applicable.
    \begin{inftheorem}
    [see \Cref{prop:mop_decidable}]\label{infthm:mop_decidable}
    For any iterative generator $\generator$, $\mop(\generator)$ 
    is decidable.
    \end{inftheorem}
    Observe that this family of next-token generators is very general.
    First, it captures existing large language models:
    for instance, to simulate an LLM $L$, we define the next-token predictor as a Turing machine that simulates $L$ on the provided string until $L$ generates one new token.
    Next, it also captures systems where an LLM can interact with another Generative AI model or algorithmic system (such as a diffusion model or a code interpreter) -- as these auxiliary systems can also be simulated by the generator.
    Given this, it becomes evident that this class of generators for which $\mop{}(\cdot)$ is decidable is fairly large and interesting.

    \paragraph{Implications for the Gold--Angluin Model.}
    We repeat that all the aforementioned results hold in the {statistical} setting. En route to obtaining our results in this  setting (\cref{infthm:breadth:mop,infthm:generation}), we show several connections to the online setting of \citet{gold1967language, angluin1979finding,angluin1980inductive,kleinberg2024language}, which lead to the following result.

        \begin{inftheorem}
        [see \Cref{mainthm:gen:limit}]\label{infthm:gen:limit}
         For any language collection $\cL$ that is not identifiable in the limit, no generating algorithm $(\generator_n)$, for which $\mop{}(\cdot)$ is decidable, can generate from $\cL$ with breadth in the limit.
         \label{infthm:inthelimit}
        \end{inftheorem}
        To be more concrete, a generating algorithm generates with breadth in the limit 
        if its support is eventually $K \setminus S_n$, where $S_n$ is the set of the first $n$ positive examples (\ie{}, examples that belong to $K$).
        We emphasize that \cref{infthm:gen:limit} is in a similar spirit as our \Cref{infthm:breadth:mop}, but holds in the online model instead of the statistical model discussed earlier.
        In particular, \Cref{infthm:inthelimit} combined with the algorithm of \citet{kleinberg2024language} give a separation between consistent generation with and without breadth in the Gold--Angluin model.
        Further, as explained before, this result applies to any iterative generator due to \cref{infthm:mop_decidable}.
        Moreover, as $\mop{}(\cdot)$ is decidable for the generating algorithm of \citet{kleinberg2024language} (since its support contains a singleton element $x$ which can be computed by running their algorithm), the above result, in particular, shows that the algorithm of \citet{kleinberg2024language} cannot generate with breadth in the limit from any non-identifiable collection. 

    \paragraph{Results with Two Relaxations of Breadth.}
    Next, we consider two relaxations of the breadth requirement (\cref{def:breadth}), which we term \emph{unambiguous generation} (\cref{def:unambiguousGen}) and generation with \emph{approximate breadth} (\cref{def:gen:oneSidedRobust}). Informally, the former requires that the generator's output is ``closer'' to the target language $K$ -- measured by symmetric difference -- than to any other language in the collection $\cL$. The latter requires that the generator's output is contained in $K$ and differs from it on only \emph{finitely} many elements.
While both definitions relax exact breadth, they appear to be incomparable. Notably, unambiguous generation permits hallucinations -- that is, incorrect outputs not in $K$ -- as long as $K$ remains the closest language in $\cL$. For both notions, we study generators for which \mop{} is decidable and which satisfy a \emph{stability} condition: for any target language $K$ and enumeration thereof, the generator eventually stabilizes and ceases to change its output.
Perhaps surprisingly, we show that generation under either relaxed notion is no easier than generation with exact breadth: if $\cL$ is not identifiable in the limit, then no such stable and decidable generator can achieve unambiguous generation or approximate breadth for $\cL$.

    \paragraph{Organization of Rest of the Introduction.} We proceed with an exposition of our techniques in order to obtain our main results presented above. In \cref{sec:intro:relaxation}, we relax the definitions of consistency and breadth and give more ``robust'' trade-offs between hallucination and breadth. Next, in \cref{sec:open}, we give a list of open problems for future work. Finally, \Cref{sec:relatedWork} contains an extensive overview of related works.  
    
    \subsection{Technical Overview}\label{sec:technical-overview}

    In this section, we present the technical tools we develop to obtain our main results.
    
    \paragraph{A Natural Strategy to Prove \Cref{infthm:breadth:mop}.} At first glance, there seems to be a natural strategy to prove \Cref{infthm:breadth:mop}:
    assume that there exists a consistent generating algorithm with breadth $\generator = (\generator_n)$ for some non-identifiable collection $\cL$ in the statistical setting and then show that this implies identification in the statistical setting, which would contradict the fact that $\cL$ is non-identifiable.
    To implement this strategy one needs a method to utilize $\generator$, along with the positive samples from the target language $K$, for identification.
    This raises the question: what additional power can $\generator$ give that the positive samples do not \textit{already} provide?

    \paragraph{Initial Attempts to Implement the Strategy.}
    Indeed, if one uses \textit{no} additional properties of $\generator$, then its outputs provide no more information than an adversarial enumeration of $K$.
    To develop some intuition, we begin by considering some properties of the generator and explaining why they are insufficient to enable identification.
    
    \smallskip 
   \noindent  \textit{1. $\generator$ is non-adaptive.}~~ First, one may want to utilize the fact that the generator $\generator{}$ is fixed and, hence, the samples it outputs cannot adapt to the specific algorithm being used based on the outputs of the algorithm.
        Hence, it will probably provide an algorithm-independent enumeration of the true language. However, this is not helpful in general since there exist simple non-identifiable language collections that remain non-identifiable for many enumerations of the target language.

        \smallskip 
        \noindent \textit{2. $\generator$ samples from a fixed distribution.}~~ Another property one may want to leverage is the stochasticity of the generator: $\generator$ samples its outputs from a fixed distribution (which is valid for $K$).
        However, even this does not enable the identification of non-identifiable collections due to a result by \citet{angluin1988identifying}. Angluin shows that even if the positive examples are \iid{} from a valid distribution and do not appear as an adversarial enumeration (as in \citet{gold1967language}), this does not enable identification of any collection $\cL$ that is non-identifiable in the limit. (We prove a stronger version of this result in \Cref{lem:almost-exp-rates-ident-pos-examples}.)
        
        \smallskip 
        \noindent \textit{3. $\generator$ samples from a simple distribution.}~~ Moreover, the difficulty in the above negative result is not the complexity of the encoded distribution: it holds even when $\generator$ samples from a distribution that is computable by a Turing machine.

        \medskip
        
    \noindent At this point, it is not clear how to utilize access to a generator $\generator$ which generates with breadth from $K$.
    Next, we present a strong form of {access} to the generator $\generator$ that is useful for identification. 

    \smallskip
    \noindent
    \textit{4. Access to Subset Queries ``$\supp(\generator)\subseteq L_i$'' and ``$L_i\subseteq L_j$''.}~~
        For one of their algorithms, \citet{kleinberg2024language} utilize a subset oracle that answers queries of the form ``is $L_i\subseteq L_j?$''.
        (In general, this oracle is not guaranteed to be computable.)
     One can imagine an extension of this oracle that, given an index $i$ and description of the generator $\generator$, outputs whether $\supp(G)\subseteq L_i$.
        The existence of this oracle turns out to be sufficient to identify $K$, as we explain next:
        After a finite amount of time, \citet{kleinberg2024language}'s algorithm creates a list of ``critical'' languages $C_1,C_2,\dots,$ of the following form (see Theorem 4.1 in \citet{kleinberg2024language})
        \[
            C_1\supseteq C_2 \supseteq \dots \supseteq \inparen{C_i\coloneqq K} \supseteq C_{i+1}\supseteq \dots \,.
        \]
        In words, this list has two properties  (1) $K$ appears in this list, say, at $C_i=K$ for some $i<\infty$ and (2) each language $C_j$ in the list is a  subset of the preceding language $C_{j-1}$.
        Given this list and the aforementioned subset oracle, one can easily identify the index of $K$ as the largest $j$ for which $\supp(\generator)=K\subseteq C_j$.
        This assumption allows to identify \textit{any} collection in the limit given access to a consistent generation $\generator$ with breadth.
        However, this type of access is not very practical since it is not clear when such an oracle is implementable.

    \paragraph{Our Approach.}
    Our first idea is that a much weaker form of access to $\generator$ -- \emph{membership oracle to $\supp(\generator)$} -- is sufficient for identification. 
    This is where the membership oracle problem $\mop{}(\cdot)$ (\Cref{def:mop}) appears in the proof. 
    In fact, given this idea, it is not difficult to show that with that type of access, we can go from a generator with breadth in the online setting to an identification algorithm in the \textit{online setting}; and, hence, get \Cref{infthm:inthelimit}. 
    However, our focus is the statistical setting where there are several additional challenges in using the membership oracle to $\supp(\generator).$
    
    \itparagraph{A. Need for Universal Rates for Generation and Identification.}
        The key issue is the following:
        In the statistical setting, if we assume that we have a generator with breadth \emph{at rate $R(\cdot)$}, then we can hope to show an implication that we can get an identification algorithm \emph{at rate $R(\cdot)$}. 
        However, this need not imply a contradiction to the identifiability of $\cL$ in the online setting.
        This is because, even though $\cL$ is non-identifiable in the online setting, it may become identifiable at \emph{some} rate $R'(\cdot)$ in the statistical setting.
        Indeed, this is the case in binary classification, where there are simple hypothesis classes (such as thresholds over reals) that are not learnable in Littlestone's online setting \citep{littlestone1988learning} but become learnable (at a universal -- and uniform -- linear rate) in the statistical setting; in fact, \textit{any} hypothesis class is learnable in the statistical setting under universal rates, since there is a Bayes consistent algorithm, under benign assumptions \citep{bousquet2021theory}. 
       
    Hence, to get a contradiction, we first need to understand the taxonomy of universal rates for generation and identification.
    We remark here that both the learning task (\eg{}, classification, regression, identification, and generation) and loss function used in the problem are pivotal for the landscape of rates that one gets; for instance, with the zero-one loss for binary classification one gets a trichotomy of rates \citep{bousquet2021theory}, but with the $L_1$-loss for regression, one gets infinitely many rates \citep{attias2024universal}.  

    {To overcome the above challenge,} we provide statistical rates for identification and generation. 
    We start with identification. We show that if $\cL$ is identifiable in the limit in the adversarial Gold--Angluin setting with positive examples \citep{gold1967language,angluin1980inductive}, then
it is identifiable under \Cref{def:achievable-rates}
with (almost) exponential (universal) rates.
This is the less technical part of the proof so we will give a high-level approach. 
    
 \itparagraph{B. Identification in the Limit $\implies$ Identification at (Almost) Exponential Rates.} Our idea is reminiscent of \citet*{bousquet2021theory} and requires splitting the input dataset into multiple batches whose size is carefully selected, running the online algorithm on each batch, and then taking a majority vote over the outputs of the algorithm. We remark that there are some technical issues which require further care, compared to \citet{bousquet2021theory}. First, unlike the setting of \citet{bousquet2021theory}, we only see positive examples and we get no feedback about our guesses. Thus, we cannot use their approach to ``estimate'' a time after which the learner will stop making mistakes. Moreover, when we run the learners on multiple batches, it can be the case that different batches output different indices of languages that correspond to $K$ (since the target language can appear at multiple positions in the countable collection $\cL$). 
 
    Thus, taking a majority vote over these indices might not work. Nevertheless, we manage to handle these issues and get almost exponential rates for collections that satisfy Angluin's criterion for identification in the limit \citep{angluin1979finding}. 
    A bit more concretely, {to circumvent the first issue, our approach
    is to ``guess'' the right batch size, and
    this guess needs to be an increasing function of $n$ -- this is why we get almost exponential rates instead of exactly exponential rates (\cref{lem:almost-exp-rates-ident-pos-examples}). 
    The second issue is more subtle. 
    At a high level, we use a voting scheme
    where the output of every batch $\hat L_i$ gives a ``vote''
    to every language $L \in \cL$ such that $\hat L_i = L,$
    and we predict the lowest-indexed language that
    is voted by at least half of the batches. In its
    current form, this scheme is not computable, nevertheless, we show that it can be modified so that it becomes computable (\cref{lem:post-processing-same-index}). %
}

The more interesting half {of establishing universal rates for identification} is the lower bound showing that if a collection is not identifiable in the limit, it is also not identifiable in the statistical setting \emph{at any rate $R(\cdot)$ such that $\lim_{n\to \infty} R(n)=0$}.

\itparagraph{C. Impossible to Identify in the Limit $\implies$ Impossible to Identify at Any Rate.} 
    Recall that the statistical setting was studied by \citet{angluin1988identifying}. \citet{angluin1988identifying} showed that
    every learner, with probability at least 
    $\nfrac{1}{3,}$  
     does not converge to outputting 
     a (stable) index of the target language in an 
     infinite stream of examples drawn from a valid
     distribution. In other words, with probability at least $\nfrac{1}{3,}$
     the algorithm will either
      stabilize to an index that does not correspond to
     the target language or it will not stabilize
     to any index. Notice that this does not
     rule out algorithms that output different
     indices of the target language, for all but
     finitely many $n \in \N.$ The first step
     towards establishing our desired lower bound is
     to strengthen Angluin's result: we show that 
     any learning algorithm, with probability at least
     $\nfrac{1}{3},$ outputs indices that do not correspond
     to the target language infinitely often. More formally,
     let us consider an 
     identification algorithm
    $h_n$, which maps a training set of $n$ examples $x_1,\dots,x_n$ to an index $h_n(x_1,\dots,x_n)$ so that $L_{h_n(x_1,\dots,x_n)}$ is the $n$-th prediction for the true language.
    The aforementioned lower bound means that\footnote{
        {Informally, $\limsup$ of a sequence of events captures the events that occur infinitely often.
        For instance, $\Pr[\limsup_{n\to \infty} \cE_n]$ represents the probability that infinitely many of the events $\cE_1,\cE_2,\dots$ occur.
        On the other hand, $\limsup_{n\to\infty} \Pr[\cE_n]$, roughly speaking, denotes the largest value that the probabilities $\Pr[\cE_1],\Pr[\cE_2],\dots,\dots$ approach infinitely often as $n\to \infty$.}
    }
    \[
        \Pr_{\left\{x_{i}\colon i\in \N\right\} \sim \cP^\infty } \left[ \limsup_{n \rightarrow \infty} \inbrace{L_{h_n\left(x_{1},\ldots,x_{n}\right)} \neq K} \right] \geq \frac{1}{3} \,,
        \yesnum\label{eq:firstLB}
    \]
    where $X \sim \cP^\infty$ corresponds to an infinite \iid{} draw from $\cP.$ 
    One may be tempted to conclude that this implies that with probability $\nfrac{1}{3}$ we cannot identify the target language {(in the statistical setting)}.
    However, the quantity we wish to bound away from 0 to derive the desired lower bound is
    \[
      \limsup_{n \rightarrow \infty} \Pr_{x_{1},\ldots,x_{n} \sim \cP^n}\left[\inbrace{L_{h_n\left(x_{1},\ldots,x_{n}\right)} \neq K}\right]  \,.
    \]
    It is well-known that for any sequence of events $\{\cE_n\}_{n \in \N}$, 
    \[
        \Pr\left[ \limsup_{n \rightarrow \infty} \cE_n \right] \geq \limsup_{n\rightarrow\infty} \Pr[\cE_n]\,.
        \yesnum\label{eq:technicalOverview:wrongInequality}
    \]
    This, however, is not sufficient to deduce the result we need; we need the opposite inequality.
    Hence, Angluin's guarantee does not suffice to get our lower bound. In order to show our result, we use a boosting argument (\cref{lem:no-rate-identification-positive}): 
    if there exists a learner $h_n$ whose probability of misidentification
    \[
    \Pr_{x_{1},\dots,x_{n} \sim \cP^n}\left[L_{h_n\left(x_{1},\dots,x_{n}\right)} \neq K\right]
    \]
    converges to a number {strictly} less that $\nicefrac{1}{2}$, then we can convert it to a learner whose error rate decreases (almost) exponentially quickly. 
        This (almost) exponential rate, in particular, implies that
        \[
            \sum_{n=1}^\infty \Pr_{x_{1},\ldots,x_{n} \sim \cP^n}\insquare{L_{h_n\left(x_{1},\ldots,x_{n}\right)} \neq K}  < \infty\,.
        \]
        This, crucially, enables us to use the Borel--Cantelli lemma (see \Cref{lem:first-borel-cantelli}) which gives us that 
        $ \Pr\left[ \limsup_{n \rightarrow \infty} \inbrace{L_{h_n\left(x_{1},\ldots,x_{n}\right)} \neq K} \right] = 0 $
        , and, thus, a contradiction to \cref{eq:firstLB}. This implies the desired impossibility result.

As consequence of the above results, we get a dichotomy for universal identification rates:
\begin{inftheorem}[see \cref{thm:dichotomy-identification-positive}]\label{infthm:dichotomy-identification-positive}
For any language collection $\cL$ that is identifiable in the limit and for any $g(n) = o(n)$, there exists a learner that identifies $\cL$ at rate $\exp(-g(n))$. Otherwise, $\cL$ is not identifiable at any rate. 
\end{inftheorem}
We remark that if we have access to
subset queries for $\cL,$ we can show that there exists
an algorithm that achieves exactly exponential rates, for all identifiable collections (see \cref{prop:exp-rates-id-subset oracle}).

Next, we move to understanding universal rates for language generation.

\itparagraph{D. Universal Rates for Generation (Possibly Lacking Breadth) without Boosting.}
        One might suspect that a similar batching argument would give us exponential rates for generation: just run the online algorithm of \citet{kleinberg2024language} multiple times and aggregate. The issue is that aggregation for generation is different than prediction: for prediction, it is clear how to implement majority vote as a boosting technique; for generation, it is unclear how to aggregate different generated strings which is, typically, necessary to obtain a boosting algorithm. {One immediate attempt is to take majority votes over the strings that each batch outputs; unfortunately, even if the majority of them are generating from the target language, they might be outputting different strings, thus, even a few batches outputting the same invalid strings are enough to fool our aggregation rule. 
        
        Another tempting approach is to mimic the strategy we used to aggregate different indices of the target language in the identification setting: we go over every output of the batches and we let them give a vote to each of the languages in $\cL$ they belong to.\footnote{The astute reader might realize that, as stated, this strategy is not computable -- as we explain, even if one could implement it, this aggregation scheme does not work.}  It is not hard to see that every batch whose output corresponds to a valid generator will vote for the target language. Unfortunately, it will also vote for all \emph{supersets} of the target language. This is exactly the heart of the difficulty of identification: telling apart supersets of the target language from the target language, which is colloquially called overgeneralization. Taking it to the extreme, imagine that the first language of the collection
        contains all the strings, \ie{}, $L_1 = \cX.$
        Then, all the batches will
        vote for $L_1.$ 
        This is problematic for two
        reasons: generating a fresh string
        from the majority-voted language 
        is as good as random guessing, and choosing
        a string among the ones that voted
        for the majority-voted language is as good
        as picking one of the outputs of all batches
        uniformly at random.

        Perhaps surprisingly, it turns out that a much simpler approach works: we show that the algorithm of \citet{kleinberg2024language} directly enjoys exponential rates in the statistical setting, without the use of batching and boosting. 
        This observation is based on a sufficient condition that allows one to use an
        algorithm that works ``in the limit''
        to obtain exponential rates in the 
        statistical setting, without any modification (see \Cref{lem:generation-from-online-to-statistical-property}).
        \begin{inftheorem}[see \cref{thm:statistical-generation}]\label{infthm:statistical-generation}
            {For any countable
              language collection $\cL$
              there exists a generating algorithm
              that generates from $\cL$ at an (optimal) exponential
              rate.
              }
        \end{inftheorem}

This pair of results for identification (\cref{infthm:dichotomy-identification-positive}) and generation (\cref{infthm:statistical-generation}) allow us to get \Cref{infthm:breadth:mop} and \ref{infthm:generation}. 
The idea for \Cref{infthm:breadth:mop} is that we will use the algorithm $\generator$ that generates with breadth at some rate $R(\cdot)$ for an arbitrary non-identifiable collection $\cL$ and membership oracle access to $\generator$ in order to get an identification algorithm for $\cL$ with some rate $R'(\cdot)$ such that $\lim_{n\to \infty} R'(n) = 0$.
This is a contradiction since \cref{infthm:dichotomy-identification-positive} shows that $\cL$ admits no rate in the universal setting.
Finally, \Cref{infthm:generation} follows almost immediately from our universal rates result for generation.

    \subsection{Additional Results with Relaxation of Consistency and Breadth} \label{sec:intro:relaxation} 
        {
            Next, we study two relaxations of consistent generation with breadth, and ask: \textit{is there a generator that achieves these weaker notions for a non-identifiable collection?}
        }

        In this section, we will allow the generator to repeat examples in the training data. We make this choice for simplicity. Like all of our results with breadth, this choice is not crucial, and all of the results have analogs where the generator does not repeat training examples (see \cref{rem:generatingFromTrainingSet}). We show that {generation with relaxations of breadth defined in this section} from non-identifiable collections is impossible for any generator $\generator$ for which $\mop{}(\generator)$ is decidable and that satisfies the natural property that $\generator$ ``stabilizes'' after seeing sufficiently many examples:
       \begin{restatable}[Stability]{definition}{defStabilityIntro}\label{def:stable-generators}
            A generating algorithm $(\generator_n)$ is stable for a language collection $\cL$ if for any target language $K \in \cL$ 
            and for any enumeration of $K,$ 
            there is some finite $n^* \in \N$ such that for all $n, n' \geq n^*,$
            it holds that
             $\supp(\generator_n) = \supp(\generator_{n'}) $.
        \end{restatable}
        We make {some} initial remarks about stable generators.
        First, any generator $\generator$ that is consistent and achieves breadth is also stable, since after some finite time its support, union the training set, becomes $K$ and remains so.
        {(Here, whether $\generator$ repeats training examples or not is not crucial -- the two types of generators are interchangeable; see \cref{rem:generatingFromTrainingSet}.)}
        Second, this notion of stability can be seen as trying to capture practical heuristics such as learning rate schedules and early stopping that reduce the amount of changes to the generator as more and more samples are seen.
        Moreover, the original work of \citet{gold1967language} also requires the identifier to stabilize to a consistent guess, and, more recently, the stability property of learning algorithms was explored in the PEC learning setting of \citet{malliaris2022unstable}.

        \subsubsection{{Relaxation 1: Unambiguous Generation}}

        Having defined stability, we proceed to discuss {our first relaxation} of generation with breadth.
        Intuitively, consistent generation with breadth requires the generator to eventually stop making mistakes -- where a mistake is any element $x$ that $\generator$ incorrectly includes (if $x\not\in K$ or $x$ is part of the training samples) or excludes (if $x\in K$) from its support. 
        We now relax this and only require that, eventually, the generator $\generator$ makes \textit{finitely}  many mistakes. 
        Observe that this is a non-trivial requirement because the languages contain infinitely many strings and, so, at the start, $\generator$ is expected to make infinitely many mistakes.
        A valuable observation is that it is possible for two languages $L_1$ and $L_2$ to only differ in finitely many strings even if each contains infinitely many strings. 
        With this observation, it is not too hard to see that the aforementioned requirement is too weak to capture a reasonable notion of generation from the target language $K$.
        Indeed, it would allow generators that, given examples from $K$, perpetually generate outputs (with breadth) from a language $L$ that is not the actual target language -- which is a severe form of hallucination.
        
        Hence, to create a meaningful model, we must impose some further restrictions on the mistakes of the generator $\generator$.
        The above example motivates that, at the least, the generator $\generator$ should be ``closer'' to generating from $K$ than some language $L\neq K$ with $L \in \cL$.
        We call such a generator \textit{unambiguous}.
    
         \begin{restatable}[Unambiguous Generator]{definition}{defUnambiguousGenIntro}\label{def:unambiguousGen}
            A generating algorithm $\generator = (\generator_n)$ is unambiguous for a language collection $\cL$ if, for any $K \in \cL$ {and every enumeration of $K$}, its support eventually becomes closer to $K$ than to any other language $L\neq K$ in $\cL$ in terms of the {symmetric difference metric,} \ie{}, {there exists some $n^* \in \N$ such that
            for all $n \geq n^*$ it holds that 
            \[\abs{
            \supp(\generator_n) 
            \triangle K}
                            < 
                            \min_{L\in \cL\colon L\neq K}
                            \abs{
                            \supp(\generator_n) 
                            \triangle L},\]} 
                            where recall that for two sets $S$ and $T$, $S\triangle T\coloneqq \inparen{S\setminus T}\cup \inparen{T\setminus S}$.%
        \end{restatable}

        \begin{figure}[!ht]
        \centering
        {\hspace{20mm}\includegraphics[width=0.7\linewidth]{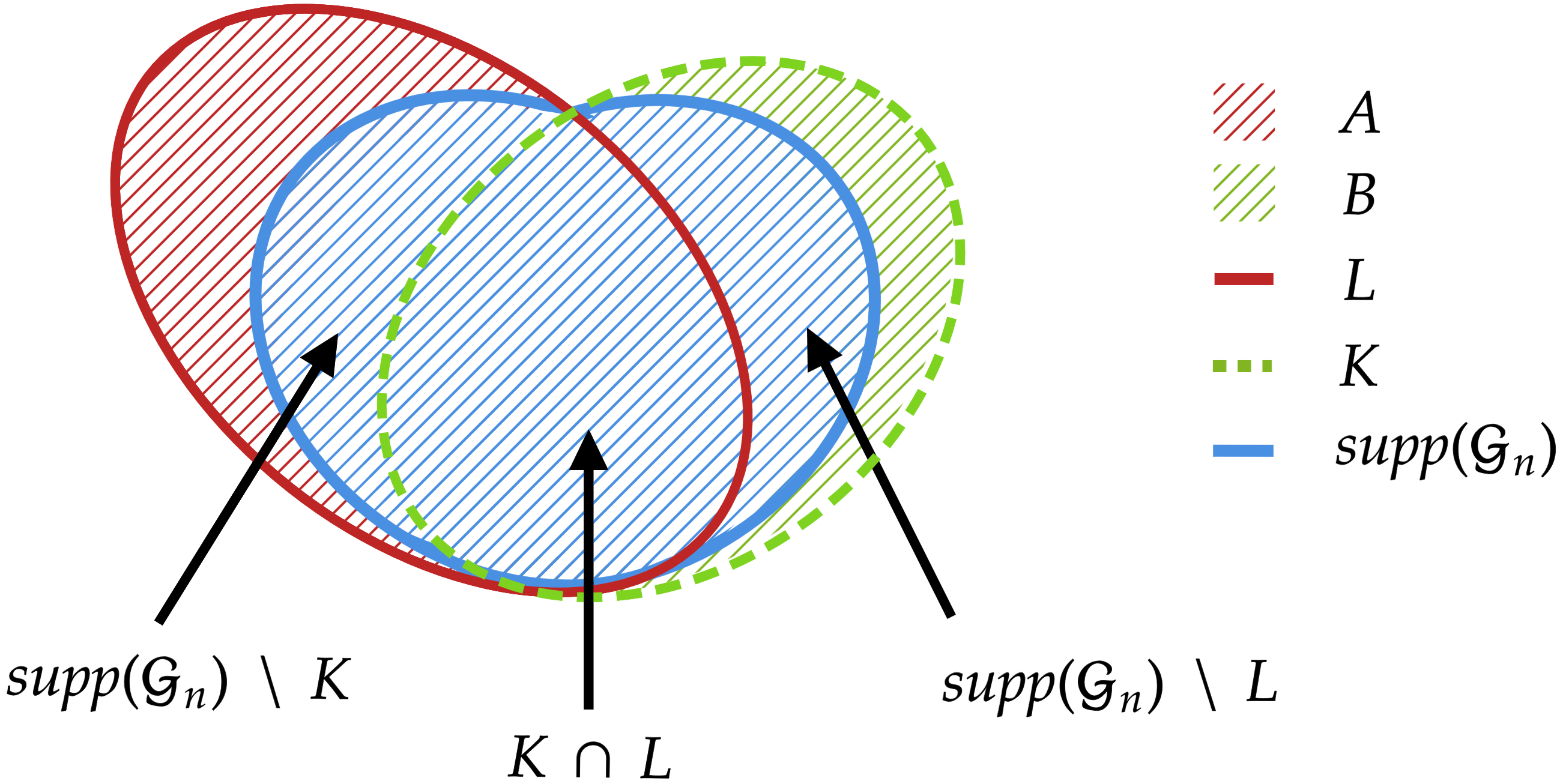}}
            \caption{
            \textit{An Unambiguous Generator That neither Has Consistency nor Breadth.}
            In this example, the language collection $\cL$ has two languages $L$ and $K$, where $K$ denotes the target language.
            The \textcolor{niceRed}{red curve} denotes $L$, the \textcolor{limeGreen}{dashed green curve} denotes $K$, and the \textcolor{royalBlue}{blue curve} denotes the support of $\supp(\generator_n).$
            The generator $\generator_n$ hallucinates since $\supp(\generator_n) \setminus K \neq \emptyset$ and does not achieve breadth for the target $K$ since $B=K\setminus \supp(\generator_n)$ is non-empty.
            Nevertheless, this generator is unambiguous as 
            $\abs{\supp(\generator_n) \setminus K} + \abs{B} < \abs{\supp(\generator_n) \setminus L} + \abs{A}.$}
            \label{fig:unambiguous-generation}
        \end{figure}
        
        \noindent Here, we pause to observe that this notion of generation is a significant relaxation of generation with breadth that we considered earlier (\cref{def:breadth}).
        Not only does it allow the generator to hallucinate certain strings not in the target $K$ and omit strings actually in $K$ for arbitrarily long, the number of hallucinations and omissions can be arbitrarily large, depending on the structure of the language collection $\cL$.
        Surprisingly, we show that even this very weak notion of ``consistent generation with breadth'' is not achievable by a large class of generators.
        
        \begin{inftheorem}[see \cref{thm:impossibility:robust}]\label{infthm:impossibility:robust}
            For every language collection $\cL$ that is not identifiable in the limit, no stable generating algorithm $\inparen{\generator_n}$ for which $\mop{}(\cdot)$ is decidable, can generate unambiguously from $\cL$ at any rate.
        \end{inftheorem}
        Thus, under mild conditions, no stable algorithm can generate unambiguously from a non-identifiable collection.
        Moreover, we also prove an analog of \cref{infthm:impossibility:robust} in the online setting (see \cref{thm:online:impossibility:robust}), which extends our earlier result for generation with breadth in the online setting (\cref{infthm:gen:limit}).
        This raises several questions regarding unambiguous generation, which we leave as interesting open problems  (see \cref{sec:takeawaysOpenQuestions}).
        Note that while this impossibility result has a benign requirement that the generator is stable, it already considerably extends our main result (\cref{infthm:breadth:mop}), since any generator that achieves breadth must be stable -- otherwise, its support cannot settle on the target language $K$.
        {(Note that while \cref{infthm:breadth:mop} requires the generator to not repeat training examples, any generator that repeats training examples can be converted into one that does not repeat training examples and vice-versa; see \cref{rem:generatingFromTrainingSet}.)}

    \subsubsection{{Relaxation 2: Consistency with Approximate Breadth}}
        {Next, we present our second relaxation of generation with breadth -- which we call \textit{approximate breadth}.}
        {This notion, informally,} requires that the generating algorithm is consistent and puts zero mass only on \emph{finitely} many points of the target language $K$. 
        The formal definition is as follows.
        \begin{restatable}[Generation with Approximate Breadth]{definition}{defApproxBreadth}\label{def:gen:oneSidedRobust}
            A generating algorithm $\generator=(\generator_n)$ is said to {generate with approximate breadth} for a collection $\cL=\inbrace{L_1,L_2,\dots}$ if, for any $K\in \cL$ and enumeration $x_1,x_2,\dots$ of $K$, there is an $n_0\geq 1$, such after seeing $n\geq n_0$ elements $x_1,\dots,x_n$,
                $\supp(\generator_n)\subseteq K$
                and 
                $\abs{K\setminus \supp(\generator_n)}$ is finite.
        \end{restatable}
        \noindent The above is a significant weakening of generation with breadth -- since the generating algorithm $\generator$ is allowed to miss elements in the target language infinitely often -- and the number of elements missed is finite but can be arbitrarily large.
        This weakening of generation with breadth seems to be incomparable to the notion of unambiguous generation studied in the previous section.
        To see this, observe that 
        on the one hand, $\generator$ can satisfy the above definition while generating with breadth from a language $L$ that is a strict subset of $K$
        and 
        on the other hand, unambiguous generation allows the generator to generate samples outside of $K$ infinitely often, which is barred by the above definition.
        
        Perhaps surprisingly, even this relaxation of generation with breadth -- which appears incomparable to the one presented in the last section -- is not achievable by a large class of generators.
        \begin{inftheorem}
            \label{infthm:approxbreadth:stat}
            For every language collection $\cL$ that is not identifiable in the limit, no stable generating algorithm $\inparen{\generator_n}$ for which $\mop{}(\cdot)$ is decidable, can generate with approximate breadth from $\cL$ at any rate.
        \end{inftheorem}
        Hence, no stable generator can generate with approximate breadth from non-identifiable collection under mild conditions.
        We also prove an analog of \cref{infthm:approxbreadth:stat} in the online setting \cref{thm:online:impossibility:robustOneSided}.
        Finally, we note that while this result has a requirement that the generator must be stable, like the result for the other relaxation of breadth (unambiguous generation), it already considerably relaxes our main result (\cref{infthm:breadth:mop})

\subsection{Takeaways, Discussion, and Open Problems}\label{sec:takeawaysOpenQuestions}
\label{sec:open}
We believe that a key takeaway of our results is that the question of \citet{kleinberg2024language} seems to open an avenue towards a formal modern theory of language generation bridging learning theory and traditional TCS fields, like complexity theory and formal languages. As we explain in the subsequent technical overview, our tools contribute to this direction by connecting classical lines of work on identification of formal languages tracing back to \citet{gold1967language}, \citet{angluin1979finding,angluin1980inductive,angluin1988identifying} and computability theory \citep{sipser2012introduction,soare1999recursively}, to modern learning paradigms such as learning curves \citep{bousquet2021theory} and language generation \citep{kleinberg2024language,kalai2024calibrated}.

Next, we emphasize that our impossibility result (\cref{mainthm:gen:mop}) is \emph{not} a dead end for language generation.
Instead, it illustrates the need for additional human feedback during the post-training process -- which provides additional information over positive samples alone -- to achieve effective language models. 
Indeed, if both positive and negative examples are available, then generation with breadth is achievable for all countable collections of languages.\footnote{This follows from the work of \citet{gold1967language}, which showed that any countable collection of languages can be identified with such feedback. Using appropriate batching and boosting, we show that this identification algorithm (which works in the limit) can be converted to a generation algorithm with breadth that achieves an exponential rate. Concretely,  \cref{thm:identification-positive-negative} shows how to identify at an exponential rate and \cref{prop:index-to-sampler} shows how to convert this to a generation algorithm.}
In other words, our results can be seen as further theoretical evidence of the benefits of post-training with human feedback, highlighting its importance in developing language models that achieve both consistency and breadth,
and adding to prior theoretical results from \citet{kalai2024calibrated}.

Further, we underline that even though we focus on a prompt-less generation
setting \citep{kalai2024calibrated,kleinberg2024language}, 
most of our results immediately extend to a prompted
setting using the approach of \citet{kleinberg2024language}.

    \paragraph{Remarks and Open Questions.}
We now state a few remarks regarding our results and pose some interesting open questions. 
    First, as a byproduct of our results, we establish almost tight rates for identification and generation with positive examples (see \cref{sec:results:gen} and \cref{sec:results:identification:additional} for formal statements and discussion). 
    Obtaining tight rates for these tasks is an interesting problem. 
    
    Next, our {impossibility} results capture a large class of language-generating algorithms but do not completely forbid consistent generation with breadth.
    An immediate open question is how much further we can extend the class of generating algorithms for which {the impossibility result in} \cref{infthm:breadth:mop} holds. 
       
\begin{description}
    \item[Open Question 1.] 
        {Is there a class of generative algorithms for which the induced generators can be modeled as Turing machines and which achieve breadth and consistency for all countable collections of languages?}
\end{description}
{\itparagraph{Subsequent Progress on Open Question 1.} Since the initial arXiv posting, \citet{charikar2024exploringfacetslanguagegeneration} resolved the question by constructing a language collection and showing that no generator can produce with breadth from it. Subsequently, \citet{charikar2024exploringfacetslanguagegenerationV2} and \cite{kalavasis2024characterizationslanguagegenerationbreadth} concurrently extended this result to all collections violating a variant of Angluin's condition (as in \cref{def:angliun-criterion}). Moreover, \cite{kalavasis2024characterizationslanguagegenerationbreadth} provided unconditional characterizations of collections that can be generated unambiguously and with approximate breadth, both with and without the requirement of stability.
    We refer the reader to \cite{charikar2024exploringfacetslanguagegenerationV2,kalavasis2024characterizationslanguagegenerationbreadth} for a discussion of current open questions in this direction.
}

{Further, we also proved more robust versions of our main result (\cref{infthm:breadth:mop}), namely, \cref{infthm:impossibility:robust,infthm:approxbreadth:stat}, which showed that no algorithm from a large class of generators can generate while (a) making a ``small'' number of hallucinations or omissions  or (b) making a finitely many omissions and zero hallucinations.}
It is interesting to understand if one can prove a more robust version of \cref{infthm:breadth:mop}. 
To this end, we propose the following problem.
\begin{description}
    \item[Open Question 2.] 
    What is the Pareto frontier
    of an approximate notion of breadth and consistency? In other words, 
    if we fix a collection of languages
    and allow the generator
    to hallucinate at some given rate,
    what is the minimal fraction of the mass from the target language
    that this generator has to miss?
\end{description}
Subsequent work of \citet{kleinberg2025density} makes progress on a variant of this question in the \emph{online} setting: they, for instance, design an algorithm that achieves no hallucinations in the limit and covers a non-trivial fraction of the target language (which is in some sense the best possible due to our impossibility results).
Next, to the best of our knowledge, it is not possible to test if a language collection is identifiable in the limit (without access to a strong oracle); this, for instance, becomes evident by inspecting Angluin's criterion for identifiable collections (see \Cref{def:angliun-criterion}). Hence, we would like to know the following:
\begin{description}
    \item[Open Question 3.] Is there a best-of-both-worlds algorithm between consistent generation and generation with breadth, \ie{}, is there an algorithm that will always generate in the limit from the target language consistently but, whenever identification is possible, it will also achieve breadth?
\end{description}
\noindent We make some initial progress on this question by showing that the algorithm proposed by \citet{kleinberg2024language} already achieves
this best-of-both worlds guarantee, provided it has access to a \emph{subset oracle} for $\cL$ that answers queries of the form ``is $L_i\subseteq L_j$?''
{(see \cref{sec:subsetOracle:bestOfBothworlds})}.
{In subsequent work \citet{kalavasis2024characterizationslanguagegenerationbreadth}, we show that a variant of this algorithm (given a subset oracle) has a best-of-three-worlds guarantee: for any countable collection $\cL$, it consistent, stable, and generates with breadth when possible.}

Finally,  our algorithm that achieves (almost) exponential rates for identification uses an algorithm for identification in the limit as a black box. However, our algorithm that achieves exponential rates for generation makes use of certain specific properties of the algorithm of \citet{kleinberg2024language}. Thus, we ask the following question.

\begin{description}
    \item[Open Question 4.] Is there a black-box transformation from an algorithm that generates in the limit in the online setting to an algorithm that generates with \emph{exactly} exponential rates in the statistical setting?
\end{description}

\subsection{Further Related Works}
\label{sec:relatedWork}
Our setting is based on the statistical formulation of \citet{angluin1988identifying}, who studied identification from stochastic examples in the limit. However, \citet{angluin1988identifying} does not provide any learning rates which is one of the main aspects of our work. In terms of techniques, our inspiration for the statistical rates comes from \emph{universal learning}, initiated by \citet*{bousquet2021theory} and studied in \citet{bousquet2021theory,kalavasis2022multiclass,hanneke2022universal,hanneke2023universal,attias2024universal,bousquet2023fine}. However, as we have already explained there are various differences between our setting and our techniques (we provide a more extensive and self-contained discussion in \Cref{sec:comparison}).

Our work connects various disjoint strands of research and we discuss each one of them below.

\paragraph{{Theory on Hallucinations.}}
In terms of rigorous evidence about hallucinations in LLMs, we have already mentioned the work of \citet{kalai2024calibrated} at the start of \cref{sec:intro:new}. {The result of \citet{kalai2024calibrated} is that
calibrated\footnote{The exact definition of calibration is not important for this work: a language model is calibrated if, roughly speaking, the strings that the model assigns probability mass $p$, appear in a $p$ fraction of the true distribution \citep{dawid1982well}.} 
language models must hallucinate. The fascinating implication of this result 
is that one can lower bound the {rate} of hallucination, \ie{}, the quantity $\Ex_{S\sim \cP^n,~x \sim \generator_n}\ind\inbrace{x \notin K}$, by the extent of a model's calibration. 
Their intuition is that the root of hallucinations are \emph{rare} patterns in the training data. Informally, their main result (under assumptions on $K$ and $\cP$) is that for any trained model $\generator_n$ with $n$ samples, the hallucination rate $\Ex_{S\sim \cP^n,~x \sim \generator_n}\ind\inbrace{x \notin K} \geq \wh{R} - \mathrm{Mis}_{\cP}(\generator_n) - \nfrac{1}{\sqrt{n}}$, where $\wh{R}$ is the fraction of facts that only appear \emph{once} in the training data and $\mathrm{Mis}_{\cP}(\generator_n)$
is the {amount of} miscalibration of the model.
Hence, if the model is calibrated, \ie{}, $\mathrm{Mis}_{\cP}(\generator_n) \approx 0$, the hallucination rate is lower bounded by the rare facts' rate. Compared to our work, their goal is to show a quantitative lower bound, which is obtained under assumptions on the training distribution $\cP$ and the fact that the model is calibrated. Our goal is different: we want to understand whether a model can achieve breadth while avoiding hallucinations building on the recent work of \citet{kleinberg2024language}.} We also refer the reader to \citet{kalai2024calibrated} for an extensive overview of applied works on hallucinations.

\citet{peng2024limitations}
 use communication complexity to prove that the transformer layer is incapable of composing functions if the domains of the functions are large
enough. This work could also be seen as rigorous evidence about the hallucinations of LLMs since function composition is a fundamental task for reasoning \citep{guan2024mitigating}. 

The work of \citet{xu2024hallucination} is also studying hallucinations of LLMs. 
They define hallucination as a failure to identify the target function {which belongs to an \emph{uncountable} collection of functions.} {This} is \textit{significantly} stronger than the definition we and prior works \citep{kalai2024calibrated,kleinberg2024language} have considered (making their impossibility results {\textit{significantly}} easier to prove).
Their main result is that all LLMs must hallucinate.
This is easy to see: consider an LLM learning to predict the next element in a sequence of $0$s and $1$s, after observing only a finite prefix of the enumeration, it has no way of knowing the next element in the order (since they allow both continuations) and, hence, the target sequence cannot be identified.

Finally, the work of \citet{aithal2024understanding}, which is mainly empirical, aims to explain hallucinations on the other important {family of} generative models, namely diffusion-based models, via mode interpolation which, in theory, relies on difficulties in approximating non-smooth parts of the score function.

\paragraph{{Language Learning.}}
In our results, we make no implicit assumption about the architecture of our models; this is in accordance with the works of \citet{solomonoff1964formal,gold1967language,angluin1982inference,angluin1983inductive,angluin1988identifying,pitt1989probabilistic,kleinberg2024language}.
However, there are various works aiming at understanding language learning capabilities of specific architectures, \eg{}, 
\citep{hahn2020theoretical,elman1990finding,gers2001lstm,bhattamishra2020ability,hewitt2020rnns,merrill2019sequential,merrill2023parallelism,yao2021self,ebrahimi2020can}. For instance,
\citet{liu2022transformers}
show that low-depth transformers can represent the computations of any finite-state automaton, while \citet{sanford2024representational} identify a particular mathematical problem that cannot be computed by single-layer multi-head transformers.
The aforementioned works share some similarities with us in the sense that they focus on whether models can be trained to generate
or recognize strings in a \emph{fixed} formal language. \citet{akyurek2024context} study in-context language learning: the language model is prompted with a
finite collection of strings from an unknown regular language (which changes across different tasks), and must infer the distribution over strings corresponding to the full language. In a similar spirit, \citet{edelman2024evolution} study in-context learning of Markov chains.
Other related works are those of \citet{hahn2023theory,xie2021explanation} that study conditions under which in-context learning can arise for language learning.

\citet{allen2023physics} design context-free grammars and empirically study the consistent generation (accuracy) and breadth (diversity) of GPT models on these synthetic examples.
In comparison to this work, we provide a theoretical treatment of the trade-off between consistency and breadth under a very abstract model, studied by \citet{gold1967language,angluin1979finding,angluin1988identifying, kleinberg2024language}. Our results indicate that, even in a very idealized framework, achieving (perfect) consistency and breadth is impossible. We view the empirical findings of \citet{allen2023physics} as an exciting indication that, in the real world (or more concretely in controlled experiments on ``small'' models and synthetic datasets), a balance between (imperfect) consistency and breadth is possible and modern LLMs can achieve it. Further understanding how much consistency and breadth one can achieve at the same time theoretically is an exciting direction. 

Finally, in a concurrent and independent work, \citet{li2024generationlenslearningtheory} also study language generation, interpreting it in a learning-theoretic setting reminiscent of the PAC framework and the online learning setting of \citet{littlestone1988learning}. 
They propose ``non-uniform generatability'' -- which relaxes ``uniform generatability'' \citep{kleinberg2024language} --  and characterize the collections for which uniform and non-uniform generatability are achievable in the Gold--Angluin model; in particular, unlike \citet{kleinberg2024language} they also allow the collection $\cL$ to contain uncountably many languages. 
These dimensions are analogs to the Littlestone dimension (and its extension to the non-uniform setting \citep{lu2023non}), which only holds for finite collections of languages. 
Moreover, they show the proposed dimension is incomparable to the VC dimension. 
Finally, they give analogous characterizations in the ``prompted generation'' setting, extending some of the results of \citet{kleinberg2024language}.
Our work is orthogonal to theirs: first, we study trade-offs between generating with and without breadth -- both in a statistical setting and the Gold--Angluin model -- and, second, we study the ``learning curves'' for generation and identification in the framework of \citet{bousquet2021theory}.

\paragraph{{Probably Eventually Correct Learning.}} As we mentioned Gold's model is a predecessor of the famous PAC model of \citet{vapnik2013nature} and \citet{valiant1984theory}. A natural question is whether there is a conceptual meeting point for the two works. Is there a notion of ``PAC learning in the limit?'' The answer to this question is affirmative and 
comes from the field of \emph{algorithmic stability} (see \eg{}, \citep{alon2022private,moran2023bayesian,kalavasis2023statistical,bun2023stability,chase2023stability} and the references therein), studied in the context of binary classification \citep{malliaris2022unstable}.

\citet{malliaris2022unstable} introduce the Probably Eventually Correct (PEC) model of learning. Here we fix a collection $\cL = \{L_1,L_2,\dots\}$ of languages
and a distribution $\cP$ over positive and negative labeled examples (in contrast to the standard identification setting of Gold).
PEC learning focuses on distributions $\cP$ realizable by the collection $\cL$ in the sense of \citet{bousquet2021theory} (see \Cref{sec:comparison}). 
An algorithm is said to PEC learn $\cL$ if for any realizable distribution $\cP$, with probability 1 over \iid{} samples $\inbrace{(x_i,y_i)\colon i\in \N}$ drawn from $\cP$, there exists time $t^* \in \N$ such that for all $t\geq t^*$, given $\inbrace{(x_i,y_i)\colon 1\leq i\leq t}$, the algorithm outputs an $L_t \in \cL$ such that 
    \[
        \Pr_{(x,y)\sim \cP} [L_t(x) \neq y]=0\,.
    \]
Malliaris and Moran give a combinatorial characterization of the collections of languages that are PEC learnable: a collection of languages $\cL$ is PEC learnable if and only if 
it does not shatter an infinite Littlestone tree. We stress that, when the learner has access to positive and negative examples, the absence of an infinite Littlestone tree does not characterize identification in our setting. This is in stark contrast with 
binary classification. In particular, in \Cref{sec:comparison}, we show that 
there exists a set of languages
that have an infinite Littlestone
tree, hence not learnable in the online
setting of \citet{bousquet2021theory}, 
but it allows for identification in the 
limit with positive and negative examples. 
In fact, the collection we use
in \Cref{ex:infinite-littlestone-tree-identifiable}
is identifiable in the limit even with just positive examples.
This already sets the stage 
for a starkly different landscape
of optimal learning rates
between the setting of \citet{bousquet2021theory} and 
\citet{angluin1988identifying}, as we will see in \Cref{sec:results:iden}.

As we said before, the online model of \citet{gold1967language} and the classical online setting of \citet{littlestone1988learning} have various differences.
\citet{lu2023non} studies non-uniform online learning in order to bridge the gaps between the inductive inference model of \citet{gold1967language} and classical online learning. In this setting, the adversary is oblivious and fixes the true language $K$ in advance (as in Gold's model). At each round, an example from $K$ is revealed, the learner makes a prediction but then she observes feedback. The model is non-uniform in the sense that the mistake bound depends on $K.$

\paragraph{{Learning from Positive Examples.}} Learning from positive examples occurs very frequently in
real-world applications and has been extensively studied.
A lot of work has been done on learning from positive examples in
Gold’s model of learning in the limit
\citep{gold1967language,angluin1980inductive,angluin1988identifying,berwick1986learning,shinohara1989inductive,zeugmann2005guided}. Apart from that, an extension of Valiant's PAC model has been also studied \citep{natarajan1987learning,denis1998pac}. 
{\citet{natarajan1987learning} considered the setting where the learner only has access to positive examples and showed that even very simple classes such as halfspaces in two dimensions are not learnable from positive examples alone.
\citet{denis1998pac} relaxed this requirement: they study a setting where the learner} has access to both positively labeled examples but also to unlabeled examples \citep{denis2005learning}. At the
heart of virtually all of the results in this line of work is the use of unlabeled samples in order to generate negative examples. When the original distribution is uniform, better algorithms are known: \citet{de2014learning} gave efficient learning algorithms for DNFs and LTFs, \citet{frieze1996learning,anderson2013efficient} gave efficient learning algorithms for learning $d$-dimensional
simplices. On the other side, \citet{eldan2011polynomial,goyal2009learning} give lower bounds for learning with positive examples. 

Recently, interest in learning from positive examples has sparked from work on truncated statistics (\eg{}, \citep{daskalakis2018efficient,daskalakis2019computationally,Kontonis2019EfficientTS,fotakis2020efficient,daskalakis2021statistical,de2023testing,de2024detecting,plevrakis2021learning,de2024detecting,diakonikolas2024statistical,lee2024efficient,lee2025learningpositiveimperfectunlabeled}).
\citet{Kontonis2019EfficientTS} show how to learn concept{s} of bounded Gaussian surface area from positive {near-diagonal-}Gaussian examples, \citet{lee2024efficient} generalize this to {(non-diagonal)} Gaussian examples (and some more distributions)}, {and \citet{lee2025learningpositiveimperfectunlabeled} further generalize to learn concepts approximable in $L_1$-norm from positive examples; provided the positive example distribution is ``smooth.''}
However, all these works focus on computationally efficient learning/testing while we focus on statistical consistency of identification and generation {without any restrictions on computation time}.

\section{Model and Preliminaries}
    In this section, we introduce notation and preliminaries that are useful in subsequent sections.
    
    \paragraph{Countable Domains and Enumerations.}
        We always assume that languages are subsets of some fixed infinite and \emph{countable} domain $\cX.$
        Since $\cX$ is infinite and countable, after a suitable bijective mapping, one can think of $\cX$ as $\N$.
        In some cases, one may also like to think of $\cX$ as the set of (arbitrarily long) strings over a finite alphabet $\Sigma$, \ie{}, $\Sigma^*$.
        This is again without loss of generality since $\N$ is bijective to $\zo^{*}$ (\eg{}, using the standard binary encoding).
        Depending upon the context,  we use one interpretation ($\cX= \N$) or the other ($\cX= \Sigma^*$), whichever is more intuitive.
        The notion of \emph{enumeration} is important in our work; 
            fix a set $L\subseteq \cX$.
            We refer to $L$ as a language.
            An enumeration of $L$ is a complete and ordered listing of all the elements in $L$ that allows for, potentially, repetitions of elements.
            In particular, an enumeration $x_1,x_2,\dots$ of $L$ has the property that for any element $w\in L$ there is a finite index $i$ such that $x_i=w$.
            For example, $1,2,3,\dots$ is a valid enumeration of $\N$ but $2,4,6,\dots$ is not (since, in the latter sequence, odd numbers do not appear at any finite position).
    
    \paragraph{Additional Notation.}
        We use $\algo{A}$, $\algo{I}$, and $\generator$ to denote algorithms, and often reserve $\generator$ for a generator, \ie{}, an algorithm that given examples $x_1,\dots,x_n \in \cX$, outputs a new example from $\cX$.
        We use $\cP$ and $\cD$ to denote distributions over the elements of some language $L\subseteq \cX$.
        We use standard notation related to distributions:
        Fix a distribution $\cP$ over language $L$.
        Given an element $x\in \cX$, $\cP(x)$ denotes the probability mass $\cP$ assigns to $x$.
        The support of distribution $\cP$ is denoted by $\supp{(\cP)}$, \ie{}, $\supp{(\cP)}\coloneqq \inbrace{x\in L\colon \cP(x)>0}$.
        As a shorthand, given a sequence $x_1,x_2,\dots,x_n$, for each index $1\leq i\leq n$, we use $x_{\leq i}$ to denote the prefix $\inbrace{x_1,x_2,\dots,x_i}$.
        Finally, we use standard notation for indicator functions and limits:
            Given an expression $E$ (such as $h\neq K$ or $s\in K$), $\ind\inbrace{E}$ denotes the indicator that $E$ is true.
        {For a function $R\colon\N\to \R_{\geq 0}$, $R\downarrow 0$ denotes that $\lim_{n\to \infty}R(n)=0$.}
        
    \paragraph{Language Collections and Membership Oracle to Languages.}
        We always consider a \textit{countable} collection of languages $\cL=\inbrace{L_1, L_2, \dots}$ and assume we have access to a \textit{membership oracle} that, given an index $i$ and a string $s$, outputs $\ind\{ s \in L_i\}$, as is standard in all prior works \citep{gold1967language,angluin1980inductive,kleinberg2024language}.
        This is motivated by the fact that if these languages are ``reasonable,'' \eg{}, they are generated by context-free grammars or decided by Turing machines \citep{sipser2012introduction}, then (1) there can be only countably many of them and (2) each of them admits a membership oracle.
        Finally, we reserve the letter $K$ to denote the unknown target language $K\in \cL$. 
        We will say that an example $x$ is a \emph{positive} example for $K$ if $x \in K$; otherwise $x$ will be a \emph{negative} example for $K$.

    \subsection{Language Identification and Generation in the Limit}

        In this section, we first present the Gold--Angluin model for identification in the limit and, then, Kleinberg and Mullainathan's model for generation in the limit.

   \subsubsection*{Language Identification in the Limit} 
    The problem of language identification in the limit from positive examples was introduced by \citecustom{gold1967language} and further studied by \citecustom{angluin1979finding,angluin1980inductive}.
    The setting is specified by a collection of languages $\cL = \{L_1,L_2,\dots\}$. %
    For a fixed collection $\cL$, an adversary and an identifier play the following game: 
    The adversary chooses a language $K$ from $\cL$ without revealing it to the identifier, and it begins \emph{enumerating} the strings of $K$ (potentially with repetitions) $x_1,x_2,\dots$ over a sequence of time steps $t = 1,2,3,\dots$. 
    The adversary can repeat strings in its enumeration,
    but the crucial point is that for every string $x \in K$, there must be at least one time step $t$ at which
    it appears. 
   
   At each time $t$, the identification algorithm $\algo{I}$, given the previous examples $x_1,x_2,\dots,x_t$, outputs an index $i_t$ that corresponds to its guess for the true language $K$.

    \begin{definition}
        [Language Identification in the Limit \citep{gold1967language}]\label{def:Identification}
        Fix some language $K$ from collection $\cL$.
            The identification algorithm  $\algo{I}$ identifies $K$ in the limit if there is some $t^* \in \N$ such that for all steps $t > t^*$, the identifier’s guess $i_t$ satisfies $i_{t} = i_{t-1}$ and $L_{i_t} = K.$
            The language collection $\cL$ is identifiable in the limit if there is an identifier that identifies in the limit any $K \in \cL,$ for any enumeration of $K$.
        \end{definition}
\citecustom{gold1967language} showed that collections of finite cardinality languages, \ie{}, each language in the collection $\cL$ is finite, can be identified in the limit from positive examples.
{This is true since in the limit, one will see all the elements of the target (finite) language, at which point it can be identified.}
The identification algorithm is the following: at time $t$, guess $L$ to consist
solely of the elements that have occurred in the 
sequence. Since $L$ is finite, there will be a finite time after which all elements of $L$ will have been revealed, so after that the algorithm will have identified the target. Interestingly,
all finite collections of languages
are also identifiable in the limit \citep{gold1967language}.

A super-finite collection of languages denotes any collection which contains all
languages of finite cardinality and at least one of infinite cardinality. Gold showed that super-finite collections of languages cannot be identified in the limit from positive examples. Further, he showed that negative examples help: any super-finite collection can be identified in the limit using positive and negative examples\footnote{This means that the adversary presents an enumeration of the whole domain $\cX$, with a label
indicating whether the example is in the target language.} (the idea is simple: keep guessing the infinite language until seeing a negative example; then it reduces to the finite case).

\begin{theorem}
[\citep{gold1967language}]
\label{thm:gold}
Let $\cL = \{L_\infty, L_1, L_2, \dots\}$ be the language collection with
$L_1 \subset L_2 \subset \dots \subset L_\infty = \cup_{i \geq 1} L_i$ and for each $i,$ $\abs{L_i} < \infty$. Then, there is no algorithm that identifies $\cL$ in the limit from positive examples. Moreover, this collection can be identified in the limit when the algorithm has access to both positive and negative examples.
\end{theorem}
The above result already shows a separation in terms of identification between observing only positive examples and observing positive \emph{and} negative examples in Gold's model. Moreover, it raises the question of whether there exist non-trivial collections of languages identifiable in the limit from positive examples. In that direction, \citet{angluin1979finding}
studied pattern languages (whose definition is not important for our work) and showed that for that collection identification in the limit is possible only with positive examples.

The next question is whether one can get a \emph{characterization}
of the language collections that can be identified from positive examples.
\citecustom{angluin1980inductive} resolved this problem. %

\begin{definition}
[Angluin's Condition \citep{angluin1980inductive}]
\label{def:angliun-criterion}
    Fix a language collection $\cL = \{L_1, L_2, \dots\}$.
    Suppose there is a membership oracle which, given a string $x$ and index $i$, answers $\ind\{x \in L_i\}$.
    The collection $\cL$ is said to satisfy Angluin's condition if there is an oracle that given an index $i$ enumerates a set of \textit{finite} strings $T_i$ such that
    \begin{center}
        $T_i \subseteq L_i$ and for all $j \geq 1$, if $T_i \subseteq L_j$ then $L_j$ is not a proper subset of $L_i$.
    \end{center}
\end{definition}
The difficulty in trying to identify a language from positive examples is the problem of \emph{over-generalization}. If while seeing positive examples the algorithm specifies a language that is a proper superset of the true answer $K$, then by only seeing positive examples it will never see a counterexample to that language. This would be avoided with positive and negative examples. Angluin's condition essentially ensures this over-generalization problem can be avoided by from just positive examples (without the help of negative examples).

Before proceeding to Angluin's result, we stress one important point: inspecting Angluin's definition, we can see that it requires access to a procedure that \emph{finds} this set of strings $T_i$. This oracle is called a \emph{tell-tale} oracle and is quite crucial for Angluin's algorithm to work.

\Cref{def:angliun-criterion} led to the following characterization.
\begin{theorem}
[\citep{angluin1980inductive}]\label{thm:angluin-id-limit}
A countable language collection $\cL$ is identifiable in the limit if and only if it satisfies Angluin's criterion. %
\end{theorem}
Finally, let us consider the case of language identification with both positive and negative examples, \ie{}, when
the adversary provides an enumeration of the whole domain
$\cX$ and every example has a label indicating whether
it is in the true language $K.$
We mention that focusing on algorithms equipped with membership oracle, the following result appears in \citecustom{gold1967language}. %
\begin{theorem}[\citep{gold1967language}]
Any countable language collection is identifiable in the limit from positive and negative examples.
\end{theorem}
To see how the algorithm works,
let $\cL = \{L_1,L_2,\ldots\}$ and denote by $L_z$ the 
smallest indexed language in $\cL$ for which $L_z = K.$
The algorithm observes an enumeration of the form $(x_{t},y_{t}) \in \cX \times \{0,1\}$ for $t \geq 1$. Recall this means that $\ind\{x_{t} \in K\} = y_t.$
The algorithm works as follows: in every timestep $t \in \N$, it predicts the lowest index of a consistent language, \ie{}, the smallest $j \in \N$ for which 
$\ind\{x_{\tau} \in L_j\} =  y_\tau$ for all $\tau \leq t.$ 
Consider two cases: if $z = 1,$ then the algorithm
will never predict any language $L_{z'}, z' \geq 2,$
so it will be correct from the first step.
If $z > 1,$ then for all $L_{z'}, z' < z,$ that come before $L_z$ in the enumeration of $\cL$, there is a finite time $t_{z'}$ when the example $\inparen{x_{t_{z'}},y_{t_{z'}}}$ contradicts the language
$L_{z'}$.

\subsubsection*{Language Generation in the Limit} 
    We now move to language generation in the limit from positive examples, introduced by \citecustom{kleinberg2024language}.  The setup is exactly the same as in the Gold--Angluin model (the adversary provides an enumeration of $K$), but now the goal
    of the learner is to \emph{generate unseen examples}
    from $K$ instead of identifying the index of $K$. Their formal definition is the following.
        
    \begin{definition}
        [Language Generation in the Limit \citep{kleinberg2024language}]\label{def:consistentGeneration}
        Fix some language $K$ from the collection $\cL = \{L_1, L_2,\dots\}$ and a generating algorithm $\generator.$
            At each step $t$, let $S_t \subseteq K$ be the set of all strings that the algorithm $\generator$ has seen so far. 
            $\generator$ must output a string $x_t \notin S_t$ (its guess for an unseen string in $K$). 
            The algorithm  $\generator$ consistently generates from $K$ in the limit if, for all enumerations of $K$, there is some $t^* \in \N$ such that for all steps $t \geq t^*$, the algorithm’s guess $a_t$ belongs to $K \setminus S_t$ ({or $K \setminus S_t$ is empty)}. The collection $\cL$
            allows for consistent generation in the limit if there is an algorithm  $\generator$ that, for any choice of the target language $K \in \cL,$ 
            it consistently generates from $K$ in the limit.
        \end{definition}
        \cref{def:consistentGeneration} straightforwardly generalizes to randomized algorithms; 
            consider the same setup as before except that now the output string $a_t$ may be randomized.
            The definition of generation is also the same except that instead of requiring $a_t\in K \setminus S_t$ one requires that the support $A_t$ of the distribution from which $a_t$ is sampled is non-empty and satisfies $A_t\subseteq K \setminus S_t$.
          
     Observe that language generation requires that the algorithm's outputs are \textit{consistent} with $K$ (in the limit), but allows the algorithm to not generate certain strings from $K$.
        For instance, if $K$ is the set of all strings, then the algorithm that always outputs even length strings (not in $S_t$), generates from $K$ in the limit but also misses infinitely many strings in $K$ (namely, all strings of odd length).
        Consistency is clearly a desirable notion: without consistency, algorithms may keep outputting strings outside the target language $K$ which, when $K$ is the set of all meaningful and true strings, inevitably leads to hallucinations.
    
    A trivially consistent generator is one that outputs data already seen in the training set. As we already mentioned, we count such outputs as mistakes. This form of predicting unseen positive examples makes the task of generation interesting. At first sight, it seems that there is an easy strategy that achieves generation in the limit: given an enumeration of all hypotheses $L_1,L_2,\ldots$, we sequentially generate from $L_i$ $(i=1,2,\ldots)$ until it becomes inconsistent with the sample $S_n$; then we move to $L_{i+1}.$   
This strategy seems natural for generation because we know that there is some index $k$ such that the true language $K = L_k$. This idea has a fundamental issue, already reported by \citecustom{kleinberg2024language}: if there exists an index $i$ such that $i < k$ and $L_k \subsetneq L_i$, then the generator will get stuck at $L_i$ and never update. 

A non-trivial solution to this problem was given by \citecustom{kleinberg2024language}. They show that
     all countable sets of languages
     in countable domains allow
     for generation in the limit from positive
     examples; this is in stark contrast with identification in the limit from positive examples.
     
     \begin{theorem}
    [Theorem 1 in \citecustom{kleinberg2024language}]
    There is an algorithm with the property that for any countable collection of languages $\cL =
    \{L_1, L_2,\dots\}$, any target language $K \in \cL,$ and any enumeration of one of these languages $K$, the algorithm generates from $K$ in the limit with positive examples.
     \end{theorem}
     We now provide some intuition on how this algorithm works. Let $L_1, L_2,\dots$ be an enumeration of the collection of languages and $K$ be the true language.
    Let $z$ be an index such that $L_z = K$. We say that a language $L_i$
is consistent with the sample $S_t$ at time $t$ if $S_t$ is contained in $L_i$. Now assume that we have two languages $L_i$ and $L_j$ with $L_i \subseteq L_j$ which 
are both consistent with $S_t.$
Then, it is clear that the generating algorithm should prefer to generate from $L_i$ rather than $L_j$: any $w \in L_i \setminus S_t$ satisfies $w \in L_j \setminus S_t.$ This property inspired \citecustom{kleinberg2024language} to define the notion of a \emph{critical language.} Let $\cC_n = \{L_1,L_2,\dots,L_n\}$. A language $L_n$ is critical at step $t$ if $L_n$ is consistent with $S_t$ and for every $L_i \in \cC_n$ that is consistent with $S_t$, it must be $L_n \subseteq L_i.$ There are some key properties upon which the generating algorithm is built:
\begin{itemize}
    \item At any time, there is at least one language consistent with $S_t$, the true one $L_z = K.$ Also, there is at least one
critical language at any step $t$: for any $t$, the consistent language $L_i$ with the lowest index $i$ must
be critical at step $t$, as it is the only consistent language in $C_i.$ 

\item There exists times $t$ for which $L_z$ (which is $K)$ is not critical. But eventually, $L_z$ will become critical at some step and then remain critical forever
after that. Also, any critical language coming after $L_z$ must be a subset of $L_z$, thus it is safe to generate from it.

\item Hence the algorithm, roughly speaking, keeps track of a list of critical languages and generates from the last one in the list; this is because, after some finite index, all the critical languages are subsets of $L_z$ and, hence, it is safe to generate from any of them.
\end{itemize}
More details about this algorithm will appear later on when we design our generation algorithms for the probabilistic setting (see \Cref{sec:statistical-generation}).

\section{Overview of Results}\label{sec:results}
    In this section, we present the formal statements of our main results. 
    We begin with statistical rates for identification and for consistent generation (without the requirement of breadth) in \cref{sec:results:iden}.
    Next, in \cref{sec:results:genBreadth}, we present our results for generation with breadth -- showing that no generator from a large family of generators (that includes present-day LLMs) can generate with breadth from any language collection that is non-identifiable.
    Contrasting \citecustom{kleinberg2024language}'s result for generation without breadth, these results show that generation with breadth is significantly harder -- as hard as identification, for a large and natural class of generators.
    \cref{sec:results:genBreadth:robust} extends this impossibility result to a relaxation of generation with breadth, showing that even this relaxed definition of generation and breadth cannot be achieved by the same large class of generators.
    Finally, in \cref{sec:results:identification:additional}, we present additional results for identification when one has some additional structure (\eg{}, access to a stronger oracle) or information (\eg{}, negative examples). %

     \subsection{Results for Identification and Generation without Breadth}\label{sec:results:iden}\label{sec:results:gen}
\label{sec:statistical-learning-results}         %
    Prior work of \citecustom{gold1967language,kleinberg2024language} studies language identification and generation in an online, \ie{}, adversarial setting. In this work, we study the distributional versions of these problems. The identification problem we study is not new and, in fact, goes back to Angluin's work in 1988 \citep{angluin1988identifying}. However, \citecustom{angluin1988identifying} does not provide any \textit{rate} at which language identification can be achieved as the number of samples observed increases (when it is achievable). 

\paragraph{Summary of Results in This Section.}
In this section, we give learning rates for both identification and generation (see \cref{thm:dichotomy-identification-positive,thm:statistical-generation} respectively). 
For both tasks, we study the learning curves {-- that is how the identification or generation error decays as the sample size increases}. %
As a result, we extend the results of \citecustom{gold1967language} and \citecustom{kleinberg2024language} to the statistical setting.
Our results in this section achieve a near-optimal rate for identification (\cref{thm:dichotomy-identification-positive}) and an optimal rate for generation (\cref{thm:statistical-generation}).

\subsubsection{Universal Rates: Model and Preliminaries}
We work under the \emph{universal rates} framework, introduced
by \citet*{bousquet2021theory}, in order to capture the notion of a learning curve for language identification and generation. 
Following the notation we used before, recall that we have a countable
set of languages $\cL = \{L_1,L_2,\ldots\}$, where each
$L \in \cL$ is also countable and $\cup_{L \in \cL} L \subseteq \cX,$
for some countable domain $\cX.$ Recall
the notion of a valid distribution proposed
by 
\citecustom{angluin1988identifying} in this setting
(\cref{def:valid-language-distribution}).
Intuitively, this condition can be thought of as the equivalent of \emph{realizability} in the classification setting. 

The learning algorithm is a sequence of 
(universally measurable and computable) functions 
$\{h_n\}_{n \in \N},$ where $n$ captures the size of the training set.
We are interested
in understanding the behavior of the \emph{error}
of the algorithm, which is defined appropriately based on the downstream task -- either
identification or generation for this paper. Given
some rate function $R\colon  \N \rightarrow [0,1]$
we say that we can achieve rate $R(n)$ for the set of language $\cL$ and the loss function $\mathrm{er}(\cdot)$ 
if there exists a learning algorithm $\{ h_n\}_{n\in \N}$
whose error
satisfies
\[
\left(\forall \text{ valid } \cP\right)
\left(
\exists C,c\right) \quadtext{such that}
\E[\mathrm{er}(h_n))]
\leq
C\cdot R(c n)\,,\quad \forall n \in \N \,.
\]
Crucially, these learning curves are distribution-specific; the constants $c,C$ depend on $\cP$ but the rate $R$ holds universally for all valid distributions.
Such learning curves are a well-studied topic in learning theory \citep{antos1996strong,schuurmans1997characterizing,bousquet2021theory,viering2022shape}.
The above gives rise to the following definition. 

\begin{definition}[Learning Rates  \citep{bousquet2021theory}]
\label{def:achievable-rates}
Given a language collection $\cL$, an error function $\mathrm{er}(\cdot)$, and a rate function $R\colon\N \rightarrow [0,1]$ satisfying $\lim_{n \to \infty}R(n)\rightarrow 0$, we say:
\begin{itemize}
    \item Rate $R$ is achievable for $\cL$ if there is an algorithm $\{{h}_n\}_{n \in \N}$
    such that for every valid distribution $\cP$, 
    there exist $c,C$ for which 
    $\E[\mathrm{er}(h_n)] \leq C\cdot R(c\cdot n), \forall n \in \N.$ 
    
    \item No rate faster than $R(n)$ is achievable for $\cL$
    if for all algorithms $\{{h}_n\}_{n \in \N}$
    there exists a valid distribution $\cP$ and $c,C$ for which $\E[\mathrm{er}(h_n)] \geq C\cdot R(c\cdot n)$, for infinitely many
    $ n \in \N.$ 
\end{itemize}
Further, we have the following.
\begin{itemize}
    \item (Optimal Rate) Rate $R$ is optimal for $\cL$ if it is achievable and no rate faster than $R$
    is achievable.

    \item (No Rate) We say that $\cL$ admits no rate  
 if for every algorithm $\{{h}_n\}_{n \in \N}$
    there exists a valid distribution $\cP$ such that    $\limsup_{n \rightarrow \infty}\E[\er(h_n)] > 0.$ 
\end{itemize}
\end{definition}
In the case of identification, 
to avoid trivial cases, we consider collections $\cL$ that contain at least two distinct languages that contain one common element.
\begin{definition}[Non-Trivial Collections of Languages for Identification]\label{def:non-trivial-collections}
    A language collection $\cL$ is non-trivial for identification if there exist two languages $L_1, L_2 \in \cL$
    such that $L_1 \neq L_2$ and 
    $\abs{L_1\cap L_2} > 0$.
\end{definition}
Notice that if the collection $\cL$ does not satisfy \Cref{def:non-trivial-collections},  then one can identify the target language $K$ immediately after observing a single element from $K$. 

In the case of generation, the ``non-triviality'' condition turns out to be more nuanced, \eg{}, compared to the case of identification above 
or binary classification \citep{bousquet2021theory}. We give an informal
definition below, and we refer to \cref{def:non-trivial-collection-generation}
for the formal one and a discussion about its necessity.

\begin{infdefinition}[Non-Trivial Collections of Languages for Generation,
see \cref{def:non-trivial-collection-generation}]
     A language collection $\cL$ is non-trivial for generation if any algorithm
     needs to see at least two examples from the target language
     to be able to generate from it.
\end{infdefinition}

\subsubsection{Universal Rates for Identification} 
For any language collection $\cL = \{L_1,L_2,\dots\}$  and $n \in \N$, with true language $K \in \cL$, and set of examples $x_1,\ldots,x_n \in \cX^n$, 
an identification algorithm $\algo{I}_n$ gets as input $x_1,\dots,x_n$ and outputs an index $\algo{I}_n(x_1,\dots,x_n)$. We define the \emph{identification error} of the learner $\{\algo{I}_n\colon \cX^n \rightarrow \N \}_{n \in \N}$ as
\begin{equation}\label{eq:error-statistical-identification-pos}
    \mathrm{er}(\algo{I}_n(x_1,\ldots,x_n)) = \ind\{L_{\algo{I}_n(x_1,\ldots,x_n)} \neq K\} \,.
\end{equation}
Under this definition, $\E_{x_1,\dots,x_n \sim \cP}[\er(\algo{I}_n)] = \Pr_{x_1,\dots,x_n \sim \cP}[L_{\algo{I}_n(x_1,\dots,x_n)} \neq K],$
\ie{}, the probability that it fails to identify the correct language after it sees $n$ examples from $\cP$.\footnote{One subtle point is that this definition allows the learner to output any index $j \in \N$ such that $L_j = K$ and there may be many such indices since we do not assume all languages in $\cL$ are distinct. {Our identification algorithms will have the property that they output the smallest index at which $K$ appears in $\cL=\inbrace{L_1, L_2,\dots}$.}}

Our main result for identification is a fundamental dichotomy: every non-trivial collection of languages is identifiable with positive examples at either an (almost) exponential rate or it is not identifiable at any rate.

\begin{theorem}[Dichotomy of Rates for Identification with Positive Examples]\label{thm:dichotomy-identification-positive}\label{mainthm:statisticalRates:iden}
    For every collection of countably many languages $\cL$ that is non-trivial for identification exactly one of the following holds:
    \begin{itemize}
        \item For every $g(n) = o(n)$ there exists a learner that identifies $\cL$ at rate $e^{-g(n)}.$ Moreover, no learner can achieve a rate faster than $e^{-n}.$
        \item $\cL$ is not identifiable at any rate.
    \end{itemize}
    Concretely, the first condition holds for $\cL$ if and only
    if it satisfies Angluin's condition (\Cref{def:angliun-criterion}).
    
\end{theorem}
This dichotomy of rates differs from prior universal rates for classification where the usual theme is a trichotomy of rates \citep{bousquet2021theory,kalavasis2022multiclass,hanneke2023universal}. Moreover, while in the universal setting for binary classification, any measurable class of functions is learnable at arbitrarily slow rates, in identification, this is not the case: there exist collections of languages that do not admit a Bayes consistent learner and these are exactly the collections that do not satisfy Angluin's condition. For the full proof, we refer the reader to \cref{sec:proofof:mainthm:statisticalRates:iden}.

\subsubsection{Universal Rates for Consistent Generation} 
    The main difference between this setting and the setting of language identification is the definition of the error rate. There exists a valid text-generating distribution $\cP,$ meaning one that is supported on some target language $K \in \cL,$ and the learning (or rather, generating) algorithm is a sequence of  (universally measurable and computable) functions  $\{\generator_n\colon \cX^n \rightarrow \cX\}_{n \in \N},$ where each $\generator_n$ takes as input $n$ samples generated \iid{} from $\cP$ and outputs a new word, with the goal that this word belongs to the target language (see \cref{rem:generatorNotation}). 
    As in the online setting, to avoid trivial solutions, we want to generate examples that do not appear in the training set.

    \begin{remark}[Notation for Generating Algorithms]\label{rem:generatorNotation}
        More formally, a generating algorithm is a collection of mappings $\inbrace{\generator_n}_{n\in \N}$, where for each $n$, $\generator_n$ is a mapping from the domain of $n$ training samples $\cX^n$ to the set of ``generators'' or (randomized) Turing machines $\cG$ that, on each execution, output a sample from $\cX$.
        For this section, it is sufficient to imagine generators as
        being deterministic 
        (\ie{}, generating samples from a point mass) and, hence, we simplify writing $\generator_n$ as a mapping from $\cX^n$ to $\cX$.
        In the next section, where we study generation with breadth, to have any hope of achieving breadth, we need to consider $\generator_n$ in its full generality as a mapping from $\cX^n$ to $\cG$.
    \end{remark}

    \noindent 
    Now, we are ready to define the \textit{generation error}:
    for any $n \in \N$ and set of examples $x_1,\ldots,x_n \in \cX^n$  we define the generation error of the learner $\{\generator_n\colon \cX^n \rightarrow \cX\}_{n \in \N}$ for this task as
\begin{equation}\label{eq:error-statistical-generation-pos}
    \mathrm{er}(\generator_n(x_1,\ldots,x_n)) = \ind\{\generator_n(x_1,\ldots,x_n) \notin K \setminus \{x_1,\ldots,x_n\}\} \,.
\end{equation}
Notice that, under this definition, 
\[
    \E_{x_1,\ldots,x_n \sim \cP^n} [\er(\generator_n(x_1,\ldots,x_n))] = \Pr_{x_1,\ldots,x_n \sim \cP^n}[\generator_n(x_1,\ldots,x_n) \notin K \setminus \{x_1,\ldots,x_n\}] \,,   
    \]
\ie{}, the probability that the learner fails to generate a new
word from the target language after observing $n$ examples from it. Our main result in this section is that we can achieve
consistent generation with exponential rates. %

\begin{theorem}[Rates for Generation]\label{thm:statistical-generation}
    For every countable collection of languages  $\cL$ there exists a generating algorithm that generates from $\cL$ at rate $e^{-n}.$ Conversely, for every  collection of languages that is non-trivial for generation (\Cref{def:non-trivial-collection-generation}), no generating algorithm can achieve rate faster than $e^{-n}.$
\label{mainthm:statisticalRates:gen}
\end{theorem}
Surprisingly, this shows that consistent generation can be achieved at an exponential rate for \textit{any} countable collection of languages. We mention that the result we prove is slightly stronger:
    we show that, for any $\cL$, with probability at least $1-C\cdot e^{-c \cdot n}$, we can generate \emph{infinitely} many new strings from $K$, after training the algorithm on $n$ examples -- not just a single word.
    Together, \cref{mainthm:statisticalRates:iden,thm:statistical-generation} show that the stark separation between language identification and generation in the online setting, obtained by \citecustom{kleinberg2024language}, also extends to the statistical setting of \citecustom{angluin1988identifying} and \citet{bousquet2021theory}.
    The proof of \cref{mainthm:statisticalRates:gen} appears in \cref{sec:proofof:mainthm:statisticalRates:gen}; see \cref{fig:outline:thm:statistical-generation} for an outline of the proof.

    \subsection{Results for Generation with Breadth}\label{sec:results:genBreadth}
        Next, we present our results for language generation with breadth.
        Clearly, generation with breadth is a stronger requirement than generation.
        But, at least intuitively, it is weaker than identification: it only requires one to generate samples from the entire support of $K$ and not identify the index of $K$.
        Contrary to this intuition, our results show that, for a large class of generators, generation with breadth is as hard as identification.
        Our results show that,
            while this class of generators is powerful enough to generate without breadth, 
            no generator in this class can achieve generation with breadth for non-identifiable collections of languages.

        \subsubsection{Membership Oracle Problem}
        The family of generators we consider is implicitly determined by the decidability of a certain problem associated with the generator.
        \defMOP*
        \noindent As mentioned before the decidability of problems is extensively studied in formal languages and complexity theory \citep{sipser2012introduction}.
        Our main result (\cref{infthm:breadth:mop} whose formal statement appears as \cref{mainthm:gen:mop}) applies to any generator $\generator$ for which $\mop{}(\generator)$ is decidable.
        Note that our result only needs a decider of $\mop{}(\generator)$ to \textit{exist} --  this is purely a property of the generation algorithm used -- and it does \textit{not}, for instance, require the individuals training the generator or the users to have access to the decider in any fashion.  
        
        To gain some intuition about the membership oracle problem, let us consider a simple example.
        \begin{example}[Standard Next-Token Predictor]\label{ex:decidableMOP}
            Let $\generator{}_{\rm next\text{-}token}$ be a text generator or language model that generates text token-by-token: at each step $t$, it generates certain scores $\inbrace{p_{t}(\sigma)\colon \sigma\in \Sigma}$ and outputs token $\sigma$ with probability $\propto p_{t}(\sigma)$.
            It is not important how these scores are generated.
            They can be generated in various ways.
            For instance, they can be the logit-scores of transformer-based models. 
            They could also be generated by thresholding logit-scores in any complicated but computable way --  such as, by using beam search, top-$K$, or top-$p$ sampling \citep{Holtzman2020The}.
            $\mop{}(\generator{}_{\rm next\text{-}token})$ is decidable and, in fact, there is a simple decider:
            given a string $w$ of length $n$, it computes the scores for the first $n$ iterations; where in the $t$-th iteration ($t>1$), it conditions on the event that $\generator{}_{\rm next\text{-}token}$ has generated the string $w_1w_2\dots w_{t-1}$ so far.
            Then it computes the following function and outputs the result
            \[
                \begin{cases}
                    \textsf{Yes} & \text{if } \prod_{t=1}^n p_{t}(w_t) > 0\,,\\
                    \textsf{No} & \text{otherwise}\,.
                \end{cases}
            \]
            We stress that our main result only needs the existence of such a decider, and does not require the individuals training the generator or the users to have any access to it.
        \end{example}

    \subsubsection{Results for Generators for which \texorpdfstring{$\mop(\cdot)$}{MOP} Is Decidable}
        Before stating our result about the rate at which generation with breadth can be achieved, we need to define the corresponding error function.
        For the error to make sense, let $\cG$ be the set of (randomized) Turing machines that do not take any input and output one element from $\cX$ (on each execution).
        Given a target language $K$ and examples $x_1,\dots,x_n \in \cX$, we define the {\emph{error for generation with breadth}} for the learner $\{\generator_n \colon \cX^n \to \cG\}_{n \in \N}$
        as 
        \[
            \mathrm{er}(\generator_n(x_1,\dots,x_n))
            =
            \ind\{\supp(\generator_n(x_1,\dots,x_n)) \neq K \setminus \{x_1,\dots,x_n\}\}\,,
        \]
        where $\supp(\generator_n(x_1,\dots,x_n))$ is the set of strings $\generator_n(x_1,\dots,x_n)$ can output with positive probability, \ie{}, it is the support of the distribution of outputs of $\generator_n(x_1,\dots,x_n)$.
        The above means that we count each step $t$ as a mistake if the generating algorithm has a positive probability of outputting a string outside of $K$ (\ie{}, hallucination), a zero probability of outputting an unseen element of $K$ (\ie{}, mode collapse), or a positive probability of repeating a seen training example.
        
        \begin{remark}[Generating Examples From the Training Set]\label{rem:generatingFromTrainingSet}
            For generation without breadth, it is important to restrict the generator from outputting elements it has already seen.
            Otherwise, the futile generator, which always outputs the first training sample it sees, achieves generation without breadth.
            This requirement, however, is not important for generation with breadth:
                any generator $\generator$ that generates with breadth without repeating training examples can be converted to one $\generator'$ that generates with breadth and repeats the training examples and vice versa.\footnote{For instance, $\generator'$ can run $\generator$ with probability $\nfrac{1}{2}$ and with the remaining $\nfrac{1}{2}$ probability output a training sample selected uniformly at random.
                Given $\generator'$, $\generator$ can be implemented by rejection sampling as follows:
                    repeatedly execute $\generator'$ until it generates an unseen element $x$ and output $x$.
                }
                Hence, all of our results hold with either notion of generation with breadth.
        \end{remark}

        Our main result shows a separation between the rates achievable for generation with and without breadth by any generating algorithm for which $\mop{}(\cdot)$ is decidable.

        \begin{theorem}\label{mainthm:gen:mop}
                Let $\mathfrak{G}$ be the set of all generating algorithms $(\generator_n)$ for which $\mop{}(\cdot)$  is decidable (\cref{def:mop,def:mopAlgo}).
                For every collection of countably many languages $\cL$ that is non-trivial for generation  (\cref{def:non-trivial-collection-generation}) and not identifiable in the limit:
                \begin{itemize}
                    \item No generating  algorithm in $ \mathfrak{G}$  generates with breadth from $\cL$ at any rate; and  
                    \item There is a generating algorithm in $\mathfrak{G}$ that generates consistently without breadth from $\cL$ at rate $e^{-n}.$ 
                    Conversely, no generating algorithm {(even outside of $\mathfrak{G}$)} can generate
                    at a rate faster than $e^{-n}.$
                \end{itemize}
                Further, for any collection of countably many languages $\cL$ that is non-trivial for generation (\cref{def:non-trivial-collection-generation})  and \emph{identifiable} in the limit, and for any $g(n) = o(n),$ there is a generating algorithm in $ \mathfrak{G}$ that generates with breadth from $\cL$ at rate $e^{-g(n)}$. 
                 Conversely, no generation algorithm can generate consistently at a rate faster than $e^{-n},$ even without the breadth requirement.
        \end{theorem}
            Thus, while generation without breadth is achievable for any countable collection of languages (whether it is identifiable or non-identifiable), generators in $\mathfrak{G}$ can only generate with breadth from identifiable collections -- which are a very restricted subset of all languages \citep{gold1967language,angluin1980inductive,kleinberg2024language}.
            It remains to discuss which types of generators $\mop(\cdot)$ is decidable for, and we present a large family in the next section.
            Meanwhile, due to \cref{ex:decidableMOP}, it is already clear that  \cref{prop:mop_decidable} applies to present-day LLMs.
            The proof of this result appears in \Cref{sec:gen:mop}; see \cref{fig:outline:mainthm:gen:mop} for an outline of the proof.
           
            Our negative result leaves several interesting questions open which we already discussed in \cref{sec:takeawaysOpenQuestions}.

    \subsubsection{A Family of Generators for which \texorpdfstring{$\mop(\cdot)$}{MOP} Is Decidable}
        \cref{ex:decidableMOP} already shows that $\mop{}(\cdot)$ is decidable for many existing language models.
        Next, we show that $\mop{}(\cdot)$ is decidable under even fewer restrictions on the generator $\generator$ -- informally, we will allow for \textit{any} generator which generates text token-by-token.
        \begin{definition}[Token-by-Token Generators]\label{def:tokenByToken}
            Token-by-token generators $\generator$ are parameterized by randomized Turing machines $M$.
            $M$ can be randomized and halts on all inputs.
            Given $M$, the corresponding token-by-token generator $\generator_M$ generates outputs as follows: for each $t\in \N$, 
            \begin{enumerate}[itemsep=0pt]
                \item Let $w_1w_2\dots w_{t-1}$ be the tokens generated so far.
                \item Let $A_{t}$ be any auxiliary information generated so far\textit{, where $A_1$ is the empty string}.
                \item Generate $(s_t,A_{t+1})$ by running $M$ with input $w_1w_2\dots w_{t-1}$ and $A_t$.
                \item \mbox{If $s_t=\eos{}$ (\ie{}, end of string), then output $s_1\dots s_t$ and halt; otherwise proceed to iteration $t+1$}.
            \end{enumerate}
        \end{definition}
        Note that token-by-token generators are a very powerful class: for instance, any distribution over $\Sigma^*$ for some finite alphabet $\Sigma$ admits a token-by-token generator by the Bayes rule.
        That said, of course, one can also construct non-token-by-token generators.
        
        We show that $\mop{}(\generator)$ is decidable for all token-by-token generators.

        \begin{restatable}[]{theorem}{thmMOPDecidable}\label{prop:mop_decidable}
            For any token-by-token generator $\generator$, $\mop{}(\generator)$ is decidable.
        \end{restatable}
        Next, we demonstrate that token-by-token generators capture several interesting language models. 
        First, the family of token-by-token generators captures existing large language models (LLMs):
                    for instance, to simulate an LLM $L$, we define the next token predictor $M$ as a Turing machine that simulates $L$ on the provided string until $L$ generates one new token.
             Further, since we do not place computational restrictions on $M$, $M$ can also simulate interactions between LLMs or auxiliary systems that select a suitable LLM to respond depending on the request--a strategy that has led to recent advances in text generation \citep{schick2023toolformer,mosaic2024DRBX,jiang2024mixtralexperts,willknight2024DBRX}. 
             Finally, due to a reduction to the halting problem, there are some generators for which $\mop{}(\cdot)$ is undecidable and give an explicit example in \cref{sec:undeciable}.

    \begin{remark}[Noisy Membership Oracle]
        A supposedly weaker requirement than the decidability of $\mop{}(\cdot)$ is the existence of a \textit{noisy oracle} that, given a string $x$, correctly (and in finite time) decides the membership of $x$ into $\supp(\generator)$ with a probability at least $\nfrac{2}{3}$.
        However, due to the folklore result that $\textsf{BPP}\subseteq \textsf{EXP}$ \citep{arora2009computational}, a noisy oracle is equivalent to the decidability of $\mop{}(\cdot)$. %
    \end{remark}

        \subsubsection{Results for Generation with Breadth in the Limit}
            In this section, we state the implications of our techniques for generation with breadth in the adversarial or online setting of \citecustom{gold1967language} and \citecustom{angluin1979finding, angluin1980inductive}.
            \begin{restatable}[]{theorem}{thmImpossibleGenerationBreadth}\label{mainthm:gen:limit}
                
                For every non-identifiable collection of countably many languages $\cL$, no generating algorithm, for which $\mop{}(\cdot)$ (\cref{def:mop,def:mopAlgo}) is decidable, can generate with breadth from $\cL$ in the limit. %
                If $\cL$ is identifiable, then there is a generator $\generator$ (for which $\mop{}(\generator{})$ is decidable) that generates with breadth from $\cL$.
            \end{restatable}
                This result makes important progress on a question left open by \citecustom{kleinberg2024language} for a fairly large family of generators, which includes all iterative generators due to \cref{prop:mop_decidable}. 
                In particular, $\mop{}(\cdot)$ is decidable for the generation algorithm of \citecustom{kleinberg2024language} (since it is deterministic and the unique element it outputs can be computed by executing the algorithm) and, hence, the above result shows that \citecustom{kleinberg2024language}'s algorithm cannot generate with breadth in the limit from any non-identifiable collection.
                Further, in \cref{sec:ApproxConsBreadth} we strengthen this result by showing that even a relaxed notion of generation with breadth remains unreachable for a large class of generators.  
                The proof of this result can be found in \Cref{sec:proofof:mainmainthm:gen:limit}.

    \subsection{Results for {Relaxations of Consistent} Generation with Breadth} %
    \label{sec:ApproxConsBreadth}\label{sec:results:genBreadth:robust}
        In this section, we study {two relaxations} of generation with breadth, which we call unambiguous generation {and approximate breadth}, and ask: \textit{Is there a generator that unambiguously generates from a non-identifiable collection? {Can generators generate with approximate breadth from non-identifiable collections?}}

        We recall that, in this section, we will allow the generator to repeat examples in the training data. We make this choice to simplify the notation. Like all of our results with breadth, this choice is not crucial, and all of the results have analogs where the generator does not repeat training examples (\cref{rem:generatingFromTrainingSet}). 

        \subsubsection{{Relaxation 1: Unambiguous Generation}}
        We refer the reader to \cref{sec:intro:relaxation} for a discussion and motivation of the definition for unambiguous generation, which we restate below.
        \defUnambiguousGenIntro* 
        \noindent This notion is a significant relaxation of generation with breadth that we considered so far (see \cref{sec:results:genBreadth}):
        Not only does it allow the generator to hallucinate certain strings not in the target $K$ and omit strings actually in $K$ for arbitrarily long, the number of hallucinations and omissions can be very large and, depending on the structure of the language collection $\cL$, even arbitrarily large.
        
        Surprisingly, even this very weak notion of ``generation with breadth'' turns out to be unachievable by a very large family of generators.
        Concretely, it is unachievable by any generator for which $\mop{}(\cdot)$ is decidable and that satisfies the natural property that it stabilizes after a finite time. We state the formal notion of stability below.
        \defStabilityIntro*
        \noindent Before turning to our formal result, we need to construct the error function that defines unambiguous generation, and we use the natural choice:
        for a language $K$ and examples $x_{i_1},\dots,x_{i_n} \in \cX$, we denote  and we define the {\emph{error for unambiguous generation}} for the generating algorithm $\{\generator_n \colon \cX^n \to \cG\}_{n \in \N}$ 
        on input $S_n = \{x_{i_1},\dots,x_{i_n}\}$ as\footnote{Recall that $\cG$ is the set of (randomized) Turing machines that do not take any input and output one element from $\cX$ (on each execution).}
        \begin{equation}\label{eq:unambiguous-error}
            \mathrm{er}(\generator_n(S_n))
                =
                \ind\inbrace{
                    \abs{\supp(\generator_n(S_n)) \triangle {K}}
                    < 
                    \min_{L\in \cL\colon L\neq K}
                    \abs{\supp(\generator_n(S_n))  \triangle {L}} \nonumber
                }\,,
        \end{equation} 
            where $\supp(\generator_n(S_n))$ is the set of strings $\generator_n(S_n)$ can output with positive probability.
            Similar to the case of identification and generation, we say that
            an algorithm achieves unambiguous generation for a collection $\cL$ at some rate $R,$ where
            $R\colon \N \rightarrow \R_{\geq 0}, R\downarrow 0,$ if for any valid
            distribution $\cP$ with respect to $\cL$ there are $c,C > 0$
            so that $\E_{x_{i_1},\ldots,x_{i_n} \sim \cP^n} \insquare{\mathrm{er}(\generator_n(x_{i_1},\dots,x_{i_n}))} \leq C\cdot R(c\cdot n)$. The following result shows that this notion of 
            generation is not achievable, for a large and natural 
            class of generating algorithms. 
        \begin{theorem}[Impossibility of Unambiguous Generation]\label{thm:impossibility:robust}
            For every non-identifiable collection of countably many languages $\cL$, no stable generating algorithm, for which $\mop{}(\cdot)$ (\cref{def:mop,def:mopAlgo}) is decidable, can unambiguously generate from $\cL$ at any rate.
        \end{theorem}
        Note that while this result has a benign requirement that the generator is stable, it already considerably extends our main result \cref{mainthm:gen:mop}, since any generator that achieves breadth must be stable -- otherwise, its support cannot settle on the target language $K$.
        {(To be precise, \cref{mainthm:gen:mop} required generators to not repeat their training examples, but this requirement is not crucial and any generator that does repeat its training examples can be converted into one that does not repeat its training examples, and vice-versa; see \cref{rem:generatingFromTrainingSet}.)}

        In addition to \cref{thm:impossibility:robust}, we also prove its analog in the online setting -- significantly extending our earlier impossibility result in the online setting (\cref{mainthm:gen:limit}).
        Before stating the result in the online, we introduce
         unambiguity in the limit, which is a natural counterpart to its statistical definition:
            \begin{itemize}
                \item A generating algorithm $\generator=(\generator_n)$ is said to be unambiguous for a collection $\cL=\inbrace{L_1,L_2,\dots}$ if, for any $K\in \cL$ and enumeration $x_{i_1},x_{i_2},\dots$ of $K$, there is an $n_0\geq 1$, such after seeing $n\geq n_0$ elements $S_n = x_{i_1},\dots,x_{i_n}~$,
                \[\abs{\supp(\generator_n(S_n)) \triangle K}
                            < 
                            \min_{L\in \cL\colon L\neq K}
                            \abs{\supp(\generator_n(S_n)) \triangle L}.
                            \]
            \end{itemize}
        \begin{restatable}[Impossibility of Unambiguous Generation in the Limit]{theorem}{thmUnambiguousLimit}\label{thm:online:impossibility:robust}
            For every non-identifiable collection of countably many languages $\cL$, no generating algorithm stable in the limit for which $\mop{}(\cdot)$ (\cref{def:mop,def:mopAlgo}) is decidable can unambiguously generate from $\cL$ in the limit.
        \end{restatable}
        The proofs of \cref{thm:impossibility:robust,thm:online:impossibility:robust} appear in \cref{sec:proofof:thm:online:impossibility:robust} and \ref{sec:proofRobust-1}, respectively.  
        To develop some intuition, we recommend reading the proof of \cref{thm:online:impossibility:robust} before the proof of \cref{thm:impossibility:robust}.

        \subsubsection{{Relaxation 2: Consistent Generation with Approximate Breadth}}
            {Next, we study generation with approximate breadth.
            We refer the reader to \cref{sec:intro:relaxation} for a discussion of the definition of generation with approximate breadth, which we restate below.}
            \defApproxBreadth*
            \noindent Informally, generation with approximate breadth which, informally, requires that the generating algorithm is consistent and puts zero mass only on \emph{finitely} many points of the target language $K$. This is also a weakening of generation with breadth -- since it allows the generator to miss elements from $K$ infinitely often, and {seems} to be incomparable to the notion of unambiguous generation studied in the previous section.
            Indeed, on the one hand, approximate breadth requires the generator to be consistent while unambiguous generation does not make this requirement.
            On the other hand, unambiguous generation requires the generator to be a ``better'' generator for $K$ than for any language $L\neq K$, approximate breadth does not require this.

            Our approach to get this result follows the same high-level
        idea with \cref{thm:impossibility:robust}. 
        The error of a generating algorithm in this setting, based on \Cref{def:gen:oneSidedRobust}
        is
        \begin{equation}\label{eq:errorMissFinite}
                        \mathrm{er}\inparen{\generator_n} = \ind\inbrace{\supp(\generator_n)\not\subseteq K \text{ or } \abs{K\setminus \supp(\generator_n)} = \infty}
        \end{equation}
        The formal statement is below.

        \begin{restatable}[Impossibility of Approximate Generation]{theorem}{thmApproxBreadthStat}\label{thm:approxBreadth:statistical}
               For every non-identifiable collection of countably many languages $\cL$, no stable generating algorithm, for which $\mop{}(\cdot)$ (\cref{def:mop,def:mopAlgo}) is decidable, can generate from $\cL$ {with approximate breadth} according to \cref{def:gen:oneSidedRobust}, at any rate. 
        \end{restatable}
        \noindent Using the tools we developed for the setting of unambiguous generation, we can show the following result which transforms a learner that works in the statistical setting into a learner that works in the online setting.
        {Then the result follows due to the following result in the online setting.}
        \begin{restatable}[Impossibility of Approximate Generation in the Limit]{theorem}{thmApproxBreadthOnline}\label{thm:online:impossibility:robustOneSided}
               For every non-identifiable collection of countably many languages $\cL$, no generating algorithm stable in the limit, for which $\mop{}(\cdot)$ (\cref{def:mop,def:mopAlgo}) is decidable, can generate from $\cL$ in the limit {with approximate breadth} according to \cref{def:gen:oneSidedRobust}. 
        \end{restatable} 
        \noindent The proof of \cref{thm:online:impossibility:robustOneSided}, like the proof of \cref{mainthm:gen:limit} is also by a contradiction to the non-identifiability of $\cL$.
        The difference is that in this proof we also need to identify the finitely many elements of $K$ ``missed'' by $\generator{}$.
        Any generator that aims to achieve breadth must exhibit stability; otherwise, its support cannot converge to the target language \( K \).  (Note that although \cref{infthm:breadth:mop} assumes the generator does not repeat training examples, any generator that repeats them can be transformed into one that does not and vice-versa. See \cref{rem:generatingFromTrainingSet} for details.)

        {The proofs of \cref{thm:approxBreadth:statistical,thm:online:impossibility:robustOneSided} appear in \cref{sec:proofof:thm:approxBreadth:statistical} and \ref{sec:proofof:thm:online:impossibility:robustOneSided}, respectively.  
        To develop some intuition, we recommend reading the proof of \cref{thm:online:impossibility:robustOneSided} before the proof of \cref{thm:approxBreadth:statistical}.}

    \subsection{Further Results for Identification}
    \label{sec:results:identification:additional}
        In this section, we present identification algorithms that achieve \textit{exact} exponential rate when one has some additional structure -- access to a stronger oracle, or a finite collection $\cL$, or a countable collection $\cL$ of finite languages -- or additional information -- negative examples. 

        In \cref{sec:exact-exp-rates-id-subset-oracle}, we allow the identifier to make queries of the form ``is $L_i\subseteq L_j$?''
        Next, in \cref{sec:exact-exp-rates-id-finite-cL}, we consider generation from collections $\cL$ containing finitely many languages $\cL=\inbrace{L_1,L_2,\dots,L_k}$.
        (Note that each language in $\cL$ can still be infinite.)
        Finally, in \cref{sec:exact-exp-rates-pos-neg-ex}, in addition to positive examples, we also give the identifier access to negative examples (\ie{}, elements $x\in \cX$ not in the target language $K$).

        \subsubsection{Exponential Rates for Identification Using     Subset Oracle}\label{sec:exact-exp-rates-id-subset-oracle}
        Our first result shows that when $\cL$ satisfies Angluin's condition
        and the learning algorithm has access to a subset oracle for $\cL$ (which answers queries of the form ``$L_i\subseteq L_j$?'')
        then it is possible to achieve exact exponential rates.

        \begin{proposition}\label{prop:exp-rates-id-subset oracle}
            For every countable language collection $\cL$  that satisfies Angluin's condition (\cref{def:angliun-criterion}), there exists a learning algorithm that has access to a subset oracle
            for $\cL$ and identifies $\cL$ at a rate $e^{-n}.$ 
            Formally, a subset oracle is a primitive that, given two indices $i$ and $j$, outputs $\textsf{Yes}$ if $L_i\subseteq L_j$; otherwise, it outputs \textsf{No}.
        \end{proposition}
        Recall that our algorithm that achieves almost exponential rates requires merely black-box access to an algorithm
        that identifies $\cL$ in the limit. In other words, it does not
        make use of the particular structure of the online identification
        algorithm. To achieve exact exponential rates, we make 
        use of a particular algorithm: the one proposed
        by \citecustom{kleinberg2024language}. At a high level, the proof
        consists of the following steps:
        \begin{enumerate}[label=C\arabic*]
            \item First, we show that \citecustom{kleinberg2024language}'s algorithm with access to a subset oracle for $\cL$ can, in fact,  \emph{identify} the target language (see \Cref{sec:identification:subsetOracles}).

            \item Next, we identify a sufficient condition that allows one to use any identification algorithm that identifies $\cL$ in the limit to obtain exponential rates (see \Cref{lem:suff-cond-use-online-id-in-stat}).
            Interestingly, this conversion does not need \textit{any} changes to the identification algorithm.\label[condition]{condition:sufficient-to-convert-online-to-statistical}

            \item Finally, we show that the algorithm of \citecustom{kleinberg2024language} satisfies this condition.
        \end{enumerate}
        The full proof of \cref{prop:exp-rates-id-subset oracle} appears in \Cref{sec:proof-exp-rates-id-subset}.

        \subsubsection{Exponential Rates for Identification of Finite Collections}\label{sec:exact-exp-rates-id-finite-cL}
        We now shift our attention to finite collections
        of languages. 
        \citecustom{gold1967language,angluin1980inductive} showed 
        that all finite collections are identifiable
        in the limit. %
        We show that for such collections we can
        get exact exponential rates, \textit{without} the need of the subset oracle we used in the previous result (\cref{prop:exp-rates-id-subset oracle}).

        \begin{proposition}\label{prop:exp-rates-id-finite-collections}
            For every finite language collection $\cL$, there exists a learning algorithm which identifies $\cL$ at a rate $e^{-n}.$ 
        \end{proposition}
        The proof of this result builds on the proof of \Cref{prop:exp-rates-id-subset oracle}.
        In particular, we show that the algorithm
        of prior work satisfies the sufficient condition
        that allows an algorithm that identifies $\cL$ in the limit to obtain exponential rates (\cref{condition:sufficient-to-convert-online-to-statistical}). 
        The full proof of \cref{prop:exp-rates-id-finite-collections} appears in \Cref{sec:proof-exp-rates-id-finite}.

        \subsubsection{Exponential Rates for Identification of Collections of Finite Languages}\label{sec:exact-exp-rates-id-finite-languages}
        We now move on to a result about identifying countable collections of \emph{finite}
        languages with exactly exponential rates.
        \citecustom{gold1967language} showed 
        such collections are identifiable
        in the limit through a very simple algorithm:
        predict the first language that contains the set of all examples seen so far.
        We show that for such collections we can
        get exact exponential rates.

        \begin{proposition}\label{prop:exp-rates-id-finite-languages}
            For every countable language collection $\cL$ that only contains languages of finite size, there exists a learning algorithm which identifies $\cL$ at a rate $e^{-n}.$ 
        \end{proposition}
        The idea of the proof is simple. Since any valid distribution has finite
        support, for large enough $n$, the sample will contain all the elements
        of the support with probability $1-C\cdot e^{-c\cdot n}.$
        The formal proof of \cref{prop:exp-rates-id-finite-languages} appears in \Cref{sec:proof-exp-rates-id-finite-lang}.

        \subsubsection{Exponential Rates for Identification from Positive and Negative Examples}\label{sec:exact-exp-rates-pos-neg-ex}

        We now shift our attention to a setting {-- introduced by \citecustom{gold1967language} --} where, in addition to an enumeration of the target language $K$, one also receives an enumeration of $\cX\setminus K$. 
        
        Let us first recall the difference between
        the different types of information in
        the two settings. In the case of just positive
        examples {(considered so far)}, the adversary picks a target language
        $K$ from $\cL$ along with an enumeration of
        this language, and presents the examples
        from this enumeration sequentially to the learner.
        In the case of positive and negative examples,
        the adversary again picks a target language
        $K$ from $\cL,$ but now it chooses a \emph{labeled} enumeration of the whole domain $\cX,$ where
        now the label of each element indicates
        whether it belongs to the target language $K$
        or not. It is known that every countable
        collection of languages is identifiable in 
        the limit with positive and negative examples
        \citep{gold1967language}.

        Naturally, we need a different notion
        of a \emph{valid} distribution in this setting.
        We adopt a definition that was proposed
        by \citecustom{angluin1988identifying}.

        \begin{definition}[Valid Distributions Under Positive and Negative Examples \citep{angluin1988identifying}]\label{def:valid-language-distribution-pos-neg}
    A distribution $\cP$ 
     over $\cX \times \{0,1\}$ is valid with respect to a collection of languages $\cL$ if and only if $\supp(\cP_\cX) = \cX$ and there exists some $K \in \cL$
     such that for all $x \in \cX$ it holds that $\Pr_{(X,Y) \sim \cP}[Y = 1 \mid X = x] = \ind\left\{ x \in K\right\}.$
\end{definition}
Our main result in this setting is that every countable collection of languages is identifiable with positive and negative examples at an optimal exponential rate.
\begin{theorem}[Identification with Positive and Negative Examples]\label{thm:identification-positive-negative}
    For every countable collection of languages $\cL$, there exists a learner that identifies $\cL$ at rate $e^{-n}.$ Conversely,
    for every countable collection of languages $\cL$ that is non-trivial for identification, no learner can identify
    $\cL$ at rate faster than $e^{-n}.$
\end{theorem}
The proof of \cref{thm:identification-positive-negative} appears in \Cref{sec:stat-rates-identification-pos-neg}. 
Our proof of this result is inspired by the approach
of \citet{bousquet2021theory}. 
First, we show the exponential rates lower bound by directly
using a result of \citet{bousquet2021theory}.
In order to get the upper bound, we use a black-box transformation from any learner that identifies $\cL$ in the limit, to a learner that achieves exponential rates in the statistical
setting. 

The approach shares similarities to the one with just positive examples (see \Cref{sec:technical-overview}, Paragraph B). 
The crucial reason why we can obtain exactly exponential rates here, instead of almost exponential rates as in the previous setting, is that we can use the negative examples to accurately estimate the correct sizes of the batches we  use, instead of having to use ``guesses'' of increasing size as we did in the setting of just positive
examples.

To give {a} more concrete comparison to the binary
classification setting of \citet{bousquet2021theory}, 
let us first explain some results from this work.
\citet{bousquet2021theory} define the following
infinite game sequential, which is appropriately rephrased using
the terminology from our work looks as follows:
\begin{itemize}
    \item In every round, the adversary presents a word $x_t \in \cX$ to the learner.

    \item Subsequently, the learner predicts a label
    from $\{0,1\}$ for this word, denoted by $\hat y_t.$

    \item Then, the adversary reveals the true
    label $y_t$ to the learner.
\end{itemize}
The only constraint on the adversary is that at any given point $t \in \N,$ there has to be some language $K \in \cL$ such that $y_{t'} = \ind\{x_{t'} \in K\},$ for all $t' \leq t.$ In other words, the choices of the labels have to be consistent with some language $K \in \cL.$ Crucially, the consistent language does not need to be fixed in advance and it can keep changing throughout the interaction.  In their setting, the learner ``wins'' the game if it makes only finitely many mistakes.  They provide a necessary and sufficient condition on the structure of $\cL$ which determines the existence of a winning strategy for the learner: the learner can win this game if and only if $\cL$ does not have an infinite Littlestone tree (see \Cref{defn:litt}). Interestingly, this condition does not capture the existence of a winning strategy for the learner
in Gold's setting: we have constructed a language family $\cL$ which has an infinite Littlestone tree, but it is identifiable in the limit from positive and negative examples.
Perhaps more surprisingly, this language is identifiable even with just positive examples.
The construction appears in \Cref{sec:comparison}.

\section{Organization of the Rest of the Paper}
{We next describe the organization of the rest of the paper. }
\begin{itemize}
    \item The proofs of \Cref{sec:results:iden} (statistical rates for identification and generation) can be found in \Cref{sec:proofof:mainthm:statisticalRates:iden gen}. The proof for the identification universal rates appears in \Cref{sec:proofof:mainthm:statisticalRates:iden} and for generation in \Cref{sec:proofof:mainthm:statisticalRates:gen}. 

    \item The proofs of \Cref{sec:results:genBreadth} can be found in \Cref{sec:proofof:generationWithBreadth}. In \cref{sec:proofof:prop:mop_decidable}, we discuss the decidability of $\mop(\cdot).$ In \cref{sec:gen:mop} we provide our main result that generation with breadth is not possible for generating algorithms for which $\mop(\cdot)$ is decidable. Finally, in \Cref{sec:proofof:mainmainthm:gen:limit}, we see the implications of this result for generation in the limit.

    \item  The proofs of \Cref{sec:ApproxConsBreadth} appear in \Cref{sec:proofRobust}. 
    In \Cref{sec:proofRobust-2} we give the proof of the result in the online setting and
    in \Cref{sec:proofRobust-1} the proof
    of the result in the statistical setting.

    \item The proofs of \Cref{sec:results:identification:additional} appear in \Cref{sec:addResults}. 
    In \Cref{sec:proof-exp-rates-id-subset} we give the proof 
    of exponential rates for identification using 
    a subset oracle, in \Cref{sec:proof-exp-rates-id-finite} the proof of exponential rates for identification of finite collections using 
    a membership oracle, and in \Cref{sec:proof-exp-rates-id-finite-lang} the proof of exponential rates for identification of countable collections of finite languages.
    The proof for the identification rates with positive and negative examples appears in \Cref{sec:stat-rates-identification-pos-neg}.
\end{itemize}

\section{Proofs from \texorpdfstring{\cref{sec:results:iden}}{Section 3.1} (Rates for Identification and Generation)}
\label{sec:proofof:mainthm:statisticalRates:iden gen}

\subsection{Proof of \texorpdfstring{\Cref{thm:dichotomy-identification-positive}}{Theorem 3.1} (Rates for Identification) } 
\label{sec:proofof:mainthm:statisticalRates:iden}
 In this section, we give the full proof of \Cref{thm:dichotomy-identification-positive}; see \cref{fig:outline:thm:dichotomy-identification-positive} for an outline.
 As we alluded to before, the first step in the proof
is to show that all non-trivial collections 
are not learnable at rate faster than $e^{-n}.$

\begin{figure}[h]
    \centering
    \begin{tikzpicture}[node distance=2.5cm, auto]
    \node[myNodeNarrow] (thm) at (0,0) {\cref{thm:dichotomy-identification-positive}};
    \node[myNode] (subres1) at (-5.5cm, -2cm) {\small For non-trivial collections\\ \small $e^{-n}$ is the best possible rate};
    \node[myNode] (subres2) at (0cm, -2cm) {\small Identification in the limit\\  \small $\implies$ Almost exponential rate};
    \node[myNode, text width=6cm] (subres3) at (6cm, -2cm) {\small \mbox{$\cL$ is not identifiable in the limit $\implies$} \mbox{ $\cL$ cannot be identified at any rate}};

    \node[myNodeNarrow] (lem51) at (-5.5cm, -4cm) {\small \cref{lem:exp-rates-lower-bound-ident-pos}};
    
    \node[myNodeNarrow] (lem55) at (0cm, -4cm) {\small \cref{lem:almost-exp-rates-ident-pos-examples}};
    \node[myNodeNarrow] (lem52) at (-1.5cm, -6cm) {\small \cref{prop:support-appears-in-countable-samples}};
    \node[myNodeNarrow] (lem53) at (1.5cm, -6cm) {\small \cref{prop:quantile-bound-identification}};
    
    \node[myNodeNarrow] (lem58) at (6cm, -4cm) {\small \cref{lem:no-rate-identification-positive}};
    \node[myNodeNarrow, text width=2cm] (lem54) at (4.25cm, -6cm) {\small \cref{lem:post-processing-same-index}};
    \node[myNodeNarrow] (lem57) at (7.25cm, -6cm) {\small \cref{thm:angluin-limit-statistical-online-pos-equivalent}};
    \node[myNodeNarrow] (lem56) at (7.25cm, -8cm) {\small \cref{thm:angluin-limit-statistical-online-pos-not-convergence}};
    
    \draw[-stealth, line width=0.5mm] (thm) -- (subres1.north);
    \draw[-stealth, line width=0.5mm] (thm) -- (subres2);
    \draw[-stealth, line width=0.5mm] (thm) -- (subres3.north);

    \draw[-stealth, line width=0.5mm] (subres1) -- (lem51);
    
    \draw[-stealth, line width=0.5mm] (subres2) -- (lem55);
    \draw[-stealth, line width=0.5mm] (lem55) -- (lem52);
    \draw[-stealth, line width=0.5mm] (lem55) -- (lem53);

    \draw[-stealth, line width=0.5mm] (subres3) -- (lem58);
    \draw[-stealth, line width=0.5mm] (lem58) -- (lem54);
    \draw[-stealth, line width=0.5mm] (lem58) -- (lem57);
    \draw[-stealth, line width=0.5mm] (lem57) -- (lem56);

    \end{tikzpicture}
    \caption{Outline of Proof of \cref{thm:dichotomy-identification-positive}}
    \label{fig:outline:thm:dichotomy-identification-positive}
\end{figure}

\begin{lemma}[Exponential Rate Is Best Possible for Identifying Any Non-trivial Collection]\label{lem:exp-rates-lower-bound-ident-pos}
    Let $\cL$ be a non-trivial collection of countably many languages. Then, 
    for any identification algorithm $\cA = \{h_n\}_{n \in \N}$ there
    exists a valid distribution $\cP$ such that 
    $\E[\er(h_n)] \geq e^{-2n},$ for 
    infinitely many $n \in \N.$
\end{lemma}

\begin{proof}
    Since $\cL$ is non-trivial, there exist two distinct languages $L_{i}, L_{j} \in \cL$
    and $x \in \cX$ such that $x \in L_{i}, x \in L_{j}.$ Let $\cP_{L_{i}}, \cP_{L_{j}}$ be valid distributions for $L_{i}, L_{j}$
    that place at least $\nicefrac{1}{2}$ on $x$ and if the
    languages have more elements, they spread the remaining mass 
    on the rest of the elements arbitrarily; otherwise
    they put the remaining mass on $x.$ 
    Notice that  since $L_{i} \neq L_{j}$ at least one of them has at least one more element other than $x.$
    For any $n \in \N,$ under both distributions,  with probability at least ${2^{-n}}$ the 
    algorithm will only see the element $x$ appearing in
   the samples. Let $\cE_n$ be that event
    and condition on it.
    Notice that
    \[
        \Pr\insquare{L_{h_n(x,\ldots,x)} = L_{i}\mid \cE_n} +  \Pr\insquare{L_{h_n(x,\ldots,x)} = L_{j}\mid \cE_n} \leq 1 \,,
    \]
    where the probability is with respect to the randomness
    of the identification algorithm.
    Thus, we have 
    that either $ \Pr\insquare{L_{h_n(x,\ldots,x)} \neq L_{i}\mid \cE_n} \geq \nicefrac{1}{2}$
    or $ \Pr\insquare{L_{h_n(x,\ldots,x)} \neq L_{j}\mid \cE_n} \geq \nicefrac{1}{2}$ for each $n\in \N$.
    Hence, by the pigeonhole principle, for at least one of $L_{i}, L_{j},$
    the previous inequality holds for infinitely many $n \in \N.$
    Assume, without loss of generality, that it holds for $L_{i}$ and let $\hat{N}$ denote the set of $n \in \N$ for which it holds.
    Then, for each $n \in \hat{N}$, we have that
    \begin{align*}
               \E_{X_1,\ldots,X_n \sim \cP^n_{L_{i}}}[\er(h_n(X_1,\ldots,X_n))] &= \Pr_{X_1,\ldots,X_n \sim \cP^n_{L_{i}}}[L_{h_n(X_1,\ldots,X_n) }\neq L_{i}] \\
               &\geq \Pr_{X_1,\ldots,X_n \sim \cP^n_{L_{i}}}[L_{h_n(X_1,\ldots,X_n)} \neq L_{i} \mid \cE_n] \cdot \Pr_{X_1,\ldots,X_n \sim \cP^n_{L_{i}}}[\cE_n]\\
               &\geq \frac{1}{2^n} \cdot \Pr[L_{h_n(x,\ldots,x)} \neq L_{i} \mid \cE_n] \tag{by the definition of $\cE_n$} \\
               &\geq \frac{1}{2^{n+1}} \,, \tag{due to the assumption on $L_{i}$}
    \end{align*}
    which concludes the proof.
 
\end{proof}   
We now move on to the (almost) exponential rates upper bound for identification.
This will be done via a transformation from learners
that achieve identification in the limit in Gold's model \citep{gold1967language} to learners that achieve
(almost) exponential rates in our setting.
The first step in this result is to show that when we draw countably many samples
from $\cP$ all the elements of the target language
will appear in the sample.

\begin{proposition}[Infinite Draws Are Enumerations]\label{prop:support-appears-in-countable-samples}
    Let $\cP$ be a probability distribution supported
    on a countable domain and $\{X_i\}_{i \in \N},$ where every
    $X_i$
    is \iid{} from $\cP.$ Then, \[\Pr_{\{X_i\}_{i \in \N} \sim \cP^\infty}[\supp(\cP) = \cup_{i \in \N} \{X_i\}]  = 1.
    \]
\end{proposition}

\begin{proof}
    For the direction $\Pr_{\{X_i\}_{i \in \N} \sim \cP^\infty}[\supp(\cP) \supseteq \cup_{i \in \N} \{X_i\}]$
    notice that for any element $x \notin \supp(\cP)$ 
    \begin{align*}
        \Pr_{\{X_i\}_{i \in \N} \sim \cP^\infty}[x \in \cup_{i \in \N} \{X_i\}] \leq \sum_{i \in \N} \Pr_{X_i \sim \cP}[x = X_i] = 0 \,.
        \yesnum\label{eq:support-appears-in-countable-samples:1}
    \end{align*}
    Hence,
    \begin{align*}
        \Pr_{\{X_i\}_{i \in \N} \sim \cP^\infty}[\supp(\cP) \supseteq \cup_{i \in \N} \{X_i\}] 
        ~~&=~~ 1 -\Pr_{\{X_i\}_{i \in \N} \sim \cP^\infty}[\exists x \notin \supp(\cP), x \in  \cup_{i \in \N} \{X_i\}]\\
        &\geq~~ 1 - \sum_{x \notin \supp(\cP)}\Pr_{\{X_i\}_{i \in \N} \sim \cP^\infty}[x \in  \cup_{i \in \N} \{X_i\}]\\
        &\Stackrel{\eqref{eq:support-appears-in-countable-samples:1}}{=}~~ 1 \,.
    \end{align*}
   For the other direction, \ie{}, $\Pr_{\{X_i\}_{i \in \N} \sim \cP^\infty}[\supp(\cP) \subseteq \cup_{i \in \N} \{X_i\}]$, notice that for any element $x \in \supp(\cP)$, 
   $\Pr_{X\sim \cP}[X=x]$ is a positive constant $p_x > 0$ and let $\inparen{\cE_n \coloneqq \{X_n = x\}}_{n \in \N}$ be a
   sequence of events. Notice that these
   events are independent and that
   \[
        \sum_{n \in \N} \Pr[\cE_n] = \sum_{n \in \N} p_x = \infty \,.
   \]
   Hence, we can apply the second Borel--Cantelli
   lemma {(see \Cref{lem:second-borel-cantelli})}
   and get that
   \[
        \Pr\insquare{\limsup_{n\rightarrow \infty} \cE_n } = 1\,. \yesnum\label{eq:support-appears-in-countable-samples:2}
   \]
   In other words, the element $x$ will appear
   infinitely often in the stream $X_1,\ldots,$
   with probability one.
   Therefore,
    \begin{align*}
        \Pr_{\{X_i\}_{i \in \N} \sim \cP^\infty}[\supp(\cP) \subseteq \cup_{i \in \N} \{X_i\}] 
        ~~&=~~ 1 - \Pr_{\{X_i\}_{i \in \N} \sim \cP^\infty}[\exists x \in \supp(\cP)\colon x \notin \cup_{i \in \N} \{X_i\}] \\
        &\geq~~ 1 - \sum_{x \in \supp(\cP)} \Pr_{\{X_i\}_{i \in \N} \sim \cP^\infty}[x \notin \cup_{i \in \N} \{X_i\}]\\
        &\Stackrel{\eqref{eq:support-appears-in-countable-samples:2}}{=}~~ 1 \,.
    \end{align*}     
\end{proof}
Next, we show that for any algorithm $\cA$ that identifies
the target language in the limit in the adversarial (online)
setting and for
any valid distribution $\cP$ there is some
number $t^*  \coloneqq  t^*(\cA, \cP) \in \N$ such that, when we draw 
$t^*$ many \iid{} samples from $\cP$ and use them to simulate
the adversarial game with $\cA,$ it will identify
the target language with probability at least $\nfrac{6}{7}.$
We denote the time of the last mistake of the algorithm
$\cA = \{h_n\}_{n \in \N}$ on a sequence $x_1,x_2,\ldots$ by $T_\cA(x_1,x_2,\ldots),$ \ie{},
\[
    T_\cA(x_1,x_2,\ldots) = \inf\inbrace{n_0 \in \N\colon L_{h_n(x_1,\ldots,x_n)} \neq K\,,~ \forall n \geq n_0} \,.
\]

\begin{proposition}[Tail Bound on the Distribution of Last Mistake]\label{prop:quantile-bound-identification}
    Fix any countable collection of languages $\cL$ and let $K \in \cL$ be the true language.
    For any algorithm $\cA = \{h_n\}_{n \in \N}$ that identifies $\cL$ in
    the limit in the online setting from positive examples
    and 
    any valid distribution
    $\cP$ for $K$ (\cref{def:valid-language-distribution}),
    there exists a number $t^* \in \N$ such that
    \[
    \Pr_{\{X_i\}_{i \in \N}\sim \cP^{^\infty}}
    [T_\cA(X_1,X_2,\ldots) \leq t^*] \geq \frac{6}{7}\,.
    \]
\end{proposition}

\begin{proof}
    Let $X_1,X_2,\ldots,$ be a countable \iid{} sample from $\cP.$
    From \cref{prop:support-appears-in-countable-samples} we
    get that this sample is a valid input to $\cA$ 
    since, with probability one, it
    consists only of elements of $K$ and eventually 
    every element of $K$ appears in this sequence. 
    Consider the execution of $\cA$ on prefixes of the sequence
    and denote by $T_{\cA}  \coloneqq  T_{\cA}(X_1,X_2,\ldots)$ the time it made its last mistake. 
    We have that $\Pr_{\{X_i\}_{i \in \N}\sim \cP^{\infty}}[T_\cA \in \N] = 1.$
    Thus, 
    \[
        \lim_{t \rightarrow \infty}\Pr_{\{X_i\}_{i \in \N}\sim \cP^{\infty}}[T_\cA(X_1,X_2,\ldots) \geq t] = 0\,.
    \]
    Thus, as required, there exists some $t^* \in \N$ such that
    \[
         \Pr_{\{X_i\}_{i \in \N}\sim \cP^{\infty}}[T_\cA(X_1,X_2,\ldots) \geq t^*] \leq \frac{1}{7}\,.
    \]
\end{proof}
Thus far we have shown that for every valid distribution $\cP$
there exists some number $t^* \in \N$ so that if we simulate
the online learning process with $t^*$ samples \iid{}
from $\cP$, then the algorithm identifies the true language
$K$ correctly with probability at least $\nfrac{6}{7}.$ However, 
the number $t^*$ depends on the distribution $\cP$, and hence
we cannot immediately devise a learning strategy based on it.
To make the exposition easier to follow, let us first assume
that we do know $t^*$; we will shortly relax this assumption.
For $n$ sufficiently large consider the following algorithm:
\begin{itemize}
    \item We split the input sequence into $\nfrac{n}{t^*}$
    non-overlapping batches, where the $i$-th batch
    consists of the elements $X_{(i-1)\cdot t^* + 1},\ldots,X_{i\cdot t^*}.$

    \item We use each of these sequences as an input
    to a copy of $\cA$ and we get $\nfrac{n}{t^*}$
    many predictors $\left\{h^i_n\left(X_{(i-1)\cdot t^* + 1},\ldots,X_{i\cdot t^*}\right)\right\}_{i \in [n/t^*]}.$

   \item Since these predictors
    might be outputting
    different indices (descriptions) of the same
    language, we find the smallest indexed language
 the output of each classifier
    can be mapped to. 
    In other words, if $j_i$ is the index outputted by the $i$-th batch, we find the smallest number $j' \in \N$ such that $L_{j_i} = L_{j'},$ and we set $j_i \coloneqq j'.$
    Since we only
    have query access to the languages, 
    we can only approximate this step. 
    {In particular, for every $n \in \N,$
    we set $j_i \coloneqq j'$ if $j' \in \N$
    is the smallest number for which $\ind\inbrace{x_{\ell} \in L_{j_i}} = \ind\inbrace{x_{\ell} \in L_{j'}}$, for all $\ell \in [n].$}
    The details are handled in \Cref{lem:post-processing-same-index}.

    \item We predict the index that at least $\inparen{\nfrac{5}{7}}\cdot \inparen{\nfrac{n}{t^*}}$ of the predictors agree upon; if no such language
    exists we output one arbitrarily.
\end{itemize}
Before moving to the general case where $t^*$ is unknown, it
is instructive to explain why the previous approach achieves
exponential rates. Using standard concentration bounds, 
it is not hard to see that with probability at least $1-c\cdot e^{-C\cdot n}$, where $c$ and $C$ are $\cP$-dependent constants, at least a $\nfrac{5}{7}$ 
fraction of the predictors will output an index that describes the true language. 
Conditioned on that event, it is immediate that a $\nfrac{5}{7}$-majority is well-defined and predicting based on it yields the correct answer.

Let us now explain how to handle the actual problem setting, in which
as we mentioned, we do not have knowledge of $t^*.$ 
Let $f\colon \N \rightarrow \N$ be some (very slowly) increasing function
of the input size $n,$ which we will specify shortly.
Given that function, we use the following modified approach, where $t^*$ is 
replaced by $f(n).$ 

\begin{itemize}
    \item We split the input sequence into $\nfrac{n}{f(n)}$
    non-overlapping batches, where the $i$-th batch
    consists of the elements  $X_{(i-1)\cdot f(n)+1},\ldots,X_{i\cdot f(n)}.$

    \item We use each of these sequences as an input
    to a copy of $\cA$ and we get $\nfrac{n}{f(n)}$
    many predictors $\left\{h^i_n\left(X_{(i-1)\cdot f(n) + 1},\ldots,X_{i\cdot f(n)}\right)\right\}_{i \in [n/f(n)]}.$

    \item We use the post-processing approach from \Cref{lem:post-processing-same-index}, that we also explained above.

    \item We predict the index that at least $\inparen{\nfrac{5}{7}}\cdot \inparen{\nfrac{n}{f(n)}}$ of the predictors agree upon; if no such language
    exists we output one arbitrarily.
\end{itemize}
Since
$f(\cdot)$ is increasing, there is some $n_0 \in \N$ such that
$f(n_0) = t^*.$ Thus, for $n \geq n_0$ we can repeat the previous
argument; with probability at least $1- c\cdot e^{-C \cdot \sfrac{n}{f(n)}},$
at least $\inparen{\nfrac{5}{7}}\cdot \inparen{\nfrac{n}{f(n)}}$ of the predictors
will be outputting the correct target language, so taking
the majority vote over them yields the desired result. 
Notice that now we do not achieve exactly exponential rates,
but for every sublinear function $g(\cdot)$ we can achieve
rates $e^{-g(n)}.$ 

We first state and prove the post-processing
lemma to map the outputs of predictors that correspond
to different indices of the target language $K$
to the same index.
For this result, it is useful to 
define the notion of ``projection'' of a language onto a subset of $\cX.$
\begin{definition}[$m$-Projection of a Language]\label{def:projection}
    Let $\cX = \{x_1,x_2,\ldots\}$ be a countable
    domain and let $L \subseteq \cX$
    be a language. For any $m \in \N$
    we denote by 
    $L[m] \coloneqq L \cap \{x_1,x_2,\ldots,x_m\}$
    the projection of the language onto the first
    $m$ elements of the domain.
\end{definition}
The point of the next lemma is the following: in the enumeration of the language collection $\cL$, we allow repetitions of $K$ (as in Gold's model). Hence, when running multiple copies of our identification algorithms, the majority of them will identify $K$; yet we cannot guarantee that they will identify the \textit{same} index for $K$ (due to multiple appearances of $K$ in the enumeration). This lemma guarantees that there exists a sufficiently large prefix of the enumeration of the domain $\cX$ so that the projection of predicted languages will be mapped to the smallest index version of $K$ in $\cL$, which we denote by $L_z$ below.

\begin{lemma}[Post-processing to Map to Lowest-Index Occurrence of $K$]\label{lem:post-processing-same-index}
    Let $\cL = \{L_1,L_2,\ldots,\}$ be a countable
    collection of languages over $\cX = \{x_1,x_2\ldots\}$ and $K \in \cL$.
    Let $z \coloneqq \min \{j \in \N\colon  L_j = K\}$ be the first index at which the target language appears {in $\cL$}.
    Let $\cI = (i_1,\ldots, i_m)$ be a multiset
    of indices and for all $1\leq \ell \leq \max_{j \in 
    [m]} i_j, n \in \N$, let 
    \[
        \hat i^n_j = \min\{1\leq \ell \leq i_j\colon L_{\ell}[n] = L_{i_j}[n]\} \,,
    \]
    {be the index of the first language that
    has the same projection as $L_{i_j}.$
    }
    Then, there exists a number $n_0\coloneqq n_0(K, \cL, \cX)$
    that depends on $K, \cL, \cX$, but not $\cI$, such that
    for all $n \geq n_0$
    \[
        L_{i_j} = K \implies \hat i^n_j = z, \forall j \in [m] \,.
    \]
\end{lemma}
Before we give the formal proof, let
us explain the main idea and the implication of this result. For every
language $L$ that precedes $L_z,$ there
exists some element $x \in \cX$
such that $\ind\{x \in L\} \neq \ind\{x \in L_z\}.$
{This will enable us to detect and remove all languages different from $K$ preceding $L_z$.}
{Then,} by taking projections onto large enough prefixes we indeed map any $L_j = K, j > z,$ to $L_z.$ 
This result will be useful for our constructions that require aggregating outputs from different executions of the algorithm {which, without this post-processing step, can output different indices}.

\begin{proof}[Proof of \Cref{lem:post-processing-same-index}]
    Assume without loss of generality that for some $j\in[m]$ we
    have that $L_{i_j} = K,$ otherwise the statement holds vacuously.
    We will handle the cases $z = 1, z > 1$ separately.

    \paragraph{Case A ($z=1$):} For any 
    $j \in [m]$ for which
    $L_{i_j} = K$ and any $n \in \N$ we have
    that $L_{i_j}[n] = L_{z}[n],$ and since $z = 1$
    this is the first index for which the equality
    holds. Hence, in this case, the claim holds
    with $n_0 = 1.$

    \paragraph{Case B ($z>1$):} 
    Since $L_{z}$ is the first occurrence of $K$ in $\cL,$ for all languages $L_\ell$ with $1 \leq \ell < z$,
    there exists some $x \in \cX$ such that $\ind\{x \in L_j\} \neq \ind\{x \in L_{z}\}$. Let $x_{z_\ell}$ be the
    smallest indexed element of $\cX$ for which the previous holds. Moreover, let $ n_0 = \max_{1\leq \ell < z} z_\ell.$ Notice that for all $n \geq  n_0$ and all $1 \leq j < z,$ holds that $L_j[n] \neq L_{z}[n].$ Furthermore, for all $n \in \N$
    and all $j \in [m]$ such that $L_{i_j} = K$ it holds
    that $L_{i_j}[n] = L_{z}[n].$ Combining these two claims,
    we can deduce that for all $j \in [m]$ such 
    that $L_{i_j} = K$ and for all $n \geq \hat n_0$ the first
    index $i \in \N$ such that $L_{i_j}[n] = L_i[n]$ is indeed $z.$ Notice that $n_0$ depends only
    on the enumeration of $\cL, \cX,$ and the target
    language $K.$
\end{proof}
We are now ready to state and prove the formal result regarding the identification rates of collections that
are identifiable in the limit.
\begin{lemma}[{Reduction From Identification at Almost-Exponential Rate to Online Identification}]\label{lem:almost-exp-rates-ident-pos-examples}
    Let $\cL = \{L_1,L_2,\dots\}$ be {a countable collection of} languages and $g\colon\N \rightarrow \N$
    be a sublinear function. For any algorithm $\cA = \{h_n\}_{n \in \N}$ that identifies $\cL$ in
    the limit in the online setting with positive examples and any valid distribution
    $\cP$ there exists an algorithm $\cA' = \{h'_n\}_{n \in \N}$
    such that for all $n\in \N$
    \[
    \E_{X_1,\ldots,X_n \sim \cP^n}[\mathrm{er}(h'_n(X_1,\ldots,X_n))] \leq c \cdot e^{-C \cdot g(n)}\,.
    \]
\end{lemma}

\begin{proof}
    Let $\cK \coloneqq \{L_{i_1}, L_{i_2},\dots\} \subseteq \cL$ be the set of all languages in $\cL$ that
    correspond to representations of $K,$ \ie{},
    for all $L \in \cK$ it holds that $L = K.$
    First, notice that since
    $g(n) = o(n)$ we can construct some non-decreasing
    function $f\colon \N \rightarrow \N$ with $\lim_{n \rightarrow \infty} f(n) = \infty$ 
    and $\nfrac{n}{f(n)} \geq g(n).$ Let $t^* \in \N$ be a number such that   
    \[
        \Pr_{X_1,\ldots,X_{t^*} \sim \cP^{t^*}}
    [L_{h_{t^*}(X_1,\ldots,X_{t^*})} \in \cK] \geq \frac{6}{7} \,.
    \]
    From \cref{prop:quantile-bound-identification} such a number $t^*$ is guaranteed to exist.
    Hence, there is some $n_0$ such that for all $n \geq n_0$, $f(n) \geq t^*.$ 
    Thus, for all $n \geq n_0$
    \[
        \Pr_{X_1,X_2,\ldots,X_{f(n)} \sim \cP^{f(n)}}\left[L_{h_{f(n)}\left(X_1,\ldots,X_{f(n)}\right)} \in \cK\right] \geq \frac{6}{7} \,.
    \]
    Recall the error of the classifier $h_{f(n)}(\cdot)$ as defined
    in \cref{eq:error-statistical-identification-pos}. We have
    that for all $n \geq n_0$, it holds that 
    $\mathrm{er}\left(h_{f(n)}\right)$ is a Bernoulli random
    variable with $p \leq \nfrac{1}{7}.$ Thus, if we have a collection
    of $\hat t_n  \coloneqq  \nfrac{n}{f(n)}$ such \iid{} random variables, using Hoeffding's bound \citep{Dubhashi_Panconesi_2009} we get that
    \[
        \Pr_{X_1,\ldots,X_n}\left[\frac{1}{\hat t_n}\sum_{i=1}^{\hat t_n} \er\left(h_{f(n)}\left(X_{(i-1) \cdot \hat t_n + 1},\ldots,X_{i\cdot \hat t_n}\right)\right) \geq \frac{2}{7}\right] \leq e^{-\sfrac{2\hat t_n}{49}} \leq e^{-\sfrac{2 g(n)}{49}}\,.
    \]
    Thus, for $n \geq n_0$, at least a $\nfrac{5}{7}$-fraction of the predictors outputs an index
    that corresponds to the target language.
    We condition on that event $\cE^n_0$ for the rest of the proof.
    
    Let $\cI^n = (i_1, \ldots, i_{\hat t_n})$ be the 
    multiset of the indices
    of the languages outputted by the predictors
    in the previous step and $z = \min\{\ell \in \N\colon  L_{\ell} = K\}.$ 
    Then, using \Cref{lem:post-processing-same-index} we know that there exists
    some $\hat n_0 \coloneqq \hat n_0(K, \cL, \cX)$ such that
    for all $n \geq \hat n_0$
    \begin{align*}
        L_{i_j} = K \implies \hat i_j^n = z, \forall j \in [\hat t_n] \,,
    \end{align*}
    where $\hat i_j^n$ is defined in \Cref{lem:post-processing-same-index}. Thus, letting $\hat \cI^n = \left(\hat i_1^n, \ldots, \hat i_{\hat t_n}^n\right)$
    we have that for all $n \geq \max\{n_0, \hat n_0\},$
    at least a $\nicefrac{5}{7}$-fraction of the indices
    in $\hat \cI^n$ are exactly the index $z.$

    Thus, for all $n \geq \max\{n_0, \hat n_0\}$ 
    and conditioned on the event $\cE^n_0,$
    we have that the $\nicefrac{5}{7}$-majority vote
    over the indices in $\hat \cI^n$ corresponds to the
    the first occurrence of $K$ in $\cL.$
    Since $\cE^n_0$ occurs with probability at least
    $1-e^{\sfrac{-2g(n)}{49}}$, this concludes the proof.
\end{proof}

We now move on to the final ingredient we require for the proof
of \Cref{thm:dichotomy-identification-positive}. 
It remains to show that Angluin's condition characterizes the collections of languages that can be identified at an (almost) exponential rate.
First,
we discuss a result from \citecustom{angluin1988identifying}
which our proof builds upon. Let us first briefly
describe the convergence criterion in Angluin's paper.
There is a valid distribution $\cP$ that is supported
over some $K \in \cL$ and the algorithm is presented
with an infinite sequence of \iid{} draws from $\cP.$
After seeing each example, the learner must output 
some $i \in \N$ with the goal being that $L_i = K.$
In that setting, an algorithm learns the target language
if for all but finitely many $n \in \N$ it outputs
the same index $i \in \N$ for which $L_i = K.$
Notice that the learning requirement is two-fold: \textbf{i)}
the learner needs to stabilize, and \textbf{ii)} the index
it predicts needs to correspond to the target language.
In that setting, \citecustom{angluin1988identifying} showed the following result.

\begin{theorem}[Corollary 10 \citep{angluin1988identifying}]\label{thm:angluin-limit-statistical-online-pos-not-convergence}
    Let $\cL$ be a countable collection
    of languages over a countable domain $\cX$
    that does not satisfy {Angluin's condition (\Cref{def:angliun-criterion})}.
    Then, for every learning algorithm $\cA = \{h_n\}_{n \in \N}$ there exists
    a valid distribution $\cP$ supported on some $K \in \cL$ such that, with probability at least $\nfrac{1}{3}$ over
    the \iid{} draw of $\{X_n\}_{n \in \N},$ the learner
    does not identify $K.$
\end{theorem}
In particular, Angluin's result shows that,
with probability at least $\nfrac{1}{3,}$
the learner will either not stabilize to any number $i \in \N$
or it will stabilize to a number $j \in \N$ with $L_j \neq K.$
Our next result provides a strengthening of 
Angluin's result since it shows that
with probability at least $\nfrac{1}{3},$
the learner will, in fact, predict infinitely many
times indices that do not correspond to $K.$

\begin{lemma}\label{thm:angluin-limit-statistical-online-pos-equivalent}
    Let $\cL$ be a countable collection
    of languages over a countable domain $\cX$
    that does not satisfy {Angluin's condition (\Cref{def:angliun-criterion})}.
    Then, for every learning algorithm $\cA = \{h_n\}_{n \in \N}$ there exists
    a valid distribution $\cP$ supported on some $K \in \cL$ such
    that 
    \[
        \Pr_{\{X_i\}_{i \in \N} \sim \cP^\infty } \left[
            \exists ~i_1 < i_2 < i_3 < \ldots ~:~ L_{h_{i_j}(X_1,\ldots,X_{i_j})} \neq K,    \forall j \in \N \right] \geq \frac{1}{3} \,.
    \]
\end{lemma}

\begin{proof}
    Let $\cL$ be a countable collection of languages that does not
    satisfy Angluin's condition. Assume towards contradiction
    that there is learner $\{h_n\}_{n \in \N}$ {that, for any valid distribution $\cP$, with probability $c_\cP < \nfrac{1}{3}$ misidentifies $K$ infinitely often }, \ie{}, for 
    any valid distribution $\cP$ 
    \[
        \Pr_{\{X_i\}_{i \in \N} \sim \cP^\infty } \insquare{
            \exists ~i_1 < i_2 < i_3 < \ldots ~:~ L_{h_{i_j}(X_1,\ldots,X_{i_j})} \neq K,    \forall j \in \N
        } = c_\cP < \frac{1}{3} \,.
        \yesnum\label{def:identification:cp}
    \]
    We will construct a different learner $\{h'_n\}$ which, for
    all valid distributions $\cP,$ learns the corresponding
    target language $K$ in Angluin's setting \citep{angluin1988identifying} with probability at least
    $\nfrac{2}{3}.$ This will create the desired
    contradiction {with \cref{thm:angluin-limit-statistical-online-pos-not-convergence}}.
    Let $h'_n$ be a learner that works
    as follows.
    \begin{mdframed}
        {\textbf{Learner $h'_n$}

        \medskip 
        \noindent \textbf{Input:}\quad Access to any infinite draw $X_1,X_2,\dots$ from $\cP^\infty$ and oracle access to learner $h(\cdot)$
        \medskip

        \noindent\textbf{Description:}

        \vspace{-2mm}

        \begin{enumerate}
            \item \textbf{For each} $n\in \N$ \textbf{do:} 
            \vspace{-2mm}
            \begin{enumerate}
                \item Compute $i^* = h_n(X_1,\dots,X_n)$
                \item Compute $h'_n(X_1,\dots,X_n)$ as the smallest index $j^*$ such that $L_{j^*}$ classifies the first $n$ elements $x_1,x_2,\dots,x_n$ of the domain $\cX$ in the same way as $L_{i}$, \ie{},
                \[
                    h'_n(X_1,\ldots,X_n) \coloneqq \min\inbrace{j \in [i^*]\colon \ind\inbrace{x_i \in L_j} = \ind\inbrace{x_i \in L_{i^*}}, \forall i \in [n]} \,.
                \]
                \textit{\#~~Notice that this can be done with $O\inparen{n\cdot i^*}$ many membership queries; where we send $n$ queries to each of the first $i^*$ languages in $\cL$}.
            \end{enumerate}
        \end{enumerate}

        }
    \end{mdframed}
    \noindent Let $z \in \N$ be the smallest number for which
    $L_z = K.$ 
    We will handle the cases $z = 1$ and $z > 1$
    separately. 

    \paragraph{Case A ($z=1$):} If $L_{h(X_1,\ldots,X_n)} = K,$ we have
    that $\ind \inbrace{x \in L_{h(X_1,\ldots,X_n)}}
    = \ind \inbrace{x \in L_1}$ for all $x \in \cX.$
    Hence, since 1 is the smallest index, 
    we have that for all $n \in \N$
    \[
               L_{h_n(X_1,\ldots,X_n)} = K \implies h'_n(X_1,\ldots,X_n) = z \,. 
    \]
    In this case, we define $n_0 \coloneqq 1.$

    \paragraph{Case B ($z>1$):}
    Let $1\leq z' < z$.
    Then, since $L_z$
    is the first language for which $L_z = K$
    it must be the case that $L_{z'} \neq L_z.$
    Hence, the set $S_{z'} \coloneqq \inbrace{x \in \cX: \ind\inbrace{x \in L_{{z'}}} \neq \ind\inbrace{x \in L_{z}}}$ is non-empty.
    We let $\ell_{z'} \coloneqq \min\inbrace{i \in \N\colon x_i \in S_{z'}},$ \ie{}, let $\ell_{z'}$ be the smallest number that  $x_{\ell_{z'}}$ certifies that $L_{{z'}}$ is different from $L_z$. 
    Notice that $\ell_{z'} < \infty.$
    We also define $n_0 \coloneqq \max\inbrace{\ell_1,\ldots, \ell_{z-1}},$ 
    Notice that for all $n \geq n_0$
    we have that
    \[
        L_{h_n(X_1,\ldots,X_n)} = K \implies h'_n(X_1,\ldots,X_n) = z \,.
    \]

    \medskip 
    Let $\cE$ denote the event that
    \[
        \abs{\inbrace{n \in \N \colon  h_n(X_1,\ldots,X_n) \neq K}} < \infty \,.
    \]
    Notice that, by definition of $c_{\cP}$, 
    \[
        \Pr_{\{X_i\}_{i \in \N} \sim \cP^\infty }\left[ \cE\right] = 1-c_\cP \,.
    \]
    In other words, conditioned on the event $\cE,$
    the learner $\{h_n\}_{n \in \N}$ makes
    finitely many mistakes. Thus, {conditioned on $\evE$,} for
    every draw $D \coloneqq \{X_i\}_{i \in \N} \sim \cP^\infty$
    there is some number $n_D \in \N$ such that 
    $h_n(X_1,\ldots,X_n) = K, \forall n \geq n_D.$
    Let $n' = \max\inbrace{n_D, n_0}.$
    Then, {conditioned on $\evE$,} for every $n \geq n'$ we have that
    \[
        h'_n(X_1,\ldots,X_n) = z \,.
    \]
    {Since $c_{\cP}<\nfrac{1}{3}$ (see \cref{def:identification:cp}), $1-c_{\cP}>\nfrac{2}{3}$ which implies that} the learner $\inbrace{h'_n}_{n \in \N}$
    converges in Angluin's model with probability
     greater than $\nicefrac{2}{3},$
    for any choice of a valid data-generating
    distribution $\cP.$
    This contradicts \Cref{thm:angluin-limit-statistical-online-pos-not-convergence} and
    concludes the proof.
\end{proof}
Let us now explain why \Cref{thm:angluin-limit-statistical-online-pos-equivalent} does not immediately imply
a lower bound in our setting. First, notice that the previous result says that any learner $\{h_n\}_{n\in\N}$ must make infinitely many errors with probability at least $\nfrac{1}{3}$, under some valid data-generating distribution $\cP$. Expressing this using a $\limsup(\cdot)$ implies that\footnote{
        {Informally, $\limsup$ of a sequence of events captures the events that occur infinitely often.
        For instance, $\Pr[\limsup_{n\to \infty} \cE_n]$ represents the probability that infinitely many of the events $\cE_n$ occur.
        On the other hand, $\limsup_{n\to\infty} \Pr[\cE_n]$ roughly speaking denotes the largest value that the probabilities $\Pr[\cE_1],\Pr[\cE_2],\dots,\dots$ approach infinitely often as $n\to \infty$.}
    }
\[
    \Pr_{\{X_i\}_{i \in \N} \sim \cP^\infty } \left[ \limsup_{n \rightarrow \infty} \{L_{h_n(X_1,\ldots,X_n)} \neq K \}\right] \geq \frac{1}{3} \,.
\]
{{Since $R\colon \N\to \R_{\geq 0}$ is a rate function, it satisfies} $\lim_{n \rightarrow \infty} R(n) = 0$.
Thus, in order to show that $\cL$ is not learnable at any
rate, it is enough to show that for any learner
$\{h_n\}_{n \in \N},$ there exists a valid
distribution $\cP$ such that $\Pr_{X_1,\ldots,X_n \sim \cP^n}\left[L_{h_n(X_1,\ldots,X_n)} \neq K\right]$ does not 
converge as $n \rightarrow \infty$, or if it does converge it holds that
\[
    \lim_{n\rightarrow \infty} \Pr_{X_1,\ldots,X_n \sim \cP^n}\left[L_{h_n(X_1,\ldots,X_n)} \neq K\right] \neq 0 \,.
\]
}
For this, it suffices to show that 
\begin{equation}\label{eq:limsup-lower-bound}
  \limsup_{n \rightarrow \infty} \Pr_{X_1,\ldots,X_n \sim \cP^n}\left[L_{h_n(X_1,\ldots,X_n)} \neq K\right]  > 0\,,
\end{equation}
for some valid distribution $\cP.$
{It follows from the reverse of Fatou's lemma} that for every sequence of events $\{\cE_n\}_{n \in \N}$ %
\[
    \Pr\left[ \limsup_{n \rightarrow \infty} \cE_n \right] \geq \limsup_{n\rightarrow\infty} \Pr[\cE_n] \,,
\]
which is not sufficient to deduce the result we need. In fact, it is not hard to construct a family of events such
that
$\inbrace{\cE_n}_{n \in \N}$ such that $\Pr\left[ \limsup_{n \rightarrow \infty} \cE_n \right] = 1,$
but $\limsup_{n\rightarrow\infty} \Pr[\cE_n] = 0:$
consider an infinite stream of \emph{independent} coin 
flips, where the probability of success of the $n$-th
try is $\nicefrac{1}{n}.$ The second Borel--Cantelli
lemma (see \Cref{lem:second-borel-cantelli}) implies the  result.
Hence, we need to study the particular structure of our problem {to show that
$\limsup_{n\rightarrow\infty} \Pr[\cE_n] > 0.$
}.

{In fact, {to deduce that $\limsup_{n\rightarrow\infty} \Pr[\cE_n] > 0,$}
we show a stronger result: the $\limsup$ of the probability of error
of the learner is not merely bounded away from zero,
but, it is at least $\nicefrac{1}{2}$ (\Cref{lem:no-rate-identification-positive}).}
To that end, we follow a 
strategy which consists of the following two main steps:
\begin{itemize}
    \item First, we assume that there exists a learner $\{h_n\}_{n \in \N}$
    such that for every valid distribution $\cP$ there 
    is some $c > 0$ such that
    \[
    \limsup_{n \rightarrow \infty}  \Pr_{X_1,\ldots,X_n \sim \cP^n}\left[\inbrace{L_{h_n(X_1,\ldots,X_n)} \neq K}\right] \leq \frac{1}{2}-c \,.
    \]
    Then, we show that using the learner $\{h_n\}_{n \in \N}$ we can construct a learner $\{h'_n\}_{n \in \N}$ such that for
    all valid distributions $\E_{X_1,\ldots,X_n \sim \cP^n}[\ind\{L_{h'_n(X_1,\ldots,X_n)} \neq K\}] \leq C \cdot e^{-c \cdot\sfrac{n}{\log n}},$
    where $c,C$ are distribution-dependent constants. This can be viewed
    as a boosting argument for identification.
    To make this argument work, we also need
    to use our post-processing result (\Cref{lem:post-processing-same-index}) to map different outputs
    that correspond to $K$ to the same
    index.
    
    \item Subsequently, using the Borel--Cantelli lemma (see \Cref{lem:first-borel-cantelli})
    we show that for $\{h'_n\}_{n \in \N}$ it holds that for any valid
    distribution $\cP$
    \[
        \Pr_{\{X_i\}_{i \in \N} \sim \cP^\infty } \left[ \limsup_{n \rightarrow \infty} \left\{L_{h'_n(X_1,\ldots,X_n)} \neq K \right\}\right] = 0 \,,
    \]
    which, combined with \Cref{thm:angluin-limit-statistical-online-pos-equivalent}, leads to a contradiction.
\end{itemize}
The formal statement and the proof of the result follow.

\begin{lemma}\label{lem:no-rate-identification-positive}
    For every countable collection of languages $\cL$ that does not satisfy
    Angluin's condition, and every learning algorithm $\cA = \{h_n\}_{n \in \N}$ there exists a valid distribution $\cP$ supported on $K \in \cL$
    such that 
    \[
        \limsup_{n \rightarrow \infty} \Pr_{X_1,\ldots,X_n \sim \cP^n }  \left[L_{h_n(X_1,\ldots,X_n)} \neq K\right] \geq \frac{1}{2} \,.
    \]
\end{lemma}

\begin{proof} 
    Assume towards contradiction that there exists a countable collection of languages
    $\cL$ that does not satisfy Angluin's condition and a learning
    algorithm $\cA = \{h_n\}_{n \in \N}$ such that for all target languages $K \in \cL$ and for all valid
    distributions $\cP$ supported over $K$ there exists some $c < \nicefrac{1}{2}$ such that
      \[
        \limsup_{n \rightarrow \infty} \Pr_{X_1,\ldots,X_n \sim \cP^n } \left[L_{h_n(X_1,\ldots,X_n)} \neq K\right] = c \,.
    \]
    Let also $\tilde c \coloneqq 1/2 - c > 0$ By definition of the limit superior, it holds
    that 
    \[
        \abs{\left\{n \in \N\colon   \Pr_{X_1,\ldots,X_n \sim \cP^n }[L_{h_n(X_1,\ldots,X_n)} \neq K] > c + \frac{\tilde c}{2}\right\}} < \infty \,.
    \]
    For the rest of the proof, let us
    fix some valid distribution $\cP.$
    Let $n_0$ be the largest number such that $\Pr_{X_1,\ldots,X_n \sim \cP^n } \left[L_{h_n(X_1,\ldots,X_n)} \neq K\right] > c + \inparen{\nfrac{\tilde c}{2}}$. 
    Notice that $n_0$ depends on $\cP.$
    The previous
    argument shows that $n_0 < \infty.$ For all $n > n_0$ we have that
    \[
        \Pr_{X_1,\ldots,X_n \sim \cP^n } \left[L_{h_n(X_1,\ldots,X_n)} \neq K\right] \leq c + \frac{\tilde c}{2} = \frac{1}{2} - \frac{\tilde c}{2} \,.
    \]
    Consider the algorithm $\cA' = \{h'_n\}_{n \in \N}$ that works 
    as follows: for every $n$, it splits the dataset into
    $\hat t_n \coloneqq \nfrac{n}{\log n}$ consecutive and non-overlapping batches, each
    of size $\log n.$ Then, it runs algorithm $h_{\log n}$ on
    each of the batches. Let $i_j$ denote the output
    of the $j$-th batch, $\cI^n = (i_1,\ldots,i_{\hat t_n})$
    denote the multiset of all these indices and $z = \min\{\ell \in \N\colon  L_\ell = K\}.$  Then, using \Cref{lem:post-processing-same-index} we know that there exists
    some $\hat n_0 \coloneqq \hat n_0(K, \cL, \cX)$ such that
    for all $n \geq \hat n_0$
    \begin{align*}
        L_{i_j} = K \implies \hat i_j^n = z, \quad \forall j \in [\hat t_n] \,,
    \end{align*}
    where $\hat i_j^n$ is defined in \Cref{lem:post-processing-same-index}. Thus, letting $\hat \cI^n = \left(\hat i_1^n, \ldots, \hat i_{\hat t_n}^n\right)$
    we have that for all $n \geq \max\{n_0, \hat n_0\},$
    all the indices of $\cI^n$ that correspond to some index of
    $K$ are mapped to $z$ in the collection $\hat \cI^n.$

    Finally, the algorithm outputs the majority vote
    over the indices in $\hat \cI^n.$
    Using a standard Chernoff bound, we see that 
    \[
        \Pr\left[\sum_{j \in [\hat t_n]} \frac{\ind\left\{\hat i_j^n \neq z \right\}}{\sfrac{n}{\log n}} \geq \frac{1}{2} - \frac{\tilde c}{4}\right] \leq e^{-\sfrac{\tilde c^2 2n}{\log{(n)}}}, \quad\forall n \geq \max\{n_0, \hat n_0\} \,.
    \]
    {This implies that
    \[
        \Pr_{X_1,\ldots,X_n}\insquare{{ h'_n(X_1,\ldots,X_n)} \neq z} \leq  e^{-\sfrac{\tilde c^2 2n}{\log{(n)}}}, \quad\forall n \geq \max\{n_0, \hat n_0\} \,.
    \]
    }
    Thus, we have that
    \begin{align*}
        \sum_{n \in \N} \Pr_{X_1,\ldots,X_n \sim \cP^n}[L_{h'_n(X_1,\ldots,X_n)} \neq K] &= \sum_{n \leq \max\{n_0, \hat n_0\}} \Pr_{X_1,\ldots,X_n \sim \cP^n}[L_{h'_n(X_1,\ldots,X_n)} \neq K] 
        \\&+ \sum_{n > \max\{n_0, \hat n_0\}} \Pr_{X_1,\ldots,X_n \sim \cP^n}[L_{h'_n(X_1,\ldots,X_n)} \neq K] \\
        &\leq \max\{n_0, \hat n_0\} +  \sum_{n > \max\{n_0, \hat n_0\}} \Pr_{X_1,\ldots,X_n \sim \cP^n}[L_{h'_n(X_1,\ldots,X_n)} \neq K] \\
        &\leq \max\{n_0, \hat n_0\} +  \sum_{n > \max\{n_0, \hat n_0\}} \Pr_{X_1,\ldots,X_n \sim \cP^n}[{h'_n(X_1,\ldots,X_n)} \neq z] \\
        &\leq \max\{n_0, \hat n_0\}+ \sum_{n > \max\{n_0, \wh{n_0}\}} e^{-\sfrac{\tilde c^2 2n}{\log{(n)}}} \\
        &< \infty \,.
    \end{align*}
    Using the Borel--Cantelli lemma (see \Cref{lem:first-borel-cantelli}), we get that
    \[
        \Pr_{\{X_i\}_{i \in \N} \sim \cP^\infty } \left[ \limsup_{n \rightarrow \infty} \inbrace{L_{h'_n(X_1,\ldots,X_n)} \neq K} \right] = 0 \,.
    \]
    Since this holds for all valid distributions $\cP$,
    it contradicts \Cref{thm:angluin-limit-statistical-online-pos-equivalent}, which states that, for some valid $\cP'$ (which depends on $\{h'_n\}_{n \in \N}$),
    it holds that
    \[
       \Pr_{\{X_i\}_{i \in \N} \sim {\cP'}^\infty } \left[ \limsup_{n \rightarrow \infty} \inbrace{L_{h'_n(X_1,\ldots,X_n)} \neq K} \right] \geq \frac{1}{3} \,. 
    \]
    This concludes the proof.
    
\end{proof}

\noindent We now have all the components to prove \Cref{thm:dichotomy-identification-positive} {by following the outline in \cref{fig:outline:thm:dichotomy-identification-positive}}.

\begin{proof}[Proof of \Cref{thm:dichotomy-identification-positive}]
    Let $\cL$ be any non-trivial collection of 
    languages. Then, \Cref{lem:exp-rates-lower-bound-ident-pos} implies that no learner can learn $\cL$
    at a rate faster than $e^{-n}.$

    Let us first consider the case that $\cL$
    satisfies Angluin's condition. This implies
    that $\cL$ is identifiable in the limit (see \Cref{thm:angluin-id-limit}). Let $g\colon\N \rightarrow \R$, be some sublinear function, \ie{}, $g(n) = o(n).$
    Then, \Cref{lem:almost-exp-rates-ident-pos-examples}
    shows that there exists a learner that achieves rates
    $e^{-g(n)}$ for $\cL.$

    Lastly, we consider the case where $\cL$
    does not satisfy Angluin's condition. Then, \Cref{lem:no-rate-identification-positive} shows that for every learner $\{h_n\}_{n \in \N}$,
    there exists a valid distribution $\cP$ for which
    \[
        \limsup_{n \rightarrow \infty} \Pr_{X_1,\ldots,X_n \sim \cP^n }  \left[L_{h_n(X_1,\ldots,X_n)} \neq K\right] \geq \frac{1}{2} \,.
    \]
    Hence, $\cL$ is not learnable at any rate. This concludes 
    the proof.
\end{proof}

    \subsection{Proof of \texorpdfstring{\Cref{thm:statistical-generation}}{Theorem 3.2} (Rates for Generation)} \label{sec:proofof:mainthm:statisticalRates:gen}
\label{sec:statistical-generation}

    In this section, we prove \Cref{thm:statistical-generation} following the outline in \cref{fig:outline:thm:statistical-generation}.

    First, in \cref{sec:optimalRates:nonTrivialForGeneration} we formally define the family of collections that are non-trivial for generation (\cref{def:non-trivial-collection-generation}) and show that (1) for any trivial collection, it is possible to generate after seeing a \textit{constant} number of examples (\cref{lem:trivial-collection-generation-in-the-limit}) and (2) an exponential rate is the best-possible for any non-trivial collection (\cref{lem:exp-rates-lower-bound-gen}) .
    Next, in \cref{sec:suffCondition:Generation}, we present a sufficient condition under which a generation algorithm that works ``in-the-limit'' achieves exponential rate (without any modifications).
    Finally, in \cref{sec:statistical-generation-upper-bound-subset,sec:statistical-generation-upper-bound-membership}, we present algorithms that achieve exponential rates given access to a subset oracle and membership oracle for the collection $\cL$ respectively.

\begin{figure}[bth]
    \centering
    \begin{tikzpicture}[node distance=2.5cm, auto]
    \node[myNodeNarrow] (thm) at (0,0) {\cref{thm:statistical-generation}};
    \node[myNodeWide] (subres1) at (-4cm, -2cm) {\small For non-trivial collections $e^{-n}$ is the best possible rate};
    \node[myNodeWide] (subres2) at (4cm, -2cm) {\small Consistent generation at exponential rate};

    \node[myNodeNarrow] (lem59) at (2cm, -4cm) {\small \cref{lem:trivial-collection-generation-in-the-limit}};
    \node[myNodeNarrow] (lem510) at (-4cm, -4cm) {\small \cref{lem:exp-rates-lower-bound-gen}};
    \node[myNodeNarrow] (lem511) at (6cm, -4cm) {\small \cref{lem:generation-from-online-to-statistical-property}};
    \node[myNodeNarrow] (lem512) at (4cm, -6cm) {\small \cref{lem:km24-satisfies-condition-exp-rates-subsetOracle}};
    \node[myNodeNarrow] (lem513) at (8cm, -6cm) {\small \cref{lem:km24-satisfies-condition-exp-rates}};

    \draw[-stealth, line width=0.5mm] (thm) -- (subres1.north);
    \draw[-stealth, line width=0.5mm] (thm) -- (subres2.north);

    \draw[-stealth, line width=0.5mm] (subres1) -- (lem510);
    \draw[-stealth, line width=0.5mm] (subres2) -- (lem59);
    \draw[-stealth, line width=0.5mm] (subres2) -- (lem511);
    \draw[-stealth, line width=0.5mm] (lem511) -- (lem512);
    \draw[-stealth, line width=0.5mm] (lem511) -- (lem513);
    
    \end{tikzpicture}
    \caption{Outline of Proof of \cref{thm:statistical-generation}}
    \label{fig:outline:thm:statistical-generation}
\end{figure}

    \subsubsection{Optimal Rate for Non-Trivial Collections for Generation}\label{sec:optimalRates:nonTrivialForGeneration}

    For the case of generation, we need a different
notion of non-trivial languages from the one 
we did for identification (see \cref{def:non-trivial-collections}). Indeed, assume that
$\cL = \inbrace{L_1, L_2}$, and $L_1 \neq L_2, L_1 \cap L_2 = \infty.$ This collection satisfies \Cref{def:non-trivial-collections}, thus no algorithm can identify at a rate
faster than exponential. However, 
it is not hard to see there is a consistent
generation algorithm that does not 
need \emph{any} samples: just generate a string 
from $L_1 \cap L_2.$ Instead, the following condition turns out to characterize collections that are non-trivial for generation.

    \begin{figure}[t!]
            \centering
            \hspace{-4mm}\subfigure[Trivial for Generation]{     
                {\hspace{-4mm}\includegraphics[scale=0.6,trim={4cm 0.5cm 5cm 2.5cm}]{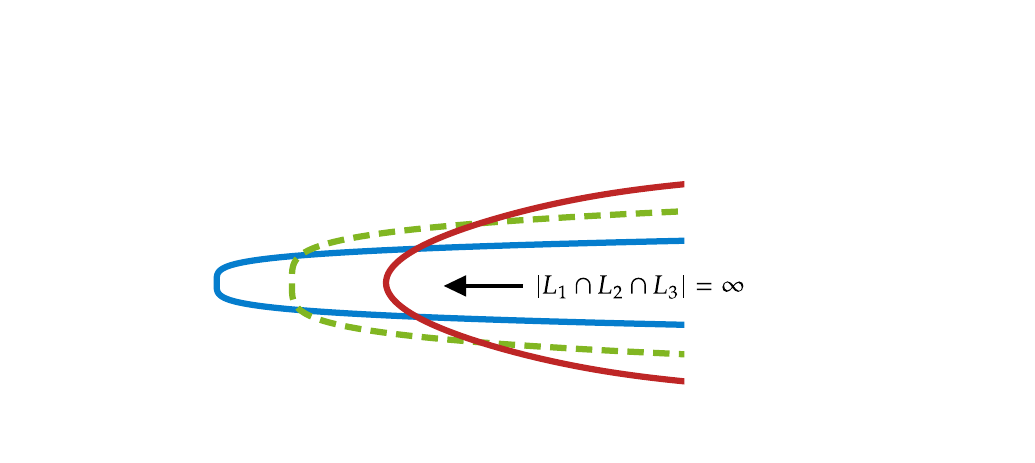}}
            }
            \subfigure[Trivial for Generation]{     
                {\includegraphics[scale=0.55,trim={6cm 0.5cm 1.75cm 0.5cm}]{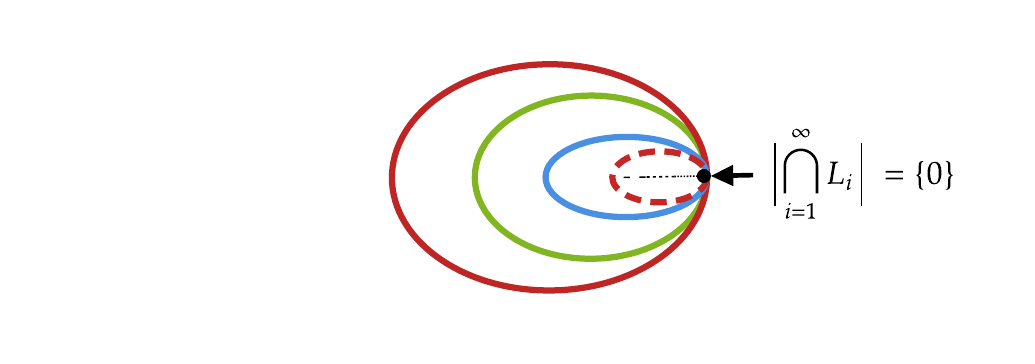}}
            }
            \subfigure[Non-Trivial For Generation]{  
                {\includegraphics[scale=0.6,trim={5.5cm 0.5cm 3cm 0cm}]{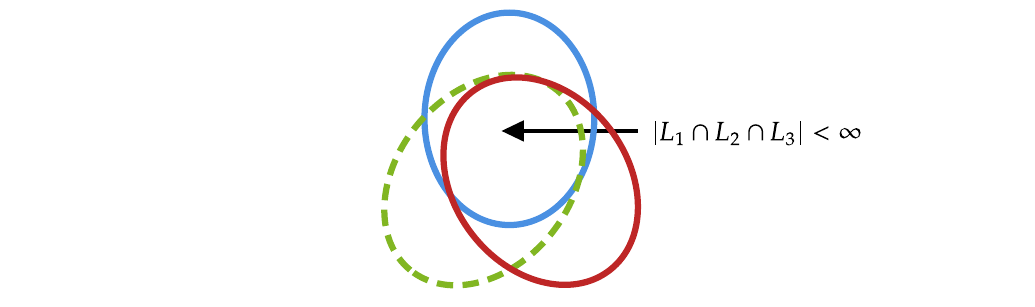}}\hspace{-4mm}
            }\hspace{-4mm}
        \caption{Illustrations of language collections that are (a,b) trivial for generation and (c) non-trivial for generation.
        In cases (a) and (c), the collection $\cL$ has three languages -- $L_1,L_2,$ and $L_3$ -- denoted by different colors. 
        Case (b) illustrates \cref{ex:prefixes-of-rationals}; here, the collection $\cL$ has infinitely many languages which follow a nested structure $L_1\supsetneq L_2\supsetneq \dots \supsetneq \inbrace{0}$.
        }
        \label{fig:generation-trivial-illustration}
    \end{figure}

\begin{definition}[Non-trivial for Generation]\label{def:non-trivial-collection-generation}
    A language collection $\cL$ is non-trivial
    for generation if there exists
    some $x \in \cX$ and a finite set of languages
    $\cL' \subseteq \cL$ such that:
    \begin{itemize}
        \item each $L\in \cL'$ contains $x$; and
        \item the intersection of all languages in $\cL'$ is finite, \ie{}, $\abs{\cap_{L \in \cL'} L} < \infty.$
    \end{itemize}
\end{definition}
To verify that the definition of non-trivial collections
is meaningful, we show in the following result
that for all trivial collections, there exists an algorithm
that generates correctly {with probability 1, for all valid distributions,} when $n$ is sufficiently large {even if the training dataset
contains only one distinct element}.
Interestingly, our result uses an algorithm that generates
in the limit in a setting that is slightly 
different from the one considered by \citecustom{kleinberg2024language}:
 we
fix some target language $K \in \cL$, we give
a single input $x \in K$ to the algorithm and then we run 
it for infinitely many steps, without giving any further inputs.
We show that there exists some
$n_{K,\cL} \in \N$ which depends only on $K$ and the enumeration of $\cL$, but, crucially, not on $x,$ such
that the algorithm generates correctly for every $n \geq n_{K,\cL}.$
The algorithm is described below.
    \begin{mdframed} 
            {\rm \textbf{Generating in the limit from a trivial collection $\cL = \inbrace{L_1,L_2,\ldots}$} }\\[2mm]
            {\rm \textbf{Input:} A set of $n$ elements $\inbrace{X_1,\ldots,X_n}$ from $\cX$, potentially containing repetitions}\\[2mm]
            \textbf{Description:}
                \begin{enumerate}[itemsep=0pt]
                        \item Select any arbitrary element $x$ from $\inbrace{X_1,\ldots,X_n}$
                        \item Initialize the index $j = 1$
                        \item \textbf{while } $x\notin L_j$ \textbf{do:} increment $j$ by 1
                        \item[] \textit{\#~~When $X_1,\ldots, X_n$ are drawn from a valid distribution this step terminates with probability 1}
                       \item Compute $V^n(x) \coloneqq \inbrace{L_i \in \cL\colon  x \in L_i,~~ 1\leq i\leq n} \cup \inbrace{L_j}$ 
    \item Let $k \coloneqq 1$
    \item \textbf{while} $x_k \notin \bigcap_{L \in V^n(x)}L \setminus \inbrace{X_1,\ldots,X_n}$ \textbf{ do:} increment $k$ by 1
    \item[] \textit{\#~~When $\cL$ is trivial for generation (see \cref{def:non-trivial-collection-generation}) this step terminates with probability 1}
   
    \item \textbf{return} $x_k$
                \end{enumerate}
        \end{mdframed}
Notice that the previous algorithm is indeed computable given access to a membership oracle {for each language in} $\cL.$

\begin{lemma}[Algorithm For Generation from Trivial Collections]\label{lem:trivial-collection-generation-in-the-limit}
       For every collection of languages $\cL$ that is trivial for
    generation, there exists a generation algorithm $(\generator_n)_{n \in \N}$ such
    that for every valid distribution $\cP$ with respect to 
    $\cL$ {it terminates with probability 1 for all $n \in \N$,
    and there exists some constant $C$ that depends on  $\cP, \cL,$} such that
    for all $n\geq C$, it holds that
    \[
        \E_{X_1,\ldots,X_n \sim \cP^n}\insquare{\ind\inbrace{\generator_n(X_1,\ldots,X_n) \not\in \supp(\cP) \setminus \inbrace{X_1,\ldots,X_n}}} = 0\,.
    \]
\end{lemma}
\begin{proof}
    Let $\cL$ be a trivial collection for generation (see \cref{def:non-trivial-collection-generation}). Fix some valid distribution $\cP$ supported over target language $K.$ Let $z \in \N$ be the smallest number such that $L_z = K.$ 
    Then, the triviality condition states that for every $x \in \cX$
and every finite set of languages $\cL' \subseteq \cL$
it holds that either $x \notin L,$ for some $L \in \cL',$ or ${\abs{\cap_{L \in \cL'} L} = \infty}.$
  We have that, with probability 1, every \iid{} draw of $n$ samples from $\cP$ satisfies $X_1,\ldots,X_n \in L_z.$ Choose an arbitrary $x \in \inbrace{X_1,\ldots,X_n}.$ Notice again that, with probability 1, $x \in K$. 
  Consider the execution 
    of the algorithm described above by fixing this $x$.
Let us now verify that this algorithm generates some $x' \in K \setminus \inbrace{X_1,\ldots,X_n}$ for all $n \geq z,$
{and terminates with probability 1 for all $n \in \N.$
}
First, notice that by definition of $z$, when $n \geq z$ it holds that $L_z \in V^n(x).$
Notice that, since for all $L \in V^n(x)$ it holds that $x \in L$
and $\abs{V^n(x)} \leq n < \infty$, the triviality definition implies that for all $n \geq z,$
\[
    \abs{\bigcap_{L \in V^n(x)} L} = \infty \,.
\]
Hence, for all $n \geq z$ we have that
\begin{itemize}
    \item $ K \in V^n(x),$ thus $\bigcap_{L \in V^n(x)}L \subseteq K.$
    \item  $ \abs{\bigcap_{L \in V^n(x)}L} = \infty.$
\end{itemize}
Thus, it holds that
\[
    \abs{\bigcap_{L \in V^n(x)}L \setminus \{X_1,\ldots,X_n\}} = \infty \,.
\]
Hence, the algorithm can generate unseen strings from the
target language $K$ for all $n \geq z,$ with probability 1. 
Notice that $z$ indeed only depends on $K, \cL.$

{We now prove the termination property of our algorithm. To that end, 
it suffices to show that both \texttt{while} loops
terminate with probability 1.
As
we argued before, $x \in L_z$ with probability 1, hence, the first
\texttt{while} loop terminates after at most $z$ steps. We now
consider the termination of the second \texttt{while} loop. 
As we argued above, $  \abs{\bigcap_{L \in V^n(x)}L \setminus \{X_1,\ldots,X_n\}} = \infty$ with probability 1,
hence the loop will terminate after a finite number of
steps.
}
This
concludes the proof. 
\end{proof}
We remark that the requirement that the set $\cL'$ in \Cref{def:non-trivial-collection-generation} is finite is crucial. The next example gives a collection of languages that is trivial for generation, yet it satisfies a modification of
\Cref{def:non-trivial-collection-generation} that allows $\cL'$ to be infinite.

\begin{example}[A Trivial Collection for Generation That ``Almost'' Satisfies \cref{def:non-trivial-collection-generation}]\label{ex:prefixes-of-rationals}
    Define the domain $\cX$ and the language collection $\cL$ as follows 
    \[
        \cX = [0,1] \cap \Q
        \qquadand
        \cL = \inbrace{\insquare{0,\frac{1}{n}} \cap \Q\colon n \in \N}\,,
    \]
    where $\Q$ is the set of rational numbers. 
    Notice that both $\cX$ and $\cL$ are countable, and each $L \in \cL$ is also countable.
    Consider the element $x = 0$ and the set $\cL' = \cL.$ First, notice
    that $x \in L, \forall L \in \cL'$ (by definition of every $L \in \cL$).
    Moreover, it is not hard to see that
    $\cap_{L \in \cL'} L = \{0\},$ hence $\abs{\cap_{L \in \cL'} L} = 1 < \infty.$
    It is also not hard to see that every finite sub-collection $\cL'' \subseteq \cL,$ satisfies $\abs{\cap_{L \in \cL''}L} = \infty.$
    Hence, the conditions of \Cref{def:non-trivial-collection-generation}
    can only be satisfied by infinite sub-collections.
    We can show that there is an algorithm that generates from $\cL,$ 
    without seeing
    any example. Indeed, consider the algorithm that in every round $n \in \N$
    outputs the element $\nicefrac{1}{n}.$ Let $K$ be any target language.
    By definition, there is some $n_K \in \N$ such that $\nicefrac{1}{n} \in K,$ for all $n \geq n_K.$ Hence, this algorithm can generate from $K.$
\end{example}
We now move on the proof the main result in this section. 
Similar to the identification setting before, we show the main result
in two parts. First, we show that 
for any non-trivial collection of languages, no algorithm can generate at a rate faster than exponential. 
The approach {shares some high-level ideas} with the identification setting, {but the
more complicated condition that characterizes
non-trivial generation makes the technical details
more nuanced. In particular, leveraging the 
non-triviality condition, we can deduce
that there exists a finite set of elements 
$\inbrace{x_{\ell_1},\ldots,x_{\ell_B}}$ and a finite
collection of languages $\cL'$ so that
$\cap_{L \in \cL'} L = \inbrace{x_{\ell_1},\ldots,x_{\ell_B}}.$
Then, whenever the training set consists exactly of
the elements $\inbrace{x_{\ell_1},\ldots,x_{\ell_B}}$ (containing also duplicates of them),
no matter what element $x \in \cX \setminus \inbrace{x_{\ell_1},\ldots,x_{\ell_B}}$ the algorithm generates, there
exists some $L \in \cL'$ so that $x \notin L'$.
This allows us to find some ``hard'' distribution,
which depends on the generating algorithm, and
for which this event happens with exponentially small probability.
}
The formal statement of the result
and the technical details of the proof follow.

\begin{lemma}[Exponential Rate Is Optimal for Generating From Any Non-trivial Collection]\label{lem:exp-rates-lower-bound-gen}
    Let $\cL$ be a non-trivial collection of languages for generation. Then, 
    for any generating algorithm $(\generator_n)_{n \in \N}$ there
    exists a valid distribution $\cP$ such that 
    $\E[\er(\generator_n)] \geq C \cdot e^{-c \cdot n},$ for 
    infinitely many $n \in \N.$
\end{lemma}

\begin{proof}
    Since $\cL$ is non-trivial
    for generation, there exists
    some $x \in \cX$ and a finite $\cL' \subseteq \cL$ such that $x \in \cap_{L \in \cL'} L$ and
    $\abs{\cap_{L \in \cL'} L} = B < \infty.$
    Let $\inbrace{x_{\ell_1},\ldots,x_{\ell_B}} \coloneqq \cap_{L \in \cL'} L$ be the
    distinct elements that appear
    in the intersection of the sub-collection $\cL'.$ Define a collection of distributions
    $\inbrace{\cP_L}_{L \in \cL'}$ that has two properties:
    \begin{itemize}
        \item For every $L \in \cL'$ it holds that $\cP_L$ is valid for $L$.
        \item All the distributions $\inbrace{\cP_L}_{L \in \cL'}$ put exactly the same mass 
        on every element of the set $\inbrace{x_{\ell_1},\ldots,x_{\ell_B}}.$
    \end{itemize}
    Notice that, by definition of $\inbrace{x_{\ell_1},\ldots,x_{\ell_B}},$
    there are collections of distributions that satisfy these two constraints.

    For any $n \geq B,$ let $\cE_n$ be the event that
    the training set is $\inparen{x_{\ell_1},\ldots,x_{\ell_B},x_{\ell_1},\ldots,x_{\ell_1}}.$ Notice that under any $\cP \in \inbrace{\cP_L}_{L \in \cL'}$,
    \[
        \Pr_{X_1,\ldots,X_n \sim \cP^n}[\cE_n] \geq C \cdot e^{-c \cdot n} \,,
    \]
    for the same constants $C,c$  for
    all $\cP \in \inbrace{\cP_L}_{L \in \cL'}.$
    Recall that for any valid distribution 
    supported on a target language $K$
    \[
        \er\inparen{\generator_n(X_1,\ldots,X_n)} = \ind\inbrace{\generator_n(X_1,\ldots,X_n) \notin K \setminus \inbrace{X_1,\ldots,X_n}} \,. 
    \]
    Since $\abs{L} = \infty, \forall L \in \cL$
    and  $\abs{\cap_{L \in \cL'} L} < \infty,$
    it follows that $\abs{\cL'} \geq 2.$ Let $k \coloneqq \abs{\cL'}$. By definition of $\cL$, $k < \infty$, and
    by the previous argument $k \geq 2$. 
    Moreover,
    notice that, for all $x \in \cX \setminus \inbrace{x_{\ell_1},\ldots,x_{\ell_B}}$ there
    exists $L\in\cL'$ such that $x \notin L.$ 
    Indeed, if $x \in L, \forall L \in \cL'$
    then $x \in \cap_{L \in \cL'} L \setminus \inbrace{x_{\ell_1},\ldots,x_{\ell_B}},$ but
    $\cap_{L \in \cL'} L \setminus \inbrace{x_{\ell_1},\ldots,x_{\ell_B}} = \emptyset.$
    For every distribution $p$ over $\cX$ it holds that 
    \[
        \Pr_{X \sim p}\insquare{X \in \inbrace{x_{\ell_1},\ldots,x_{\ell_B}} \text{ or } \exists L \in \cL' \text{ such that } X \notin L} = 1 \,.
    \]
    For every $n \in \N, n\geq B,$ conditioned on the event $\cE_n,$
    since the algorithm is a randomized mapping
    from the training set to $\cX,$
    we have that $\generator_n\inparen{x_{\ell_1},\ldots,x_{\ell_B},x_{\ell_1},\ldots,x_{\ell_1}} = p_n.$ We consider two cases:
    \begin{itemize}
        \item For infinitely many $n \in \N, n\geq B$ it holds that
        \[
            \Pr_{X \sim p_n}\insquare{X \in \inbrace{x_{\ell_1},\ldots,x_{\ell_B}}} \geq \frac{1}{2} \,.
        \]
        Let $\hat N$ be the infinite set for which the previous holds. Then, for all $\cP \in \inbrace{\cP_L}_{L \in \cL'}$ and for all $n \in \hat N$ it holds that
        \begin{align*}
             \E_{X_1,\ldots,X_n \sim \cP^n}[\er(\generator_n(X_1,\ldots,X_n))] &= \Pr_{X_1,\ldots,X_n \sim \cP^n}[\generator_n(X_1,\ldots,X_n) \notin K \setminus \inbrace{X_1,\ldots,X_n}] \\
               &\geq  \Pr_{X_1,\ldots,X_n \sim \cP^n}[\generator_n(X_1,\ldots,X_n) \notin K \setminus \inbrace{X_1,\ldots,X_n} \mid \cE_n]  \cdot \Pr_{X_1,\ldots,X_n \sim \cP^n}[\cE_n]\\
               &\geq C\cdot e^{-c\cdot n} \cdot \Pr[\generator_n(x_{\ell_1},\ldots,x_{\ell_B},x_{\ell_1},\ldots,x_{\ell_1}) \notin K \setminus \inbrace{x_{\ell_1},\ldots,x_{\ell_B}}] \tag{by the definition of $\cE_n$} \\ 
               &\geq  C\cdot e^{-c\cdot n} \cdot \Pr_{X \sim p_n}\insquare{X \in \inbrace{x_{\ell_1},\ldots,x_{\ell_B}}} \tag{by the definition of $p_n$}\\
               &\geq  C\cdot e^{-c\cdot n} \cdot\frac{1}{2} \,. \tag{by the assumption on $\hat N$}
        \end{align*}
        Thus, taking as the target distribution any 
        $\cP \in \{\cP_{L}\}_{L \in \cL'}$
        we see that the algorithm indeed has an exponential rates lower bound.

        \item For infinitely many $n \in \N, n\geq B,$
        it holds that
         \[
        \Pr_{X \sim p_n}\insquare{ \exists L \in \cL' \text{ such that } X \notin L} \geq \frac{1}{2} \,.
    \] 
       Then, due to the pigeonhole principle, there is some $L \in \cL'$ such that for infinitely many $n \in \N$ it holds that
        \[
            \Pr_{X \sim p_n}\insquare{X \notin L} \geq \frac{1}{2k} \,.
        \]
        Let $\hat N$ be the infinite set for which the previous holds. Then, for the data-generating distribution $\cP_L$ we have 
        that
        \begin{align*}
             \E_{X_1,\ldots,X_n \sim \cP_L^n}[\er(\generator_n(X_1,\ldots,X_n))] &= \Pr_{X_1,\ldots,X_n \sim \cP_L^n}[\generator_n(X_1,\ldots,X_n) \notin L \setminus \inbrace{X_1,\ldots,X_n}] \\
               &\geq  \Pr_{X_1,\ldots,X_n \sim \cP_L^n}[\generator_n(X_1,\ldots,X_n) \notin L \setminus \inbrace{X_1,\ldots,X_n} \mid \cE_n]  \cdot \Pr_{X_1,\ldots,X_n \sim \cP^n}[\cE_n]\\
               &\geq C\cdot e^{-c\cdot n} \cdot \Pr[\generator_n(x_{\ell_1},\ldots,x_{\ell_B},x_{\ell_1},\ldots,x_{\ell_1}) \notin K \setminus \inbrace{x_{\ell_1},\ldots,x_{\ell_B}}] \tag{by the definition of $\cE_n$} \\ 
               &\geq  C\cdot e^{-c\cdot n} \cdot \Pr_{X \sim p_n}\insquare{X \notin L} \tag{by the definition of $p_n$}\\
               &\geq  C\cdot e^{-c\cdot n} \cdot\frac{1}{2k} \,. \tag{by the assumption on $\hat N$}
        \end{align*}
        Then, we can pick the target distribution
        to be $\cP_L,$ and the exponential lower bound follows.
    \end{itemize}
    The proof is concluded by noticing that, from the pigeonhole principle, at least one of the previous two cases
    holds for any sequence of $\{p_n\}_{n \in \N}.$
\end{proof}

\subsubsection{A Sufficient Condition To Achieve Exponential Rate}\label{sec:suffCondition:Generation}
Let us now shift our attention to the upper
bound. Following the approach of \citecustom{kleinberg2024language}, we consider
two settings: first, we assume access 
to a subset oracle which can answer questions $L_i \subseteq L_j,$ for all $i, j \in \N.$
Then, we consider the setting
where we only have access to a membership oracle
for each language in $\cL.$

Before describing our approach let us explain
why a direct adaptation of the approach of
\citet{bousquet2021theory} does not seem
to work in this setting. Recall that \citet{bousquet2021theory}
transform in a black-box manner a learner which
is eventually ``correct'' in the adversarial setting,
to a learner that achieves exponential rates in the 
statistical setting, by running multiple 
copies of it on independent
samples of the dataset, and then aggregating their
results through a majority vote. A crucial property
of the learner of \citet{bousquet2021theory}
is that the majority vote is taken over objects
that have binary values, namely the predicted label
of the test point.
One immediate obstacle to applying this approach
here, is that the eventually correct generators
will be outputting different valid strings in 
every iteration. Further, these valid
strings might be even coming from a different
subsets of the true language. Thus, it is not clear
at all which aggregation strategy could lead
to the desired result. One potential
approach to circumvent this obstacle is to
have all the generated strings give ``votes'' to the 
different languages of $\cL$ (potentially up to a cap $n$)
that they belong to. It is clear that after some finite 
$n_0,$ with probability at least $1 - C \cdot e^{-c\cdot n},$
the target language $K$ would be collecting votes
from the majority of the strings. Unfortunately, it is not hard
to see that for infinitely many $n_0 \in \N$ there must
be another $L' \in \cL, L' \neq K,$ that is accumulating
more votes than $K;$ if this was not the case we would
have been able to identify $K,$ for all countable $\cL,$
which contradicts our established lower bounds. Thus, 
it is not clear how to make this aggregation strategy work either.

Nevertheless, we show that, perhaps surprisingly,
a much simpler strategy works: we only need to run
one copy of the algorithm proposed by \citecustom{kleinberg2024language}
on the entire dataset to get exponential rates.
In fact, we identify a sufficient condition that allows
us to use any algorithm that works ``in-the-limit'' in the statistical setting
without making any modifications to it. We believe
that this idea might find other applications 
in the universal rates literature.

The following elementary result will be crucial 
for the analysis of both settings, \ie{}, the one
with the subset oracle and the one with just the membership 
oracle. 

\begin{lemma}\label{lem:generation-from-online-to-statistical-property}
    Let $\cL$ be a countable collection of languages. Let $\cA = \{h_n\}_{n \in \N}$ be an algorithm
    that generates from $\cL$ in the limit with positive examples
    with the following additional property:
    \begin{itemize}
        \item for every target language
    $K \in \cL$ there exists a finite set of examples 
    $\{x_{i_1},\ldots,x_{i_{\ell}}\} \subseteq K$ that depends 
    only on $K$ and the enumeration of $\cL, \cX,$ and
    \item  a finite number $n_0 \in \N$ that
    depends on $K$ and the enumeration of $\cL,\cX,$
    \end{itemize}
     such that $\cA$ always generates correctly 
    if its input has size at least $n_0$
    and it contains
    $x_{i_1},\ldots,x_{i_\ell}.$ Then,
    $\cA$ generates from $K$ with exponential rates
    in the statistical setting.
\end{lemma}

\begin{proof}
    Let $\cP$ be a valid data-generating
    distribution. Then, by definition,
    $\supp(\cP) = K,$ for some $K \in \cL.$
    Let $x_{i_1},\ldots,x_{i_\ell} \subseteq K$
    be a set of points such that after 
    $\cA$ takes as input this set
    it starts generating correctly, \ie{},
    for any $S$ such that $x_{i_1},\ldots,x_{i_\ell} \subseteq S$ and $\abs{S} \geq n_0$
    it holds that $h_{\abs{S}}(S) \in K \setminus S.$
    Since $\cP$ is a valid data-generating distribution
    it holds that $x_{i_1},\ldots,x_{i_\ell} \subseteq \supp(\cP).$ Let $p_{i_i}, \ldots, p_{i_{\ell}}$
    be the mass of points $x_{i_1},\ldots,x_{i_\ell}$
    under $\cP.$ Suppose we draw $n$ samples \iid{} from
    $\cP.$ Then, the probability that we do not
    observe all $x_{i_1},\ldots,x_{i_\ell}$ in
    the sample is bounded as
    \begin{align*}
        \Pr_{X_1,\ldots,X_n \sim \cP^n}[\exists j\in[\ell]\colon  x_{i_j} \notin \{X_1,\ldots,X_n\} ] &\leq \sum_{j\in  [\ell]} \Pr_{X_1,\ldots,X_n \sim \cP^n}[x_{i_j} \notin \{X_1,\ldots,X_n\} ] \tag{by a union bound}\\
        &= \sum_{j\in  [\ell]} \inparen{1-p_{i_j}}^n & \tag{since we have \iid{} draws}\\
        &\leq \sum_{j\in  [\ell]}  e^{-p_{i_j}\cdot n} \tag{as $1-z\leq e^{-z}$ for all $z\in \R$}\\
        &\leq \ell \cdot e^{-\min_{j\in[\ell]} p_{i_j}\cdot n} \,. 
    \end{align*}
    Thus, the algorithm generates correctly in the statistical setting
    after taking as input $n\geq n_0 \in \N$
    examples, 
    with a probability at least 
    $1- C\cdot e^{-c\cdot n},$ for some distribution
    dependent constants $C, c.$ This concludes the proof.
\end{proof}
In the next two sections, we will show that
the algorithms proposed by \citecustom{kleinberg2024language}
in the setting with access to a 
subset oracle or membership oracle
already satisfy this property. For completeness,
we present their algorithms and the related definitions.
\subsubsection{Algorithm with Access To
Subset Oracle}\label{sec:statistical-generation-upper-bound-subset}
We start with the algorithm of \citecustom{kleinberg2024language} which requires 
access to a subset oracle for $\cL,$ \ie{},
an oracle that 
for any two languages $L_i, L_j \in \cL$ answers
whether $L_i \subseteq L_j.$ To that end,
we first define the notion of critical language \citep{kleinberg2024language}.

\begin{definition}[Critical Language \citep{kleinberg2024language}]\label{def:critical-languages}
    Let $\cL = \{L_1,L_2,\ldots,\}$ 
    be a countable collection of languages.
    Let $S_n = \{x_{i_n},\ldots,x_{i_n}\} \subseteq \cX.$
    For any $j \in \N$, 
    we say that $L_j$ is critical with respect to 
    $S_n$ if $S_n \subseteq L_j$ and for all $i < j$ 
    if $S_n \subseteq L_i$ then $L_j \subseteq L_i.$
\end{definition}
The intuition behind this definition is that if two languages
$L_i, L_j$ are both critical and $i < j,$ then it is ``safer''
to generate from $L_j.$ This is exactly the way the algorithm
from \citecustom{kleinberg2024language} operates. To be more precise,
in every iteration $n \in \N$ it performs the following steps:
\begin{itemize}
    \item Let $\cL_n = \{L_1,L_2,\ldots,L_n\}$ be the first $n$
    languages of $\cL$ and $S_n = \{x_{i_1},\ldots,x_{i_n}\}$
    be the set of examples observed so far.

    \item Let $\cC_n \subseteq \cL_n$ be the set of the critical languages with respect to $S_n$ within $\cL_n$ (\Cref{def:critical-languages}).
    If $\cC_n = \emptyset$ output
    an arbitrary $x \in \cX$ and proceed
    to getting the $(n+1)$-th input.
    This step makes use of the subset oracle.\footnote{Observe that it makes sense to output something arbitrary since the first consistent (in the sense that it contains the observed training examples) language in $\cL$ is critical by definition and hence if $\cC_n = \emptyset$, we have not yet encountered a consistent language.} 

    \item Let $L_{k} \in \cC_n$ be the critical language 
    with the highest index.

    \item Output the first unseen example from $L_{k},$ \ie{}, $x_j \in \cX$
    such that $j = \min\{i \in \N\colon x_i \in L_k, x_i \notin S_n\}.$
\end{itemize}
It is implicit in the analysis of \citecustom{kleinberg2024language}
that for every target language $K \in \cL$, there exists
a set $x_{i_1}, \ldots, x_{i_\ell} \subseteq K$ and $n_0 \in \N$
that depend only on $K$ and the enumeration of $\cX,\cL,$
such that
after $n_0$ steps if the above algorithm takes as input any set $S$
that contains $\{x_{i_1},\ldots,x_{i_\ell}\}$, then
it always generates a new example correctly.
We make this explicit in the following lemma and provide
a proof for completeness.

\begin{lemma}[Adaptation of (4.3) from \citecustom{kleinberg2024language}]\label{lem:km24-satisfies-condition-exp-rates-subsetOracle}
    Let $\cL = \{L_1,L_2,\ldots\}$ be a countable collection
    of languages, let $K \in \cL$ and let $z \in \N$
    be the smallest number such that 
    $L_z = K.$ Then, there exist
    $x_{i_1},\ldots,x_{i_\ell} \in K$
    that depend only on $K$ and the enumeration of $\cL,$
    such that if the algorithm of \citecustom{kleinberg2024language}
    takes as input any set $S$
    for which $x_{i_1},\ldots,x_{i_\ell} \in S$ and $\abs{S} \geq z,$ where $z$ depends only on $K$ and
    the enumeration of $\cL,$
    then it generates correctly
    from $K.$
\end{lemma}

\begin{proof}
    Let $L_{i_1},\ldots,L_{i_\ell} \subseteq \cL$ with $i_1,\ldots,i_\ell < z$
    be the set of all
    languages that precede 
    $L_z$ in $\cL$ for which $L_z \not\subseteq L_{i_j}, j \in [\ell].$ Then, for each 
    such $L_{i_j}$ there exists some $x \in L_z$
    so that $x \notin L_{i_j}.$ Let $x_{i_j}$ be the smallest indexed element in $\cX$ for which the previous holds. Notice
    that whenever $x_{i_1},\ldots,x_{i_\ell}$ 
    is part of the input sample $S,$ then $L_z$ 
    is critical; this follows immediately from the definition of criticality and the fact that the set $S$ contradicts all the languages
    $L_i, i < z,$ such that $L_z \not\subseteq L_i.$
    Moreover, when $\abs{S} \geq z,$
    the algorithm outputs an unseen word
    from a critical
    language $L_{z'}$ with $z' \geq z.$ By definition
    of the critical language, this means
    that $L_{z'} \subseteq L_z.$ Hence, the algorithm
    generates correctly.
\end{proof}
An immediate consequence of \Cref{lem:generation-from-online-to-statistical-property} is that the algorithm
of \citecustom{kleinberg2024language} with access to 
a subset query oracle generates with exponential
universal rates.

\subsubsection{Algorithm with Access To
Membership Oracle}\label{sec:statistical-generation-upper-bound-membership}
We now move on to the more involved
version of the algorithm of \citecustom{kleinberg2024language} that only
requires membership access to every $L \in \cL.$
Recall this means that for every $x \in \cX, L \in \cL$ the algorithm can ask whether $x \in L.$

Before we describe the algorithm, we provide
the definition of 
a modified notion of a critical language \citep{kleinberg2024language}, which
is based on a notion of a \emph{projection of a language}, which we defined in \Cref{def:projection}.\footnote{\citecustom{kleinberg2024language} do not explicitly define this term; we use it to simplify
our discussion.}
Recall that, given some language $L$, we denote $L[m] = L \cap \inbrace{x_1,\ldots,x_m}$ (\Cref{def:projection}).

\begin{definition}[$m$-Critical Language \citep{kleinberg2024language}]
    Let $\cL = \{L_1,L_2,\ldots,\}$ 
    be a countable collection of languages.
    Let $S_n = \{x_{i_n},\ldots,x_{i_n}\} \subseteq \cX.$
    For any $j \in \N$, 
    we say that $L_j$ is $m$-critical with respect to 
    $S_n$ if $S_n \subseteq L_j$ and, for all $i < j$, 
    if $S_n \subseteq L_i,$ then $L_j[m] \subseteq L_i[m].$
\end{definition}
We first give an intuitive description of the key modifications of the algorithm from the previous section that are required to make it work only with access to a membership oracle. First, notice that even though the algorithm cannot ask queries of the form $L_i \subseteq L_j,$ it can ask queries of the form 
$L_i[m] \subseteq L_j[m]$, for any finite $m \in \N,$ by just asking $2m$ membership queries. Thus, the high-level idea is to replace subset queries with queries of the form $L_i[m] \subseteq L_j[m]$, for a sufficiently large $m \in \N.$ The exact details are provided below.

\begin{itemize}
    \item Let $S_n = \{x_{i_1},\ldots,x_{i_n}\}$ be the set of elements that have been presented to the learner up to step $n.$ At the beginning of step $n$, set $m_n = \max\{m_{n-1},i_n\}$.\footnote{Set $m_0 = 0$.}

    \item Let $\cV_n \subseteq \{L_1,L_2,\ldots,L_n\}$ be the set of languages whose index is at most $n$ and are consistent with the input $S_n,$ \ie{}, $S_n \subseteq L, \forall L \in \cV_n.$ If no such languages exist, output an arbitrary $x \in \cX$ and proceed to reading the $(n+1)$-th input example.
    Notice that this can be done with $n^2$ membership queries.

    \item Let $m = m_n + 1$ and $\cC_n^m$ be the set of the $m$-critical languages within $\cV_n.$  Notice that since $\cV_n \neq \emptyset,$  for all $m \in \N$ there exists at least one $m$-critical language
    (the lowest indexed language within $\cC_n$
    is $m$-critical for all $m \in \N$).

    \item Let $c_n^m \in \N$ be the largest index
    of a language in $\cC_n^m.$ If for some $i \leq m,$ it holds that $x_i \in L_{c_n^m}$ and  $x_i \notin S_n,$ output $x_i$\footnote{If there are multiple such elements, output the one with the smallest index.} and let $m_n = m.$
    Otherwise, let 
    $m_n = m_n + 1$ and repeat the previous
    bullet point.
\end{itemize}
\citecustom{kleinberg2024language} showed that the previous
algorithm terminates in finitely many steps for every $n \in \N$ (Result (5.5) from \citecustom{kleinberg2024language}). Moreover, they proved that for any enumeration of any target language $K \in \cL$, there exists some $n_0 \in \N$ so that the algorithm
generates correctly for all steps $n \geq n_0$
(Result (5.7) from \citecustom{kleinberg2024language}). 
In fact, it is implicit in their analysis that 
for all $K \in \cL$ there exist $x_{i_1},\ldots,x_{i_\ell} \in K$
that depend only on $K$ and the enumeration of $\cL, \cX,$ 
as well as a finite $n_0 \in \cL$
that depends only on $K$ and the enumeration of 
$\cL,\cX,$ such that if an input
sample $S$ satisfies that \textbf{i)} $x_{i_1},\ldots,x_{i_\ell} \in S$ and \textbf{ii)} $\abs{S} \geq n_0,$ then the algorithm generates correctly. We make this fact explicit
in the next result.

\begin{lemma}[Adaptation of (5.7) from \citecustom{kleinberg2024language}]\label{lem:km24-satisfies-condition-exp-rates}
    Let $\cL = \{L_1,L_2,\ldots\}$ be a countable collection
    of languages, let $K \in \cL$ be the target language,
    and let $z \in \N$
    be the smallest number such that 
    $L_z = K.$ Then, there exist
    $x_{i_1},\ldots,x_{i_\ell} \in K$
    that depend only on $K$ and the enumeration of $\cL,\cX,$
    such that if the algorithm of \citecustom{kleinberg2024language}
    takes as input any set $S$
    for which $x_{i_1},\ldots,x_{i_\ell} \in S$ and $\abs{S} \geq z,$ where $z$ depends only on $K$ and
    the enumeration of $\cL,$
    then it generates correctly
    from $K.$
\end{lemma}

\begin{proof}
    Let $z \in \N$ be the smallest number for which $L_z = K.$ By definition, $z$ has to be finite. Let 
    $L_{k_1},\ldots,L_{i_\ell}$ be the set of all languages
    that precede $L_z$ in $\cL$ for which $L_z \not\subseteq L_{k_j}, j \in [\ell].$ Then, for any such language $L_{k_j}$ there exists
    some $x \in K$ such that $x \notin L_{k_j}.$
    Define $x_{i_j}$ to be the smallest indexed element for which the previous holds.
    Hence, when
    the input sample $S$ contains $x_{i_1},\ldots,x_{i_\ell}$
    none of the languages $L_{i_j}, j \in [\ell],$ are 
    consistent with $S.$ Consider any iteration $n \in \N$,
    where $x_{i_1},\ldots,x_{i_\ell} \subseteq S_n,$ and
    $n \geq z.$ It follows immediately that $L_z$ is 
    $m$-critical
    for all $m \in \N,$ and hence it is contained
    in the set $\cC_n^m.$ Thus, for all $m \in \N,$
    for the largest index $c_n^m$ of a language
    in $\cC_n^m$ it holds that $c_n^m \geq z.$
    Recall that since the algorithm terminates (Result (5.5) from \citecustom{kleinberg2024language}),
    it will output some $x_m \in \cX$ such that $x_m \notin S_n, x_m \in L_{z'}, z' \geq z,$
    and $L_{z'}[m] \subseteq L_z[m].$ 
    This is because for all $m \in \N$, the
    largest index of a language in $\cC_n^m$ cannot drop
    below $z$. Thus, it follows that $x_m \in K \setminus S_n.$ Hence, the algorithm generates correctly. 
\end{proof} 

\noindent We are now ready to prove \Cref{thm:statistical-generation}.

\begin{proof}[Proof of \Cref{thm:statistical-generation}]
  Let $\cL$ be some non-trivial collection
  for generation. An immediate corollary of \Cref{lem:generation-from-online-to-statistical-property}
  and \Cref{lem:km24-satisfies-condition-exp-rates}
  is that the algorithm
of \citecustom{kleinberg2024language} with access to 
a membership query oracle generates with exponential
universal rates.  

The exponential rates lower bound for generation follows immediately from \Cref{lem:exp-rates-lower-bound-gen}.
\end{proof}

\section{Proofs from \texorpdfstring{\cref{sec:results:genBreadth}}{Section 3.2} (Generation with Breadth)}
\label{sec:proofof:generationWithBreadth}
    \subsection{Proof of \texorpdfstring{\cref{prop:mop_decidable}}{Theorem 3.4} (\texorpdfstring{$\mop(\cdot)$}{MOP} Is Decidable For Iterative Generators)}
    \label{sec:proofof:prop:mop_decidable}
        In this section, we prove \cref{prop:mop_decidable} which we restate below.
        \thmMOPDecidable*
        \noindent Recall that a token-by-token generator $\generator$ is parameterized by a randomized Turing machine $M$, where $M$ has the property that it halts on all inputs.
        $\generator$ generates as follows: in each iteration $t$, it queries $M$ to get the next token $s_t$ and iterates until $M$ outputs $\eos{}$ (\ie{}, end of string).
        The algorithm to decide $\mop{}(\generator)$ is also simple: given a string $s$ of length $n$, check token-by-token whether $\generator{}$ can output $s_i$ conditioned on a prefix $s_1,s_2,\dots,s_{i-1}$ generated so far.
        If at any point, $s_i$ is not in the support of $\generator$ (or rather $M$) then,  output \textsf{No}. Otherwise, output \textsf{Yes}.
        At each step, we can check if $s_i$ can be generated by $M$ using the folklore fact that membership oracles are decidable for Turing machines that always halt (\cref{lem:mopDecidableForHaltingMachines}).
        Note that we cannot use this folklore result directly for the generator  $\generator$, since even though $M$ halts in each iteration, $\generator$ may not halt as the number of iterations is not bounded.

        \begin{lemma}\label{lem:mopDecidableForHaltingMachines}
            Consider a (randomized) Turing Machine $M$ that halts on all inputs.
            The following problem is decidable: given strings $s$ and $p$ and a description of $M$, output \textsf{Yes}  if $M$ can output $s$ given input $p$ and output \textsf{No} otherwise.
        \end{lemma}
        The proof of \cref{lem:mopDecidableForHaltingMachines} uses the following straightforward but subtle folklore lemmas.
        \begin{lemma}\label{lem:tm:consecutiveReads}
            Consider a (randomized) Turing Machine $M$ that halts on all inputs.
            $M$ has the following property:
                for each input string $p$, $M$ performs at most a finite number of steps between any consecutive reads of their (internal) tape containing random bits.
        \end{lemma}
        \vspace{-5mm}
        \begin{lemma}\label{lem:tm:finiteBitsRead}
            Consider a (randomized) Turing Machine $M$ that halts on all inputs.
            For each input string $p$, there is a finite number $n_p\geq 1$ such that $M$ reads at most $n_p$ random bits always (regardless of the realization of the random bits). 
        \end{lemma}
        These enable us to prove \cref{lem:mopDecidableForHaltingMachines}.
        \begin{proof}[Proof of \cref{lem:mopDecidableForHaltingMachines}]
            Consider a string $s\in \Sigma^*$ of length $m$.
            We will check if $M$ generates $s$ with positive probability by iteratively checking if, for each $1\leq t\leq m$, $M$ generates token $s_t$ with positive probability conditioned on having generated $s_1\dots s_{t-1}$ so far.
            
            Fix any $1\leq t\leq m$.
            Suppose $M$ has passed all earlier checks and, hence, it generates $s_1\dots s_{t-1}$ with positive probability.
            Now, to complete the check for step $t$, it suffices to check that $M$ generates $s_t$ with positive probability having generated $s_1\dots s_{t-1}$ so far.
            Since $M$ halts on all inputs, \cref{lem:tm:finiteBitsRead} implies that there is a finite $n_t$ such that $M$ reads at most $n_t$ bits of its internal random tape when given the corresponding input.
            Moreover, \cref{lem:tm:consecutiveReads} implies that $M$ performs finitely many operations between each of the $n_t$ consecutive reads of the internal random tape.
            Hence, one can simulate the execution of $M$ in finite time by checking all $2^{n_t}$ possible values of the random bits of $M$.
            If, for any of these $2^{n_t}$ values, $M$ outputs $s_t$ then we know that $M$ passes the test and, otherwise, we know that $M$ never generates $s_t$ when provided the corresponding input.
            
            One subtlety is that we do not know $n_t$.
            This is easy to overcome: since $n_t$ is known to be finite, we can iterate over $n_t\in \N$ until we reach a value $k$ where for each of the $2^k$ values of the first $k$ random bits,  $M$ halts before reading the $(k+1)$-th random bit.
        \end{proof}

        \subsubsection*{Proof of \cref{lem:tm:consecutiveReads}} %
            \begin{proof}[Proof of \cref{lem:tm:consecutiveReads}]
                The statement is vacuously true for $M$ and $p$ if $M$ reads its internal tape at most once on input $p$ always.
                Suppose with positive probability (over the randomness on $M$'s internal random tape), $M$ reads its internal random tape at least twice given input $p$.
                Fix a value $r_1=v$ of the first random bit such that $M$ will (eventually) read the second bit $r_2$ on the random tape.
                Consider the step after $M$ has read $r_1$.
                If $M$ performs a non-finite amount of computation before reading $r_2$, then we have a contradiction to the fact that $M$ is total since we have found an assignment $v$ of the first random bit on which $M$ performs an infinite number of steps.
                Hence, the result follows by contradiction. 
            \end{proof}
            
        \subsubsection*{Proof of \cref{lem:tm:finiteBitsRead}}
            \begin{proof}[Proof of \cref{lem:tm:finiteBitsRead}]
                Fix any input $p$ to $M$.
                Toward a contradiction suppose that for any finite $n\geq 1$, there is (at least) one assignment $v_1,v_2\dots,v_n$ of the first $n$ bits on $M$'s internal random tape on which $M$ will read the $(n+1)$-th random bit before halting.
                Therefore, for any $n\geq 1$, we have an assignment of the random bits for which $M$ rates at least $n+1$ random bits and, hence, perform at least $n+1$ steps before halting.
                This is a contradiction to the fact that $M$ halts always, for each value of the random bits on its internal tape.
            \end{proof}
            
    \subsection{Proof of \texorpdfstring{\Cref{mainthm:gen:mop}}{Theorem 3.3} (Impossibility for Generation with Breadth)}
    \label{sec:gen:mop}

    In this section, we present the proof of
    \Cref{mainthm:gen:mop} in two main
    parts; see \cref{fig:outline:mainthm:gen:mop} for an outline.

    First, we prove that if 
    $\cL$ is not identifiable in the limit,
    then no algorithm in $\mathfrak{G}$ generates
    with breadth from $\cL$ at any rate. Recall that $\mathfrak{G}$ is the class of generating algorithms for which $\mop{}(\cdot)$ is decidable.
        
    \begin{figure}[bth]
        \centering
        \begin{tikzpicture}[node distance=2.5cm, auto]
        \node[myNodeNarrow] (thm) at (0,0) {\footnotesize\cref{mainthm:gen:mop}};
    
        \node[myNode, text width=6cm] (subresA) at (-4.125cm, -2cm) {\footnotesize Results for non-identifiable collections $\cL$};
    
        \node[myNode, text width=6cm] (subresB) at (4.125cm, -2cm) {\footnotesize Results for identifiable collections $\cL$};
        
        \node[myNode, text width=3.8cm] (subres1) at (-6.75cm, -4cm) {\footnotesize Algorithms in $\mathfrak{G}$ cannot generate with breadth from $\cL$ at any rate};
        \node[myNodeFlex, text width=3.8cm] (subres2) at (-2.5cm, -4cm) {\footnotesize Algorithms in $\mathfrak{G}$ generate (without breadth) at (the optimal) exponential rate};
        \node[myNode, text width=3.8cm] (subres3) at (2.5cm, -4cm) {\footnotesize Algorithms in $\mathfrak{G}$ generate with breadth at almost exponential rate};
        \node[myNode, text width=3.8cm] (subres4) at (6.75cm, -4cm) {\footnotesize Generation faster than exponential is impossible for non-trivial $\cL$};

        \node[myNodeNarrow] (lem64) at (-6.75cm, -6cm) {\footnotesize \cref{lem:mainthm:gen:mop-lower-bound}};
        
        \node[myNodeNarrow] (thm31) at (1.125cm, -6cm) {\footnotesize \cref{thm:dichotomy-identification-positive}};
        \node[myNodeNarrow] (lem65) at (3.625cm, -6cm) {\footnotesize \cref{prop:index-to-sampler}};
        
        \node[myNodeNarrow] (thm32) at (-3.5cm, -6cm) {\footnotesize \cref{thm:statistical-generation}};
        \node[myNodeNarrow, text width=2cm] (obs) at (-1.25cm, -7.3cm) {\footnotesize \mbox{$\mop{}(\cdot)$ is} \mbox{decidable for} \mbox{Kleinberg and} Mullainathan \cite{kleinberg2024language}'s algorithm};
        
        \node[myNodeNarrow] (lem510) at (6.75cm, -6cm) {\footnotesize \cref{lem:exp-rates-lower-bound-gen}};
        
        \draw[-stealth, line width=0.5mm] (thm) -- (subresA.north);
        \draw[-stealth, line width=0.5mm] (thm) -- (subresB.north);
        \draw[-stealth, line width=0.5mm] (subresA) -- (subres1.north);
        \draw[-stealth, line width=0.5mm] (subresA) -- (subres2.north);
        \draw[-stealth, line width=0.5mm] (subresB) -- (subres3.north);
        \draw[-stealth, line width=0.5mm] (subresB) -- (subres4.north);
    
        \draw[-stealth, line width=0.5mm] (subres1) -- (lem64.north);
        \draw[-stealth, line width=0.5mm] (subres2) -- (thm32.north);
        \draw[-stealth, line width=0.5mm] (subres2) -- (obs.north);
        \draw[-stealth, line width=0.5mm] (subres3) -- (thm31.north);
        \draw[-stealth, line width=0.5mm] (subres3) -- (lem65.north);
        \draw[-stealth, line width=0.5mm] (subres4) -- (lem510.north);
        
        \end{tikzpicture}
        \caption{Outline of Proof of \cref{mainthm:gen:mop}}
        \label{fig:outline:mainthm:gen:mop}
    \end{figure}
    
    \begin{lemma}\label{lem:mainthm:gen:mop-lower-bound}
        Let $\cL$ be a countable
        collection of languages
        that is not identifiable in the limit.
        Then, for every rate $R$, there is no generating
        algorithm in $\mathfrak{G}$ that 
        can generate from $\cL$ with consistency
        and breadth at rate $R$.
    \end{lemma}

    \begin{proof}
        Let $\cL$ be a countable
        collection of languages that is
        not identifiable in the limit.
        Assume towards a contradiction
        that there exists some generating
        algorithm $(\generator_n) \in \mathfrak{G}$
        that achieves consistency and breadth
        at some rate $R(n).$ Fix also some valid distribution
        $\cP$ supported over a 
        target language $K$. This means 
        that there exist $c, C$, that depend on $\cP,$ such that
        \[
            \E_{X_1,\ldots,X_n \sim \cP^n}\insquare{\ind\inbrace{\supp(\generator_n) \neq K \setminus \inbrace{X_1,\ldots,X_n}}} \leq C\cdot R(c\cdot n) \,.
        \]
        This can be equivalently written as
        \[
            \Pr_{X_1,\ldots,X_n \sim \cP^n}\insquare{\supp(\generator_n) \neq K \setminus \inbrace{X_1,\ldots,X_n}} \leq  C\cdot R(c\cdot n) \,.
        \]
        For every $n \in \N,$ we denote
        by $\cE_n$ the event that $\supp(\generator_n) = K \setminus \inbrace{X_1,\ldots,X_n}.$
        Let $z \in \N$ be the smallest number
        such that $L_z = K.$ 
        Recall that the elements
        of the universe are $\cX = \inbrace{x_1,x_2,\ldots}.$
        Consider the following algorithm $(I_n)_{n \in \N}$ for identification: 
        \begin{itemize}
            \item For every $n \in \N$,
            denote by $\{X_i\}_{i \in [n]}$
            the sample \iid{} from $\cP$. Output the smallest index $j \in [n]$
        such that 
        \[
             \ind\inbrace{x_i \in L_j} = \ind\inbrace{x_i \in \supp(\generator_n) \cup \inbrace{X_1,\ldots,X_n}}, \forall i \in [n]\,.
        \]
        (Since $\generator_n \in \mathfrak{G}$, {$\mop{}(\generator_n)$ is decidable and, hence, the above $j$ can be computed.}) If no such index
        exists, output an index arbitrarily.
        \end{itemize}
        We consider two cases. 
        
        \paragraph{Case A ($z=1$):}
        In this case, notice
        that if $\supp(\generator_n) = K \setminus \{X_1,\ldots,X_n\},$
        then, $I_n(X_1,\ldots,X_n) = z.$
        This is because $\supp(\generator_n) \cup \inbrace{X_1,\ldots,X_n} = K$ and $L_z = K,$ so
        for all $x \in \cX$ it holds 
        $\ind\inbrace{x \in L_z} = \ind\inbrace{x \in \supp(\generator_n) \cup \inbrace{X_1,\ldots,X_n}}.$
        Thus,
        \begin{align*}
        \Pr_{X_1,\ldots,X_n \sim \cP^n}\insquare{L_{I_n(X_1,\ldots,X_n)} \neq K} &\leq
            \Pr_{X_1,\ldots,X_n \sim \cP^n}\insquare{I_n(X_1,\ldots,X_n) \neq z} \\
            & \leq \Pr_{X_1,\ldots,X_n \sim \cP^n}\insquare{I_n(X_1,\ldots,X_n) \neq z \mid \cE^c_n} \cdot \Pr_{X_1,\ldots,X_n \sim \cP^n}\insquare{\cE^c_n} \tag{since under $\cE_n$ the algorithm identifies}\\
            &\leq 1 \cdot \Pr_{X_1,\ldots,X_n \sim \cP^n}\insquare{\cE^c_n} \\
            &\leq  C\cdot R(c\cdot n) \,.
        \end{align*}
        
        \paragraph{Case B ($z>1$):}
        For every language $L_j, j \in [z-1],$
        let ${i_j} \in \N$ be the smallest
       number such that $\ind\inbrace{x_{i_j} \in L_j} \neq \ind\inbrace{x_{i_j} \in L_z}.$ By definition of $L_z,$ we have
       that $i_j$ is well-defined.
       Moreover, let $n^* \coloneqq \max_{j \in [z-1] } i_j.$ Notice that for all $n \geq n^*$, under the 
       event $\cE_n,$ we have that
       $I_n(X_1,\ldots,X_n) = z.$ To see
       why this is the case, notice that
       \begin{enumerate}
           \item Under the event $\cE_n$ it holds that
        $\ind\inbrace{x \in L_z} = \ind\inbrace{x \in \supp(\generator_n) \cup \inbrace{X_1,\ldots,X_n}}$ for all
         $x \in \cX$.

         \item Since $n \geq n^*$, for all $j \in [z-1]$
         we have that $i_j \leq n.$ Thus, under the event $\cE_n$ it cannot be the
         case that:
         \[
             \ind\inbrace{x_i \in L_j} = \ind\inbrace{x_i \in \supp(\generator_n) \cup \inbrace{X_1,\ldots,X_n}}, \forall i \in [n]\,.
        \]
       \end{enumerate}

        Hence, using an identical argument as in the 
        case $z = 1$ we have that
        \begin{align*}
        \Pr_{X_1,\ldots,X_n \sim \cP^n}\insquare{L_{I_n(X_1,\ldots,X_n)} \neq K} \leq  C\cdot R(c\cdot n), \forall n \geq n^*\,.
        \end{align*}
        Since this holds for any valid distribution $\cP$,
        using different $\cP$-dependent constants,
        we see that the algorithm $(I_n)_{n \in \N}$
        can identify $\cL$ at a rate $R.$ Since $\cL$
        is not identifiable in the limit, this 
        contradicts \Cref{mainthm:statisticalRates:iden}, 
        and, hence, concludes the proof.
        
    \end{proof}
    The last ingredient we need to prove \cref{mainthm:gen:mop} is an algorithm that given
    the index of a language, samples from it with breadth.

    \begin{proposition}\label{prop:index-to-sampler}
        There exists a randomized 
        computable algorithm $\cA$ for which $\mop{}(\cdot)$ is decidable {and that,} given as input a number $z \in \N$ and access to a collection of languages $\cL = \{L_1,L_2,\ldots,\},$ satisfies $\supp\inparen{\cA\inparen{z}}=L_z$.
    \end{proposition}

    \begin{proof}
    The algorithm works as follows. Given the index $z$ of the target
    language:
    \begin{itemize}
        \item Sample a natural natural number $\hat n \in \N$ from some distribution supported over $\N.$
        \item If $x_{\hat n} \in L_z,$ (this can be checked by querying the membership oracle) return $x_{\hat n}.$ Otherwise, repeat the previous bullet.
    \end{itemize}
    It follows immediately that this algorithm is computable and satisfies the requirements of the statement, since it is implemented via
    rejection sampling {using a membership oracle to $L_z$ (which is decidable)}, it is indeed in $\mathfrak{G}$.
    \end{proof}

    \noindent We are now ready to prove \Cref{mainthm:gen:mop}.

    \begin{proof}[Proof of \Cref{mainthm:gen:mop}]

   First, consider the case
   that $\cL$ is not identifiable in the limit. Then, for every rate $R$,  \Cref{lem:mainthm:gen:mop-lower-bound} shows that no generating
   algorithm from $\mathfrak{G}$ can generate
   from $\cL$ with consistency and breadth
   at rate $R.$

    We now show that there is a consistent
    generation algorithm for $\cL$
    at an optimal exponential rate.
    Notice that since $\cL$ is 
    non-trivial for generation, by \Cref{lem:exp-rates-lower-bound-gen}, 
    it holds that no algorithm can achieve
    rate faster than exponential.
    Moreover, by \Cref{thm:statistical-generation} there exists
    an algorithm (namely the one by \citecustom{kleinberg2024language}) 
    that achieves exponential rates. Moreover,
    since this algorithm 
    samples from a distribution
    that is a point mass, it is indeed
    in $\mathfrak{G}.$

    Let us now consider the case
    that $\cL$ is identifiable in the limit.
    Then, by \Cref{mainthm:statisticalRates:iden},
    for every $g(n) = o(n),$ there
    exists an algorithm that identifies
    $\cL$ at rate $e^{-g(n)}.$ It is not hard
    to turn this identification algorithm
    into an algorithm that is in $\mathfrak{G}$
    and generates with breadth via rejection sampling. This happens
    as described in \Cref{prop:index-to-sampler}.
    Conditioned on the event that $L_z = K,$
    the previous algorithm indeed generates with
    breadth.
    Finally, since $\cL$ is non-trivial, 
    no algorithm (even outside of $\mathfrak{G}$)
    can achieve a faster than exponential rate for consistent generation (even without breadth), 
    by \Cref{lem:exp-rates-lower-bound-gen}.

    \end{proof}

\begin{remark}\label{rem:kleinberg-infinite-support}
    {
        One subtlety in the above proof is that to use \citecustom{kleinberg2024language}'s algorithm for generation, we require it to output an arbitrarily large number of samples at each step.
        While the vanilla version of \citecustom{kleinberg2024language}'s algorithm only outputs one sample at a time, it can easily be extended so that, given a number $m\geq 1$, it outputs $m$ samples at each step.
    }
    Moreover, the resulting
    algorithm is in $\mathfrak{G}.$
\end{remark}

    \subsection{Proof of \texorpdfstring{\cref{mainthm:gen:limit}}{Theorem 3.5} (Impossibility for Generation with Breadth in the Limit)}\label{sec:proofof:mainmainthm:gen:limit}
        In this section, we prove \cref{mainthm:gen:limit}, which we restate below.
        \thmImpossibleGenerationBreadth*

        \begin{proof}[Proof of \cref{mainthm:gen:limit}]
            The proof of \cref{mainthm:gen:limit} is by a contradiction: we will show that if such a generator exists, it can be used to build an identification algorithm $\algo{I}$ for $\cL$ contradicting the fact that $\cL$ is non-identifiable.
            In addition to the generator $\algo{G}$, this identification algorithm uses another sub-routine: 
                an algorithm $\identifierPN$ that, given a {positive and negative enumeration of the target}, identifies it in the limit.
                Such an identification algorithm always exists due to a result by \citecustom{gold1967language}.
            The identifier $\algo{I}$, which we construct, is as follows:

            \smallskip
            
            \begin{mdframed}
                \textbf{Input:} 
                    Access to a generator $\algo{G}$ for $\cL$ that (1) achieves consistency and breadth in the limit and (2) for which $\mop{(\algo{G})}$ is decidable, and access to the algorithm $\identifierPN$ that identifies $\cL$ in the limit from a positive and negative enumeration of the target language.

                \medskip

                \noindent \textbf{Description:}
                \begin{enumerate}
                    \item \textbf{For each}~~$t\in \N$~~\textbf{do:}
                    \begin{enumerate}
                        \item Observe the $t$-th sample $s_t$ and let $S_t$ be the set of samples seen so far
                        \item Train the generator $\algo{G}$ from scratch over the $t$ samples in $S_t$ 
                        \item Label the first $t$ strings $x_1,\dots,x_t$ of  the domain as $\mop{}(\generator)(x_1),\dots,\mop{}(\generator)(x_t)$\footnote{Here, $\mop{}(\generator)(x)$ is the answer to the membership oracle problem for $\generator$ given input $x$.}
                        \item Train $\identifierPN$ from scratch on samples $x_1,\dots,x_t$ with labels $\mop{}(\generator)(x_1),\dots,\mop{}(\generator)(x_t)$ %
                        \item \textbf{output} the index guessed by $\identifierPN$ and go to the next iteration
                    \end{enumerate}
                \end{enumerate}
            \end{mdframed}
        
            \smallskip

            \noindent 
            Since $\mop{(\generator)}$ is decidable, the above algorithm can be implemented using a Turing machine.
            We claim that the above algorithm identifies the target language $K$ after a finite number of iterations.
            Let $z$ be the first index at which $K$ appears.
            To formalize this, fix any enumeration $s_1,s_2,\dots$ of the target language $K$.
            Since $\algo{G}$ generates with breadth in the limit, there is a finite iteration $t_{\algo{G}}$ after which $K$ generates with breadth from $K$ and, hence, $\supp{(\algo{G})} = K$.
            Hence, after iteration $t_{\algo{G}}$, for any string $x$, $\mop{(\generator)}(x) = \ind\inbrace{x\in K}$.
            In other words, $\identifierPN{}$ is provided with accurate positive and negative labels in all subsequent iterations $t\geq t_{\algo{G}}$.
            Since $\identifierPN{}$ identifies in the limit, there is a finite $t_{\rm PN}$ such that $\identifierPN{}$ identifies $K$ once it is given labels for the first $t\geq t_{\rm PN}$ examples in the domain.
            It follows that $\identifierPN{}$ and, hence, our algorithm identifies $K$ after $\max\inbrace{t_{\algo{G}}, t_{\rm PN}}<\infty$ iterations.
            This gives the desired contradiction, proving \cref{mainthm:gen:limit}.
            Note that the above identification algorithm does not need to know either $t_{\algo{G}}$ or $t_{\rm PN}$.
            (Of course, as a consequence, our algorithm does not know when it has identified $K$.)
        \end{proof}

\section{Proofs from \texorpdfstring{\Cref{sec:ApproxConsBreadth}}{Section 3.3} ({Generation with Relaxations of Breadth})}
        \label{sec:proofRobust}

        \subsection{Proof of \texorpdfstring{\Cref{thm:online:impossibility:robust}}{Theorem 3.7} ({Unambiguous Generation: Online})}
        \label{sec:proofRobust-2}\label{sec:proofof:thm:online:impossibility:robust}
            In this section, we prove \cref{thm:online:impossibility:robust}, which we restate below.
            \thmUnambiguousLimit*
            \begin{proof}[Proof of \cref{thm:online:impossibility:robust}]
             By the way of contradiction, suppose that there is an algorithm $\generator=\inparen{\generator_n}$ for which $\mop{}(\cdot)$ is decidable (at each $n$) and which is an unambiguous generator for $\cL$.
             We will use $\generator$ to construct an algorithm that identifies $\cL$ in the limit, hence, contradicting the non-identifiablity of $\cL$.
    
             Fix any enumeration $x_1,x_2,\dots$ of the domain $\cX$.
             For each language $L\in \cL$ and number $t\geq 1$, define the $t$-prefix of $L$ as the subset $L[t]$ of the first $t$-elements of the domain $\inbrace{x_1,\dots,x_t}$ in $\cL$, \ie{}, 
             \[
                L[t]\coloneqq \inbrace{x_1,\dots,x_t}\cap L\,. 
             \]
             To complete the above outline, consider the following algorithm, which we claim identifies $\cL$.
             \begin{mdframed}
                \textbf{Input:} 
                    Access to a generator $\algo{G}$ for $\cL$ that (1) that is unambiguous in the limit and (2) for which $\mop{(\cdot)}$ is decidable at each step %
    
                \medskip
                
                \noindent  \textbf{Description:}
                    \begin{enumerate}
                        \item \textbf{For each}~~$t\in \N$~~\textbf{do:}
                        \vspace{-2mm}
                    \begin{enumerate}
                        \item Observe the $t$-th sample $s_t$ and let $S_t$ be the set of samples seen so far
                        \item Train the generator $\algo{G}_{t-1}$ on $s_t$ to get $\algo{G}_t$
                        \item 
                            Create a set of languages consistent  with observed samples $C_{S}(t)\subseteq \inbrace{L_1,\dots,L_t}$ that includes each $L\in\inbrace{L_1,\dots,L_t}$ that is consistent with $S_t$ (\ie{}, $L\supseteq S_t$)
                        \item Construct a set of languages consistent with the generator $C_{G}(t) \subseteq \inbrace{L_1,\dots,L_t}$ that  languages $L$ from $\inbrace{L_1,\dots,L_t}$ except if {$L[t]\not\subseteq \supp(\generator_t)\cup S_t$,} which can be checked in finite time using a decider for $\mop{}(\generator_t)$ 
                        \item \textbf{output} the index of the smallest-indexed language in $C_{S}(t)\cap C_{G}(t)$ (or an arbitrary index if $C_{S}(t)\cap C_{G}(t)$ is empty)
                    \end{enumerate}
                    \end{enumerate} 
            \end{mdframed}
            Since the algorithm outputs the smallest index in $C_{S}(t)\cap C_{G}(t)$, it identifies $K$ if it is the smallest-indexed language in $C_{S}(t)\cap C_{G}(t).$
            The following conditions ensure this:
            \begin{itemize}
                \item[(A)] $L_z\in C_{S}(t)$ and $L_z\in C_{G}(t)$, where $z$ is the smallest index at which $K$ appears in $\cL$; and 
                \item[(B)] For any $i < z$, either $L_i\not\in C_S(t)$ or $L_i\not\in C_G(t)$
            \end{itemize}
            We claim that there are finite times $t_a$ and $t_b$ where, for any $t\geq t_a$, Condition (A) holds and, for any $t\geq t_b$, Condition (B) holds. 
            This claim implies that the above algorithm identifies $K$ in the limit, leading to the desired contradiction.
    
            \paragraph{Condition A holds after a finite time.}    
                Since $S_t$ only contains samples from $K$, $K\in C_S(t)$ for all $t\geq 1$.
                Further, since $\generator=\inparen{\generator_t}$  is an unambiguous generator for $\cL$ in the limit, there exists a finite $t_0\geq 0$, such that for all $t\geq t_0$,
                \[
                    \abs{\supp(\generator_n)\triangle K}
                    < 
                    \min_{L\in \cL\colon L\neq K}
                    \abs{\supp(\generator_n)\triangle L}\,.
                    \yesnum\label{eq:robust:umabiguous:proof}
                \]
                Hence, in particular, for all $t\geq t_0$
                \[
                    \abs{\supp(\generator_n)\setminus K}\,,~ \abs{K\setminus \supp(\generator_n)}
                    < \infty\,.
                \]
                Furthermore, as $\generator$ is stable, after some time $t_1$, $\supp(\generator_n)$ stops changing.
                Consider any $t\geq \max\inbrace{t_0,t_1}$.
                Since ${K\setminus \supp(\generator_t)}$ has finitely many elements and ${K\setminus \supp(\generator_t)=K\setminus \supp(\generator_{t'})}$ for any $t'\geq  \max\inbrace{t_0,t_1}$, 
                there is a finite time $t_2$ after which all elements of ${K\setminus \supp(\generator_t)}$ have been observed.
                Therefore, for any $t\geq \inbrace{t_0,t_1,t_2}$, it must holds that ${K\subseteq \inparen{\supp(\generator_t)  \cup S_t}}$ and, hence, that $K\in C_G(t)$.
                Thus, it suffices to fix $t_a= \max\inbrace{t_0,t_1,t_2}$.

            \begin{figure}[bt!]
                \centering
                \includegraphics[width=0.5\linewidth,trim={-1cm 0cm 7.5cm 0cm}]{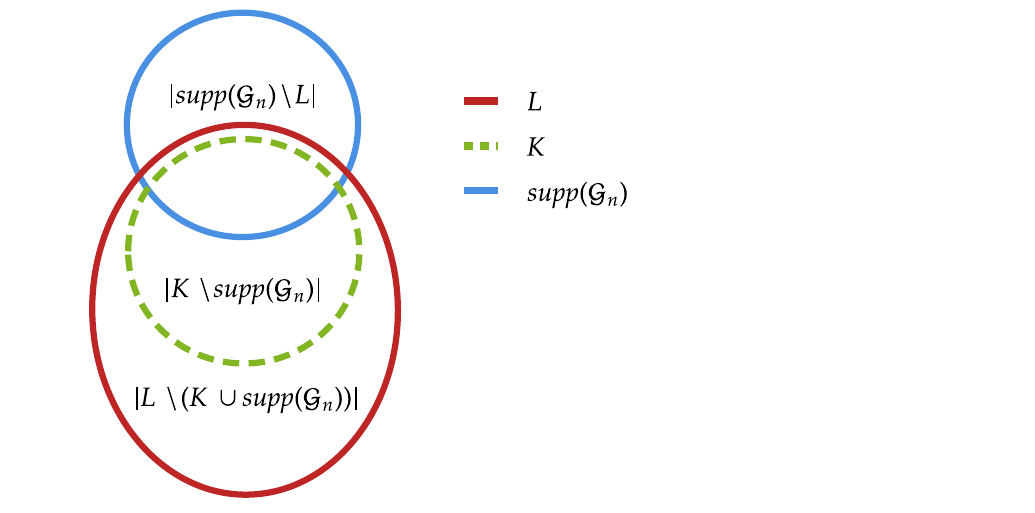}
                \caption{Figure illustrating the decomposition of $\supp(\generator_t)\triangle K$ and $\supp(\generator_t)\triangle L$.}
                \label{fig:unambiguousGeneration}
            \end{figure}
        
            \paragraph{Condition B holds after a finite time.}
                Since there are only finitely many $i < z$, it suffices to show that for each $i< z$, there is a finite time $t_i$ after which either $L_i\not\in C_S(t)$ or $L_i\not\in C_G(t)$.
                Fix any $i<z$.
                Consider two cases:
                \begin{itemize}
                    \item \textbf{Case A ($K\setminus L_i\neq \emptyset$):}    
                        In this case, there exists an $x\in K\setminus L_i$ and, hence, after some finite time $t_i$ when $x$ has been observed $L_i\not\supseteq S_t$ and, hence, $L_i\not\in C_S(t).$
                    \item \textbf{Case B ($K\setminus L_i = \emptyset$):}
                        In this case, $L_i\supsetneq K$, and our proof is based on the following observation.
                    \begin{lemma}
                        For any $t\geq t_a$ and $L\supsetneq K$, it holds that $L\setminus \supp(\generator_t)\neq \emptyset$.
                    \end{lemma}
                    \begin{proof}
                        {
                            Since $t\geq t_0$, \cref{eq:robust:umabiguous:proof} holds.
                            From \cref{fig:unambiguousGeneration}, observe that 
                            \begin{align*}
                                \abs{\supp(\generator_t)\triangle K}
                                &\geq 
                                    \abs{\supp(\generator_t)\backslash L}
                                    + 
                                    \abs{K \backslash \supp(\generator_t)}\\
                                    \abs{\supp(\generator_t)\triangle L}
                                &=
                                    \abs{\supp(\generator_t)\backslash L}
                                    + 
                                    \abs{K \backslash \supp(\generator_t)}
                                    + \abs{L\backslash \inparen{K \cup \supp(\generator_t)}}\,.
                            \end{align*}
                            Chaining the above with \cref{eq:robust:umabiguous:proof} and canceling like terms implies 
                            \[
                                \abs{L\backslash \inparen{K \cup \supp(\generator_t)}} > 0\,.
                            \]
                            Hence, in particular, $\abs{L\backslash {\supp(\generator_t)}} > 0$, which is the desired result.
                        }
                    \end{proof}
                    {
                        Hence, in this case, $L_i\setminus \supp(\generator_t)\neq \emptyset$.
                    }
                    Let $j(i)$ be the smallest natural number such that $x_{j(i)}\in L$ but $x_{j(i)}\not\in \supp(\generator_t)$.
                    (Note that the value $j(i)$ does not depend on $t$, since as discussed in the proof of Condition A after $t = t_a$, the $\supp(\generator_t)$ becomes stable and not change in subsequent iterations.)
                    Therefore, it follows that $L_i[j(i)]\not\subseteq \supp(\generator_t){\cup S_t}$ for $t\geq t_a$ and, hence, for $t\geq \max\inbrace{i(j), t_a}$ by construction, $L_i\not\in C_G(t)$.
                This completes the proof of Case $B$ and by earlier discussion, also the proof of \cref{thm:online:impossibility:robust}.
                        
                \end{itemize}

            \end{proof}

        \subsection{Proof of \texorpdfstring{\cref{thm:impossibility:robust}}{Theorem 3.6} ({Unambiguous Generation: Statistical})}
        \label{sec:proofRobust-1}

        In this section, we prove the impossibility
        result for unambiguous generation in the statistical setting. 
        Our approach is to establish
        a connection to the online setting 
        and leverage the impossibility result we have already shown there (\cref{thm:online:impossibility:robust}). 
        Namely, we will show that given such
        an unambiguous generator that works in the 
        statistical setting, we can construct
        a generator that works in the
        online setting, with high probability.
        Using the construction from \Cref{sec:proofof:thm:online:impossibility:robust}, we can turn this generator to
        one that \emph{identifies}.
        The details of our approach follow.
        
        First, we describe some constructions
        due to \citecustom{angluin1988identifying}
        that will be useful for our derivation.
        The following can be found in Example 3 from \citecustom{angluin1988identifying}.

        \begin{definition}[Distribution Induced by Sequence]\label{def:distribution-induced-sequence}
            Let $\sigma = (x_{i_1},x_{i_2},\ldots,)$
            be some countable sequence of 
            elements in $\cX$, and $\sigma_j = x_{i_j}, j \in \N$. Define $\cP_\sigma$ to be a distribution
            such that its mass $\cP_\sigma (x)$ on any point $x \in \cX$ is
            \begin{align*}
                \cP_\sigma(x) \coloneqq  \sum_{j \in \N\colon  \sigma_j = x} \frac{1}{2^{j+1}} \,,
            \end{align*}
            with a sum over an empty set of indices interpreted as 0.
        \end{definition}
        \citecustom{angluin1988identifying} describes a way to draw \iid{} samples from $\cP_\sigma,$ given only access to finite
        prefixes of $\sigma$ and to an oracle that simulates a fair coin.\footnote{In fact, Angluin's construction generalizes
        to oracles that simulate any (non-deterministic) coin in a straightforward way.} The idea is natural: flip the fair coin until a head is observed, let $I$ be the random variable denoting the number of trials it needed, and output
        the string $x_{I+1}.$ This process gives \iid{} 
        draws from $\cP_\sigma.$

        \begin{proposition}[Example 4 from \citet{angluin1988identifying}]\label{prop:angluin-sampling}
            Let $\sigma = (x_{i_1},x_{i_2},\ldots,)$
            be some countable sequence of 
            elements in $\cX$. Given access to an oracle
            that simulates fair coin flips and an oracle
            which given input any $j \in \N$ returns $\sigma_j,$
            there exists a computable algorithm that samples
            from $\cP_\sigma$ (\Cref{def:distribution-induced-sequence})
            and terminates with probability 1.
        \end{proposition}
The next result shows that a stable generating
algorithm for which \mop{}$(\cdot)$ is decidable and achieves unambiguous
generation at some rate $R(\cdot)$, ``stabilizes''
to an unambiguous generator when executed on 
an infinite \iid{} stream of data drawn from a valid
distribution.

\begin{lemma}\label{lem:stable+rates-convergence-iid-draws}
    Let $R\colon\N \rightarrow \R_{\geq 0}$ be a rate function, \ie{}, $\lim_{n \rightarrow \infty} R(n) = 0,$ let $\cL$ be a countable language collection, and $\inparen{\generator_n\colon \cX^n \rightarrow \mathfrak{G}}_{n \in \N}$ be a generating algorithm for which
    \mop{}$(\cdot)$ is decidable and which satisfies the following two properties:
    \begin{itemize}
        \item $\inparen{\generator_n}_{n \in \N}$ is a stable generator (\Cref{def:stable-generators}), and
        \item for its unambiguous generation error $\mathrm{er}(\cdot)$, it holds that, for every valid distribution $\cP$ with respect to $\cL$ there exist $c,C > 0$ such that $\Ex_{X_1,\ldots,X_n\sim \cP^n}\insquare{\mathrm{er}\inparen{\generator_n\inparen{X_1,\ldots,X_n}}} \leq C\cdot R(c\cdot n).$
    \end{itemize}
    Then, for every valid distribution $\cP$ with respect to $\cL$ it holds that
    \[
        \Pr_{\inbrace{X_i}_{i \in \N} \sim \cP^\infty}\insquare{\exists n^* \in \N\colon  \forall n \geq n^* \text{ it holds that } {\mathrm{er}\inparen{\generator_n\inparen{X_1,\ldots,X_n}} = 0}} = 1 \,.
    \]
\end{lemma}
{In other words, the generating algorithm $\generator=\inparen{\generator_n}$ stabilizes to an unambiguous generation in the online sense with probability 1.
Roughly speaking, given the above result, \cref{thm:impossibility:robust} will follow from a contradiction to the impossibility result in the online setting \cref{thm:online:impossibility:robust}.}
\begin{proof}[Proof of \cref{lem:stable+rates-convergence-iid-draws}]
    Assume towards contradiction that there exists some valid $\cP$ with respect to $\cL$
    so that
     \[
        \Pr_{\inbrace{X_i}_{i \in \N} \sim \cP^\infty}\insquare{\exists n^* \in \N\colon  \forall n \geq n^* \text{ it holds that } \inbrace{\mathrm{er}\inparen{\generator_n\inparen{X_1,\ldots,X_n}} = 0}} = c' < 1\,.
    \]
    Let us also denote $c'' \coloneqq 1-c'.$ Notice that $c'' > 0.$
        Since $\cP$ is a valid distribution with respect to $\cL,$ 
    it is supported over some $K \in \cL,$ so 
    we have that, with probability 1, an infinite \iid{} draw 
    from $\cP$ is an enumeration of $K$ (see \Cref{prop:support-appears-in-countable-samples}). Let us call this event $\cE_1.$

    Moreover, since $\generator_n$ is a stable generator {(in an online sense)}, under the
    event $\cE_1$ {(\ie{}, when the samples from $\cP$ form an enumeration of $K$)},  there exists some smallest number $t^* \coloneqq t^*(X_1,\ldots)\in \N$ such that
    for all $n \geq t^*$
    \[
        \supp\inparen{\generator_n\inparen{X_1,\ldots,X_n}}  = \supp\inparen{\generator_{n+1}\inparen{X_1,\ldots,X_{n+1}}}  \,.
    \]
    {Now, $t^*$ depends on the specific enumeration drawn and, hence,} the distribution $\cP$ induces a distribution over $t^*.$
    {Further, note that} with probability 1, $t^* < \infty.$ 
    {Hence, $\Pr_{\inbrace{X_i}_{i \in \N} \sim \cP^\infty}\insquare{t^*(X_1,\ldots) > n}$ approaches 0 as $n\to\infty$.}
 {In particular,}
    there is some number $n_1 \in \N$ such that for all $n \geq n_1$ 
    \[
        \Pr_{\inbrace{X_i}_{i \in \N} \sim \cP^\infty}\insquare{t^*(X_1,\ldots) > n} \leq \frac{c''}{3} \,.
    \]
    Moreover, since the generator achieves rate $R(\cdot)$ {and $\lim_{n\to \infty}R(n) = 0$}, it holds
    that
    \[
        \lim_{n \rightarrow \infty}\Pr_{X_1,\ldots,X_n \sim \cP^n}\insquare{\mathrm{er}\inparen{\generator_n\inparen{X_1,\ldots,X_n}} \neq 0} = 0\,.
    \]
    Thus, 
    there is 
    some $n_2 \in \N$ such that, for all $n \geq n_2$
    \[
        \Pr_{X_1,\ldots,X_n \sim \cP^n}\insquare{\mathrm{er}\inparen{\generator_n\inparen{X_1,\ldots,X_n}} \neq 0} \leq \frac{c''}{3} \,.
    \]
    Let $n_3 \coloneqq \max\inbrace{n_1,n_2}$. 
    Hence, taking a union bound, we see that with probability at least 
    $1-{\nicefrac{2c''}{3}}$ over the draw of $\inbrace{X_i}_{i \in \N}$
    it holds that 
    \begin{itemize}
        \item $\mathrm{er}\inparen{\generator_{n_3}\inparen{X_1,\ldots,X_{n_3}}} = 0,$ and
        \item $\supp\inparen{\generator_n\inparen{X_1,\ldots,X_n}}  = \supp\inparen{\generator_{n_3}\inparen{X_1,\ldots,X_{n_3}}} ,$ for all $n \geq n_3.$
    \end{itemize}
    By the definition of $\mathrm{er}(\cdot)$ (\Cref{eq:unambiguous-error},)
    for any $n, n' \in \N$, samples $x_{i_1},\ldots,x_{i_n}$
    and $x_{j_1},\ldots,x_{j_{n'}}$
    it holds that
    \begin{align*}
        \supp\inparen{\generator_n\inparen{x_{i_1},\ldots,x_{i_n}}}  &= \supp\inparen{\generator_{n'}\inparen{x_{j_1},\ldots,x_{j_{n'}}}}  \implies \\
     \mathrm{er}\inparen{\generator_{n}\inparen{x_{i_1},\ldots,x_{i_n}}} &=
      \mathrm{er}\inparen{\generator_{n}\inparen{x_{j_1},\ldots,x_{j_{n'}}}}
    \end{align*}
    These two conditions immediately imply that, with probability
    at least $1-\nicefrac{2c''}{3} > c',$  for all $n \geq n_3$ it holds that
    \begin{itemize}
        \item $\mathrm{er}\inparen{\generator_{n}\inparen{X_1,\ldots,X_{n}}} = 0,$ and
        \item $\supp\inparen{\generator_{n+1}\inparen{X_1,\ldots,X_{n+1}}}  = \supp\inparen{\generator_{n}\inparen{X_1,\ldots,X_{n}}} .$ 
    \end{itemize}
    Hence, 
    \[
        \Pr_{\inbrace{X_i}_{i \in \N} \sim \cP^\infty}\insquare{\exists n^* \in \N\colon \forall n \geq n^* \text{ it holds that } \inbrace{\mathrm{er}\inparen{\generator_n\inparen{X_1,\ldots,X_n}} = 0}} > c'\,,
    \]
    which gives the desired contradiction. This concludes the proof.
\end{proof}

\noindent Having established the previous result, we are ready to show how to use such a generator that works in the statistical setting to get a
generator in the online setting. The idea of the proof is to use
the enumeration $\sigma$ provided from the adversary
to define a valid distribution $\cP_\sigma$ (see \Cref{prop:angluin-sampling}) and then
run the aforementioned generator on this distribution.

\begin{lemma}\label{lem:unambiguous-statistical-to-online}
     Let $R\colon\N \rightarrow \R_{\geq 0}$ be a rate function, \ie{}, $\lim_{n \rightarrow \infty} R(n) = 0,$ let $\cL$ be a language collection, and $\inparen{\generator_n\colon \cX^n \rightarrow \mathfrak{G}}_{n \in \N}$ be generating algorithm for which
    \mop{}~is decidable and satisfies the following two properties:
    \begin{itemize}
        \item $\inparen{\generator_n}_{n \in \N}$ is a stable generator (\Cref{def:stable-generators}), and
        \item for its unambiguous generation error, it holds that, for every valid distribution $\cP$ with respect to $\cL$ there exist $c,C > 0$ such that $E_{X_1,\ldots,X_n\sim \cP^n}\insquare{\mathrm{er}\inparen{\generator_n\inparen{X_1,\ldots,X_n}}} \leq C\cdot R(c\cdot n).$
    \end{itemize}
    Then, there is a randomized generating algorithm $\inparen{\generator'_n\colon \cX^n \overset{r}{\rightarrow} \mathfrak{G}}_{n \in \N}$
    for which, for any target language $K \in \cL$ and every enumeration $\sigma$ of $K$, it holds that
    \begin{itemize}
        \item $\inparen{\generator'_n}_{n \in \N}$ is a stable generator (\Cref{def:stable-generators}), and
        \item 
        \[
            \Pr\insquare{\exists n^* \in \N\colon \forall n \geq n^* \text{ it holds that } \mathrm{er}\inparen{\generator'_n\inparen{\sigma_1,\ldots,\sigma_n}} = 0} = 1\,,
        \]
        where the probability is with respect to the internal randomness of the algorithm.
    \end{itemize}
\end{lemma}

\begin{proof}
    Let $K \in \cL$ be any target language and $\sigma$ be any enumeration of $K.$ Let $\cP_\sigma$ be the distribution defined in \Cref{def:distribution-induced-sequence}. We know that, by definition,
    $\cP_\sigma$ is valid with respect to $\cL,$ since it is supported on $K.$ Let $\inparen{\generator'_n}_{n \in \N}$ be a generating algorithm
    which, for every $n \in \N,$ runs $\generator_n$ on $\cP_\sigma.$ 
    In order to draw samples from $\cP_\sigma$ the generator $\generator'_n$ uses its internal randomness and the process
    described in \Cref{prop:angluin-sampling}. Since $\cP_\sigma$ is
    a valid distribution with respect to $\cL,$ \Cref{lem:stable+rates-convergence-iid-draws} gives us that
    \[
        \Pr_{\inbrace{X_i}_{i \in \N} \sim \cP_\sigma^\infty}\insquare{\exists n^* \in \N\colon \forall n \geq n^* \text{ it holds that } \inbrace{\mathrm{er}\inparen{\generator_n\inparen{X_1,\ldots,X_n}} = 0}} = 1\,.
    \]
    Hence, this implies that
    \[
        \Pr\insquare{\exists n^* \in \N\colon \forall n \geq n^* \text{ it holds that } \inbrace{\mathrm{er}\inparen{\generator'_n\inparen{\sigma_1,\ldots,\sigma_n}} = 0}} = 1\,,
    \]
    where the probability is taken with respect to the internal randomness
    of the algorithm.
    Moreover, since $(\generator_n)_{n \in \N}$ is a stable generator
    it also holds that $(\generator'_n)_{n \in \N}$ is a stable generator.
    This concludes the proof.
\end{proof}

\noindent We are now ready to prove \Cref{thm:impossibility:robust}, which
follows as corollary of \Cref{lem:unambiguous-statistical-to-online}
and the impossibility result from the online setting (\Cref{thm:online:impossibility:robust}).

\begin{proof}[Proof of \Cref{thm:impossibility:robust}]
    Let $\cL$ be a countable collection of languages. 
    Assume that such a stable generating algorithm exists. Then, using
    the construction from \Cref{lem:unambiguous-statistical-to-online}
    we get a stable generator that generates unambiguously in the limit, 
    for every target language $K \in \cL$ and every enumeration $\sigma$
    of $K,$ with probability 1. This contradicts the impossibility
    result from \Cref{thm:online:impossibility:robust}.
\end{proof}

        \subsection{{Proof of \texorpdfstring{\cref{thm:online:impossibility:robustOneSided}}{Theorem 3.9} (Approximate Breadth: Online)}}
            \label{sec:proofof:thm:online:impossibility:robustOneSided}
            
            {In this section, we prove \cref{thm:online:impossibility:robustOneSided}, which we restate below.} 
            
            \thmApproxBreadthOnline*
 
            \begin{proof}[Proof of \cref{thm:online:impossibility:robustOneSided}]
                Recall that \citecustom{gold1967language} showed that for any collection of countably many languages, there is always an algorithm $\identifierPN{}$ that, given a positive and negative enumeration of the target, identifies it in the limit. 
                Now, toward a contradiction, suppose there is a stable generator for which $\mop{}(\generator{})$ is decidable and it has the property described in \cref{def:gen:oneSidedRobust}.
                We claim that using $\generator{}$ and $\identifierPN{}$, we can construct an identifier for $\cL$ (from positive examples), which contradicts the fact that $\cL$ is non-identifiable.
                
                Fix any target language $K$, its enumeration $s_1,s_2,\dots$, and the (unknown) constant $t^*$, such that after $t^*$ many
                iterations the algorithm
                achieves generation according to \cref{def:gen:oneSidedRobust}.
                We claim that the following algorithm identifies $\cL$.
                
            \smallskip
            
            \begin{mdframed}
                \textbf{Input:} 
                    Access to a generator $\algo{G}$ for $\cL$ that (1) that, in the limit, becomes consistent and satisfies  $\abs{K\setminus \supp{}(\generator{})}<\infty$ and (2) for which $\mop{(\algo{G})}$ is decidable, and access to the algorithm $\identifierPN$ that identifies $\cL$ in the limit from a positive and negative enumeration of the target language.

                \medskip
                \noindent\textbf{Description:}
                \begin{enumerate}
                    \item \textbf{For each}~~$t\in \N$~~\textbf{do:}
                    \begin{enumerate}
                        \item Observe the $t$-th sample $s_t$ and let $S_t$ be the set of samples seen so far
                        \item Train the generator $\algo{G}$ from scratch over the $t$ samples in $S_t$ 
                        \item For each $1\leq i\leq t$, label the $i$-th string $x_i$ in the domain as $y_i=\mop{}(\generator)(x_i)$\footnote{Here, $\mop{}(\generator)(x)$ is the answer to the membership oracle problem for $\generator$ given input $x$.}
                        \item For each $1\leq i\leq t$, if $x_i\in S_t$ and $y_i=0$, set $y_i=1$ \textit{$\#$ to correct elements missed by $\generator{}$ }
                        \item Train $\identifierPN$ from scratch on samples $x_1,\dots,x_t$ with labels $y_1,\dots,y_t$
                        \item \textbf{output} the index guessed by $\identifierPN$ and go to the next iteration
                    \end{enumerate}
                \end{enumerate}
            \end{mdframed}
            
            \smallskip

            \noindent 
                Since $\mop{(\generator)}$ is decidable, the above algorithm can be implemented using a Turing machine.
                We claim that the above algorithm identifies the target language $K$ after a finite number of iterations.
                To formalize this, fix any enumeration $s_1,s_2,\dots$ of the target language.
                Since after a finite time $t^*$, $\generator$ becomes consistent, stabilizes, and satisfies $\abs{K\setminus \supp{(\generator{})}} < \infty$, after iteration $t^*$, for any string $x$, $\mop{(\generator)(x)}$ matches $\ind\inbrace{x\in K}$ except for the finitely many elements of $M=K\setminus \supp{(\generator{}_t)}$.
                Note that since $\generator$'s support stabilizes, $M$ is independent of the iteration $t\geq t^*$.
                Let $t'$ be the time when all elements of $M$ appear in the enumeration $s_1,s_2,\dots$.
                Observe that in all iterations $t\geq \max\inbrace{t^*, t'}$, the labels in Step 2 (d) correct all the mismatches between $\mop{(\generator)(x)}$ matches $\ind\inbrace{x\in K}$
                Finally, since $\identifierPN{}$ identifies in the limit, there is a finite $t_{\rm PN}$ such that $\identifierPN{}$ identifies $K$ once it is given labels for the first $t\geq t_{\rm PN}$ examples in the domain.
                Combining this with the previous information implies that that $\identifierPN{}$ and, hence, our algorithm identifies $K$ after $\max\inbrace{t^*, t', t_{\rm PN}}<\infty$ iterations.
                This gives the desired contradiction, proving \cref{thm:online:impossibility:robustOneSided}.
                Note that the above identification algorithm does not need to know any of $t^*$, $t'$, and $t_{\rm PN}$.
            \end{proof}

        \subsection{{Proof of \texorpdfstring{\cref{thm:approxBreadth:statistical}}{Theorem 3.8} (Approximate Breadth: Statistical)}}
            \label{sec:proofof:thm:approxBreadth:statistical}
            {In this section, we prove \cref{thm:approxBreadth:statistical}, which we restate below.}
            
            \thmApproxBreadthStat*
        \noindent We start with the following lemma.

        \begin{lemma}\label{lem:stable+rates-approximate-convergence-iid-draws}
    Let $R\colon\N \rightarrow \R_{\geq 0}$ be a rate function, \ie{}, $\lim_{n \rightarrow \infty} R(n) = 0,$ let $\cL$ be a language collection, and $\inparen{\generator_n\colon \cX^n \rightarrow \mathfrak{G}}_{n \in \N}$ be generating algorithm for which
    \mop{}$(\cdot)$ is decidable and which satisfies the following two properties:
    \begin{itemize}
        \item $\inparen{\generator_n}_{n \in \N}$ is a stable generator (\Cref{def:stable-generators}), and
        \item for its approximate generation error $\mathrm{er}(\cdot)$ (\Cref{eq:errorMissFinite}) it holds that, for every valid distribution $\cP$ with respect to $\cL$ there exist $c,C > 0$ such that $\Ex_{X_1,\ldots,X_n\sim \cP^n}\insquare{\mathrm{er}\inparen{\generator_n\inparen{X_1,\ldots,X_n}}} \leq C\cdot R(c\cdot n).$
    \end{itemize}
    Then, for every valid distribution $\cP$ with respect to $\cL$ it holds that
    \[
        \Pr_{\inbrace{X_i}_{i \in \N} \sim \cP^\infty}\insquare{\exists n^* \in \N\colon  \forall n \geq n^* \text{ it holds that } {\mathrm{er}\inparen{\generator_n\inparen{X_1,\ldots,X_n}} = 0}} = 1 \,.
    \]
\end{lemma}
\begin{proof}[Proof of \cref{lem:stable+rates-approximate-convergence-iid-draws}]
    Assume towards contradiction that there exists some valid $\cP$ with respect to $\cL$
    so that
     \[
        \Pr_{\inbrace{X_i}_{i \in \N} \sim \cP^\infty}\insquare{\exists n^* \in \N\colon  \forall n \geq n^* \text{ it holds that } \inbrace{\mathrm{er}\inparen{\generator_n\inparen{X_1,\ldots,X_n}} = 0}} = c' < 1\,.
    \]
    Let us also denote $c'' \coloneqq 1-c'.$ Notice that $c'' > 0.$
        Since $\cP$ is a valid distribution with respect to $\cL,$ 
    it is supported over some $K \in \cL,$ so 
    we have that, with probability 1, an infinite \iid{} draw 
    from $\cP$ is an enumeration of $K$ (see \Cref{prop:support-appears-in-countable-samples}). Let us call this event $\cE_1.$

    Moreover, since $\generator_n$ is a stable generator {(in an online sense)}, under the
    event $\cE_1$ {(\ie{}, when the samples from $\cP$ form an enumeration of $K$)},  there exists some smallest number $t^* \coloneqq t^*(X_1,\ldots)\in \N$ such that
    for all $n \geq t^*$
    \[
        \supp\inparen{\generator_n\inparen{X_1,\ldots,X_n}} = \supp\inparen{\generator_{n+1}\inparen{X_1,\ldots,X_{n+1}}} \,.
    \]
    {Now, $t^*$ depends on the specific enumeration drawn and, hence,} the distribution $\cP$ induces a distribution over $t^*.$
    {Further, note that} with probability 1, $t^* < \infty.$ 
    {Hence, $\Pr_{\inbrace{X_i}_{i \in \N} \sim \cP^\infty}\insquare{t^*(X_1,\ldots) > n}$ approaches 0 as $n\to\infty$.}
 {In particular,}
    there is some number $n_1 \in \N$ such that for all $n \geq n_1$ 
    \[
        \Pr_{\inbrace{X_i}_{i \in \N} \sim \cP^\infty}\insquare{t^*(X_1,\ldots) > n} \leq \frac{c''}{3} \,.
    \]
    Moreover, since the generator achieves rate $R(\cdot)$ {and $\lim_{n\to \infty}R(n) = 0$}, it holds
    that
    \[
        \lim_{n \rightarrow \infty}\Pr_{X_1,\ldots,X_n \sim \cP^n}\insquare{\mathrm{er}\inparen{\generator_n\inparen{X_1,\ldots,X_n}} \neq 0} = 0\,.
    \]
    Thus, 
    there is 
    some $n_2 \in \N$ such that, for all $n \geq n_2$
    \[
        \Pr_{X_1,\ldots,X_n \sim \cP^n}\insquare{\mathrm{er}\inparen{\generator_n\inparen{X_1,\ldots,X_n}} \neq 0} \leq \frac{c''}{3} \,.
    \]
    Let $n_3 \coloneqq \max\inbrace{n_1,n_2}$. 
    Hence, taking a union bound, we see that with probability at least 
    $1-{\nicefrac{2c''}{3}}$ over the draw of $\inbrace{X_i}_{i \in \N}$
    it holds that 
    \begin{itemize}
        \item $\mathrm{er}\inparen{\generator_{n_3}\inparen{X_1,\ldots,X_{n_3}}} = 0,$ and
        \item $\supp\inparen{\generator_n\inparen{X_1,\ldots,X_n}}  = \supp\inparen{\generator_{n_3}\inparen{X_1,\ldots,X_{n_3}}},$ for all $n \geq n_3.$
    \end{itemize}
    By the definition of $\mathrm{er}(\cdot),$ 
    for any $n, n' \in \N$, samples $x_{i_1},\ldots,x_{i_n}$
    and $x_{j_1},\ldots,x_{j_{n'}}$
    it holds that
    \begin{align*}
        \supp\inparen{\generator_n\inparen{x_{i_1},\ldots,x_{i_n}}}  &= \supp\inparen{\generator_{n'}\inparen{x_{j_1},\ldots,x_{j_{n'}}}}  \implies \\
     \mathrm{er}\inparen{\generator_{n}\inparen{x_{i_1},\ldots,x_{i_n}}} &=
      \mathrm{er}\inparen{\generator_{n}\inparen{x_{j_1},\ldots,x_{j_{n'}}}}
    \end{align*}
    These two conditions immediately imply that, with probability
    at least $1-\nicefrac{2c''}{3} > c',$  for all $n \geq n_3$ it holds that
    \begin{itemize}
        \item $\mathrm{er}\inparen{\generator_{n}\inparen{X_1,\ldots,X_{n}}} = 0,$ and
        \item $\supp\inparen{\generator_{n+1}\inparen{X_1,\ldots,X_{n+1}}}  = \supp\inparen{\generator_{n}\inparen{X_1,\ldots,X_{n}}}.$ 
    \end{itemize}
    Hence, 
    \[
        \Pr_{\inbrace{X_i}_{i \in \N} \sim \cP^\infty}\insquare{\exists n^* \in \N\colon \forall n \geq n^* \text{ it holds that } \inbrace{\mathrm{er}\inparen{\generator_n\inparen{X_1,\ldots,X_n}} = 0}} > c'\,,
    \]
    which gives the desired contradiction. This concludes the proof.
\end{proof}
Using the previous result, we derive the next statement regarding the conversion of a statistical learner to an online one.

\begin{lemma}\label{lem:missing-finite-elements-statistical-to-online}
     Let $R\colon\N \rightarrow \R_{\geq 0}$ be a rate function, \ie{}, $\lim_{n \rightarrow \infty} R(n) = 0,$ let $\cL$ be a language collection, and $\inparen{\generator_n\colon \cX^n \rightarrow \mathfrak{G}}_{n \in \N}$ be generating algorithm for which
    \mop{}~is decidable and satisfies the following two properties:
    \begin{itemize}
        \item $\inparen{\generator_n}_{n \in \N}$ is a stable generator (\Cref{def:stable-generators}), and
        \item for its approximate generation error (\Cref{eq:errorMissFinite}) it holds that, for every valid distribution $\cP$ with respect to $\cL$ there exist $c,C > 0$ such that $E_{X_1,\ldots,X_n\sim \cP^n}\insquare{\mathrm{er}\inparen{\generator_n\inparen{X_1,\ldots,X_n}}} \leq C\cdot R(c\cdot n).$
    \end{itemize}
    Then, there is a randomized generating algorithm $\inparen{\generator'_n\colon \cX^n \overset{r}{\rightarrow} \mathfrak{G}}_{n \in \N}$
    for which, for any target language $K \in \cL$ and every enumeration $\sigma$ of $K$, it holds that
    \begin{itemize}
        \item $\inparen{\generator'_n}_{n \in \N}$ is a stable generator (\Cref{def:stable-generators}), and
        \item 
        \[
            \Pr\insquare{\exists n^* \in \N\colon \forall n \geq n^* \text{ it holds that } \mathrm{er}\inparen{\generator'_n\inparen{\sigma_1,\ldots,\sigma_n}} = 0} = 1\,,
        \]
        where the probability is with respect to the randomness of the algorithm.
    \end{itemize}
\end{lemma}
The proof of \Cref{lem:missing-finite-elements-statistical-to-online}
is identical to the proof of \Cref{lem:unambiguous-statistical-to-online},
since the only property of the error function that is needed
is that once the algorithm has stabilized then its error also stabilizes.
For completeness, we give the details below.

\begin{proof}[Proof of \Cref{lem:missing-finite-elements-statistical-to-online}]
    Let $K \in \cL$ be any target language and $\sigma$ be any enumeration of $K.$ Let $\cP_\sigma$ be the distribution defined in \Cref{def:distribution-induced-sequence}. We know that, by definition,
    $\cP_\sigma$ is valid with respect to $\cL,$ since it is supported on $K.$ Let $\inparen{\generator'_n}_{n \in \N}$ be a generator
    which, for every $n \in \N,$ runs $\generator_n$ on $\cP_\sigma.$ 
    In order to draw samples from $\cP_\sigma$ the generator $\generator'_n$ uses its internal randomness and the process
    described in \Cref{prop:angluin-sampling}. Since $\cP_\sigma$ is
    a valid distribution with respect to $\cL,$ \Cref{lem:missing-finite-elements-statistical-to-online} gives us that
    \[
        \Pr_{\inbrace{X_i}_{i \in \N} \sim \cP_\sigma^\infty}\insquare{\exists n^* \in \N\colon \forall n \geq n^* \text{ it holds that } \inbrace{\mathrm{er}\inparen{\generator_n\inparen{X_1,\ldots,X_n}} = 0}} = 1\,.
    \]
    Hence, this implies that
    \[
        \Pr\insquare{\exists n^* \in \N\colon \forall n \geq n^* \text{ it holds that } \inbrace{\mathrm{er}\inparen{\generator'_n\inparen{\sigma_1,\ldots,\sigma_n}} = 0}} = 1\,,
    \]
    where the probability is taken with respect to the internal randomness
    of the algorithm.
    Moreover, since $(\generator_n)_{n \in \N}$ is a stable generator
    it also holds that $(\generator'_n)_{n \in \N}$ is a stable generator.
    This concludes the proof.
\end{proof}
The proof of \Cref{thm:approxBreadth:statistical} follows as a corollary
of the previous result (\Cref{lem:missing-finite-elements-statistical-to-online}) and \Cref{thm:online:impossibility:robustOneSided}.

\begin{proof}[Proof of \Cref{thm:approxBreadth:statistical}]
        Let $\cL$ be a countable collection of languages. 
    Assume that such a stable generating algorithm exists. Then, using
    the construction from \Cref{lem:missing-finite-elements-statistical-to-online}
    to get a stable generator that generates missing only finitely many elements in the limit, 
    for every target language $K \in \cL$ and every enumeration $\sigma$
    of $K,$ with probability 1. This contradicts the impossibility
    result from \Cref{thm:online:impossibility:robustOneSided}.
\end{proof}

\section{Proofs from \texorpdfstring{\cref{sec:results:identification:additional}}{Section 3.4} (Further Results for Identification)}
\label{sec:addResults}

    \subsection{Proof of \texorpdfstring{\cref{prop:exp-rates-id-subset oracle}}{Proposition 3.10} (Identification Using Subset Oracle)}\label{sec:proof-exp-rates-id-subset}

    We first give a sufficient
    condition on the algorithm that identifies
    in the limit that allows one to directly
    use it in the statistical setting
    and get exponential rates.

    \begin{lemma}\label{lem:suff-cond-use-online-id-in-stat}
    Let $\cL$ be a countable collection of languages. Let $\cA = \{h_n\}_{n \in \N}$ be an algorithm
    that identifies $\cL$ in the limit with positive examples
    with the following additional property:
    \begin{itemize}
        \item for every target language
    $K \in \cL$ there exists a finite set of examples 
    $\{x_{i_1},\ldots,x_{i_{\ell}}\} \subseteq K$ that depends 
    only on $K,$ 
    and the enumeration of $\cL,\cX,$
    \item  and a finite number $n_0 \in \N$ that
    depends on $K,$ the enumeration of $\cL,\cX,$
    \end{itemize}
    such that $\cA$ always identifies correctly 
    if its input has size at least $n_0$
    and it contains
    $x_{i_1},\ldots,x_{i_\ell}.$ Then,
    $\cA$ identifies $K$ with exponential rates
    in the statistical setting.
\end{lemma}

\begin{proof}
    Let $\cP$ be a valid data-generating
    distribution. Then, by definition,
    $\supp(\cP) = K,$ for some $K \in \cL.$
    Let $x_{i_1},\ldots,x_{i_\ell} \subseteq K$
    be a set of points such that after 
    $\cA$ takes as input this set
    it starts identifying correctly, \ie{},
    for any $S$ such that $x_{i_1},\ldots,x_{i_\ell} \subseteq S$ and $\abs{S} \geq n_0$
    it holds that $L_{h_{\abs{S}}(S)} = K.$
    Since $\cP$ is a valid data-generating distribution
    it holds that $x_{i_1},\ldots,x_{i_\ell} \subseteq \supp(\cP).$ Let $p_{i_i}, \ldots, p_{i_{\ell}}$
    be the mass of points $x_{i_1},\ldots,x_{i_\ell}$
    under $\cP.$ Suppose we draw $n$ samples \iid{} from
    $\cP.$ Then, the probability that we do not
    observe all $x_{i_1},\ldots,x_{i_\ell}$ in
    the sample is bounded as
    \begin{align*}
        \Pr_{X_1,\ldots,X_n \sim \cP^n}[\exists j\in[\ell]\colon  x_{i_j} \notin \{X_1,\ldots,X_n\} ] &\leq \sum_{j\in  [\ell]} \Pr_{X_1,\ldots,X_n \sim \cP^n}[x_{i_j} \notin \{X_1,\ldots,X_n\} ] \tag{by a union bound}\\
        &= \sum_{j\in  [\ell]} \inparen{1-p_{i_j}}^n & \tag{since we have \iid{} draws}\\
        &\leq \sum_{j\in  [\ell]}  e^{-p_{i_j}\cdot n} \tag{using that $1-z\leq e^{-z}$ for all $z\in \R$}\\
        &\leq \ell \cdot e^{-\min_{j\in[\ell]} p_{i_j}\cdot n} \,. 
    \end{align*}
    Thus, the algorithm identifies correctly in the statistical setting
    after taking as input $n\geq n_0 \in \N$
    examples, 
    with probability at least 
    $1- C\cdot e^{-c\cdot n},$ for some distribution
    dependent constants $C, c.$ This concludes the proof.
\end{proof}
We are now ready to prove \Cref{prop:exp-rates-id-subset oracle}.

\begin{proof}[Proof of \Cref{prop:exp-rates-id-subset oracle}]
    Let $\cL$ be a collection that is identifiable in the limit
    and assume access to a subset oracle. From \Cref{sec:identification:subsetOracles} we know
    that the algorithm of \citecustom{kleinberg2024language} given
    access to a subset oracle identifies
    any identifiable collection $\cL$ in the limit . Moreover,
    by \Cref{lem:km24-satisfies-condition-exp-rates-subsetOracle} we 
    get that this algorithm satisfies the condition of
    \Cref{lem:suff-cond-use-online-id-in-stat}, thus this result
    immediately gives us that the algorithm of \citecustom{kleinberg2024language} obtains exponential
    rates for identification in the statistical setting (assuming
    access to a subset oracle).
\end{proof}

    \subsection{Proof of \texorpdfstring{\cref{prop:exp-rates-id-finite-collections}}{Proposition 3.11} (Identification of Finite Collections)}\label{sec:proof-exp-rates-id-finite}

    Similar to the previous section, we will show that there
    exists an algorithm that satisfies \Cref{lem:suff-cond-use-online-id-in-stat}, \ie{}, for any finite collection of (potentially infinite) languages, it
    identifies with exponential rates.
     Recall the domain
    $\cX$ has an enumeration $\cX = \inbrace{x_1,\ldots,}.$ Consider the following algorithm for identification.

            \begin{mdframed}
            \label[algorithm]{alg:identification-finite-collection}
            \textbf{Algorithm~\cref{alg:identification-finite-collection}}~~-~~ 
            {\rm \textbf{Identifying a finite collection $\cL = \inbrace{L_1,\ldots,L_k}$ in the limit}}\\[2mm]
            \noindent \textbf{Description:}
                \begin{enumerate}
                    \item \textbf{for each} $t\in \N$ \textbf{do:}
                    \begin{enumerate}[itemsep=0pt]
                        \item Let $S_t = \inbrace{x_{i_1},\ldots,x_{i_t}},$ where
                        $x_{i_\ell}$ is the element the algorithm sees in round $\ell$
                        \item Construct a version space $V_t$ containing all languages $L\supseteq S_t$%
                        \item Let $V'_t = \inbrace{L \in V_t\colon \forall j \in [t], \forall L' \in V_t \text{ it holds that } x_{j} \in L \implies x_j \in L'} $
                        \item \textbf{if}~~$V'_t \neq \emptyset$~~\textbf{then:}\quad  output the smallest index $j$ such that $L_j$ is in $V'_t$
                        \item \textbf{else:}~~output an arbitrary index $j$
                    \end{enumerate}
                \end{enumerate}
        \end{mdframed}

    \begin{proof}[Proof of \Cref{prop:exp-rates-id-finite-collections}]
        Let us first show that the previous algorithm identifies in the limit.
        Let $k \coloneqq \abs{\cL}$ denote the size of $\cL.$
        Consider any target language $K \in \cL.$ 
        Notice that $\cL$ can be partitioned into three sets: the languages $L \in \cL$ such that $L = K,$ the languages
        $L \in \cL$ such that $K \subsetneq L$ and the
        languages $L \in \cL$ such that $K \not\subseteq L.$
        Then, for every language $L_j \in \cL$ such that $K \not\subseteq L_j$ there exists some $x \in K$
        such that $x \notin L_j.$ Let $i_j \in \N$ be the smallest number for
        which $x_{i_j} \in K, x_{i_j} \notin L_j.$ Let $\cL' \subseteq \cL$ be the set of all such 
        languages and $\cX'$ the set of all such smallest indexed $x \in \cX.$ 
        Notice that, since we consider a fixed enumeration of $\cX$ throughout, the
        set $\cX'$ depends only on the target language $K$ and the enumerations of $\cL, \cX.$
        Since the collection $\cL$ is finite we have that $\abs{\cX'} < k < \infty.$

        Now consider any language $L \in \cL$ such that $K \subsetneq L.$ Let $\cL'' \subseteq \cL$ be the set of all such languages. Then, for every $L_j \in \cL''$
        there is some $x \in L_j$ such that $x \notin K.$ Let ${i'_j}$ be the smallest 
        such index. Define
        \[
            n_0 \coloneqq \max_{j \in \N} \inbrace{i'_j\colon L_j \in \cL''} \,,
        \]
        \ie{}, the largest index among these elements. Since the collection $\cL$ is finite it holds that $n_0 < \infty$. 
        Consider any execution of
        the algorithm in any round $t \in \N$ for which the input sample $S$ satisfies \textbf{i)} $\cX' \subseteq S,$
        and \textbf{ii)} $\abs{S} \geq n_0.$ The first condition on $S$ implies that
        for this round $V_t = \inbrace{L \in \cL\colon  K \subseteq L}.$ Moreover, the second
        condition implies that $V'_t = \{L \in \cL\colon  L = K\}.$ Thus, the smallest
         $j \in \N$ such that $L_j \in V'_t$ is the smallest index $z \in \N$ such that
         $L_z = K.$ Thus, there is some large enough $t^*$ so that the algorithm
         outputs the index $z$ for all $t \geq t^*.$ Moreover, the set $\cX'$ and the
         number $n_0$ satisfy the conditions of  \Cref{lem:suff-cond-use-online-id-in-stat},
         thus the algorithm identifies with exact exponential rates in the statistical setting.
         This concludes the proof.
    \end{proof}

\subsection{Proof of \texorpdfstring{\Cref{prop:exp-rates-id-finite-languages}}{Proposition 3.12} (Identification of Collections of Finite Languages)}\label{sec:proof-exp-rates-id-finite-lang}

In this section, we give the proof of \Cref{prop:exp-rates-id-finite-languages}.

\begin{proof}[Proof of \Cref{prop:exp-rates-id-finite-languages}]
    Recall Gold's algorithm \citep{gold1967language} that identifies in the limit 
    for such collections $\cL$: at any step $n \in \N$ let $S_n$ be the set of 
    elements the adversary has presented so far. Output $\min\inbrace{j \in \N\colon  S_n \subseteq L_j}.$ Consider any valid distribution $\cP$ with respect to $\cL.$ Then, $\cP$ is supported on some language $K \in \cL$. Let $\{x_{i_1},\ldots,x_{i_k}\} \coloneqq K$ and $p_{i_j}$ be the mass of element $x_{i_j}, j \in [k].$ 
    For every $n \in \N$, let $\cE_n$ be the event that the $n$ \iid{} draws 
    from $\cP$ contain the set  $\{x_{i_1},\ldots,x_{i_k}\}$, \ie{},
    \[
        \{x_{i_1},\ldots,x_{i_k}\} = \{X_1,\ldots,X_n\} \,.
    \]
    Notice that under $\cE_n$, the algorithm identifies
    correctly. Then, for the complement of this event, we have that
 \begin{align*}
        \Pr_{X_1,\ldots,X_n \sim \cP^n}[\exists j\in[k]\colon  x_{i_j} \notin \{X_1,\ldots,X_n\} ] &\leq \sum_{j\in  [\ell]} \Pr_{X_1,\ldots,X_n \sim \cP^n}[x_{i_j} \notin \{X_1,\ldots,X_n\} ] \tag{by a union bound}\\
        &= \sum_{j\in  [k]} \inparen{1-p_{i_j}}^n & \tag{since we have \iid{} draws}\\
        &\leq \sum_{j\in  [k]}  e^{-p_{i_j}\cdot n} \tag{using that $1-z\leq e^{-z}$ for all $z\in \R$}\\
        &\leq k\cdot e^{-\min_{j\in[k]} p_{i_j}\cdot n} \,.  
    \end{align*}
    This concludes the proof. 
\end{proof}

    \subsection{Proof of \texorpdfstring{\cref{thm:identification-positive-negative}}{Theorem 3.13} (Identification from Positive and Negative Examples)}\label{sec:stat-rates-identification-pos-neg}
We now move on to the task of language identification with both positive and negative examples. The main difference between this
setting and binary classification is that the objective function
is different. In particular, in our setting, the learner is
required to \emph{identify} the target language, whereas in
the classification setting the learner is required to
output a function that labels most of the elements of the domain
according to some target labeling function. Thus, it is clear
that the identification task is more challenging than the classification task.
Sticking to the notation we used before, we have a countable
set of languages $\cL = \{L_1,L_2,\ldots\}$, where each
$L \in \cL$ is also countable and $\cup_{L \in \cL} L \subseteq \cX,$
for some countable domain $\cX.$ 
Recall the notion of valid distribution in this setting \citep{angluin1988identifying}: a distribution $\cP$ is valid
with respect to $\cL$ 
if and only if $\supp(\cP) \subseteq \cX \times \{0,1\}$ and there exists
some $K \in \cL$ such that for all $x \in K$ we have $\cP[(x,1)] > 0, \cP[(x,0)] = 0$ and for all $x \notin K$ we have $\cP[(x,0)] > 0, \cP[(x,1)] = 0$ (see \Cref{def:valid-language-distribution-pos-neg}).

Next, recall that for any $n \in \N$ and and set of labeled examples $S_n = (x_1, y_1),\ldots,(x_n,y_n) \in (\cX \times \{0,1\})^n$ 
 the \emph{error} of the learner $\{h_n\colon (\cX \times\{0,1\})^n \rightarrow \N\}_{n \in \N}$ for this task is 
\begin{equation}
    \mathrm{er}(h_n(S_n)) = \mathbb{1}\inbrace{L_{h_n(S_n)} \neq K} \,.
\end{equation}
Notice that, under this definition, $\E[\er(h_n)] = \Pr\left[L_{h_n} \neq K\right],$
\ie{}, the probability that $h_n$ fails to identify the correct 
language after it sees $n$ examples from the data-generating
distribution.

Our proof proceeds in two parts. First, we show that for all
countable collections of languages that are non-trivial for identification (\Cref{def:non-trivial-collections})
 exponential rate is the best possible for identification
with positive and negative examples. The approach is essentially
identical to \Cref{lem:exp-rates-lower-bound-ident-pos} and 
the lower bound from \citet{bousquet2021theory}.
Then, we show that all countable collections of languages
are learnable at exponential rates with positive and negative examples.

The formal statement regarding the exponential rates lower bound follows.

\begin{lemma}\label{lem:exp-rates-lower-bound-ident-pos-neg}
    Let $\cL$ be a non-trivial collection of languages. Then, 
    for any learning algorithm $\cA = \{h_n\}_{n \in \N}$ there
    exists a valid distribution $\cP$ such that 
    $\E[\er(h_n)] \geq C \cdot e^{-c \cdot n},$ for 
    infinitely many $n \in \N.$
\end{lemma}

\begin{proof}
    Since $\cL$ is non-trivial, there exist $L, L' \in \cL$
    and $x \in \cX$ such that $L \neq L'$ and $x \in L, x \in L'.$ Let $\cP_L, \cP_{L'}$ be valid distributions for $L, L'$
    that place at least $\nicefrac{1}{2}$ mass on $(x,1)$ and they spread the remaining mass arbitrarily as follows: half of the remaining mass of $\cP_L$
    (respectively $\cP_{L'}$) 
    is spread arbitrarily on all the elements of $L$ (respectively $L'$) with label 1, and the other half on all the elements of $K \setminus L$ (respectively $K \setminus L'$)
    with label 0.
    Notice that 
    since $L \neq L'$ at least one of them has at least 
    one more element other than $x.$
    For any $n \in \N,$ under both distributions, 
    with probability at least ${2^{-n}},$ the 
    algorithm will only see the element $(x,1)$ appearing in
    the sample. Let $\cE_n$ be that event
    and condition on it.
    Notice that
    \[
        \Pr\left[L_{h_n((x,1),\ldots,(x,1))} = L\mid\cE_n\right] +  \Pr\left[L_{h_n(x,\ldots,x)} = L'\mid \cE_n\right] \leq 1 \,,
    \]
    where the probability is with respect to the randomness
    of the learning algorithm.
    Thus, we have 
    that $\Pr\left[L_{h_n\inparen{(x,1),\ldots,(x,1)}} \neq L\mid \cE_n\right] \geq \nicefrac{1}{2}$
    or $\Pr\left[L_{h_n\inparen{(x,1),\ldots,(x,1)}} \neq L'\mid \cE_n\right] \geq \nicefrac{1}{2}.$
    By the pigeonhole principle, for at least one of $L, L',$
    the previous inequality holds for infinitely many $n \in \N.$
    Assume without loss of generality it holds for $L$ and
    denote by $\hat{N}$ the set of $n \in \N$ for which it holds.
    Then, for all $n \in \hat{N}$ we have that
    \begin{align*}
               \E_{(X_1,Y_1),\ldots,(X_n,Y_n) \sim \cP^n_{L}}[\er\inparen{h_n\inparen{(X_1,Y_1),\ldots,(X_n,Y_n)}}] &= \Pr_{(X_1,Y_1),\ldots,(X_n,Y_n)  \sim \cP^n_{L}}\left[L_{h_n\inparen{(X_1,Y_1),\ldots,(X_n,Y_n)}} \neq L\right] \\
               &\geq \Pr_{(X_1,Y_1),\ldots,(X_n,Y_n)  \sim \cP^n_{L}}\left[L_{h_n\inparen{(X_1,Y_1),\ldots,(X_n,Y_n)}} \neq L \mid \cE_n\right] \cdot \\
               &\Pr_{(X_1,Y_1),\ldots,(X_n,Y_n)  \sim \cP^n_{L}}[\cE_n]\\
               &\geq \frac{1}{2^n} \cdot \Pr\left[L_{h_n\inparen{(x,1),\ldots,(x,1)}} \neq L \mid \cE_n\right] \tag{by the definition of $\cE_n$} \\
               &\geq \frac{1}{2^{n+1}} \,, \tag{due to the assumption on $L$}
    \end{align*}
    which concludes the proof.
 
\end{proof}   
We now move on to establishing the upper bound. Following the approach
from the setting with only positive examples, we show
that an infinite draw of \iid{} samples from any valid distribution
is a ``complete presentation'' of the target language $K,$ \ie{}, all the elements
of $K$ appear with label 1 and all elements outside of $K$ appear with 
label 0. This follows immediately from \Cref{prop:support-appears-in-countable-samples}.

The next step towards proving \Cref{thm:identification-positive-negative} is to show if $\cA$ is an algorithm that identifies
the target language in the limit in the adversarial (online)
setting of \citecustom{gold1967language} with positive and negative examples, then $\cA$ is a consistent algorithm in the statistical setting. This implies that for
any valid distribution $\cP$, there is some
number $t^*  \coloneqq  t^*(\cA, \cP) \in \N$ such that, when we draw 
$t^*$ many \iid{} samples from $\cP$, it will identify
the target language with probability at least $\nfrac{6}{7}.$
We denote the time of the last mistake of algorithm
$\cA = \{h_n\}_{n \in \N}$ on a labeled sequence of examples $(x_1,y_1),(x_2,y_2),\ldots$ by $T_\cA(x_1,y_1, x_2,y_2,\ldots),$ \ie{},
\[
    T_\cA(x_1,y_1,x_2,y_2,\ldots) = \inf\inbrace{n_0 \in \N\colon  L_{h_n(x_1,y_1,\ldots,x_n,y_n)} = K, \forall n > n_0} \,.
\]
The next result formalizes the claim. Its proof is almost identical to
\Cref{prop:quantile-bound-identification}, but we present
it for completeness.
\begin{proposition}\label{prop:quantile-bound-identification-pos-neg}
    {Fix any family of languages $\cL$ over a countable domain.}
    For any algorithm $\cA = \{h_n\}_{n \in \N}$ that  {identifies $\cL$ in
    the limit} with positive and negative examples in the online setting and any valid distribution
    $\cP$ (\cref{def:valid-language-distribution-pos-neg}), there exists a number $t^*$ such that
    \[
        \Pr_{\{(X_i, Y_i)\}_{i \in \N}\sim \cP^{^\infty}}
        [T_\cA(X_1,Y_1,X_2,Y_2,\ldots) \leq t^*] \geq \frac{6}{7}\,.
    \]
\end{proposition}
\begin{proof}
    Let $(X_1,Y_1),(X_2,Y_2),\ldots,$ be a countable \iid{} sample from $\cP.$
    From \cref{prop:support-appears-in-countable-samples} we
    get that this sample is a valid input to $\cA$ 
    since, with probability one, 
    all elements of $K$ appear with label 1 and all elements of $\cX \setminus K$ 
    appear with label 0. 
    Consider the execution of $\cA$ on prefixes of the sequence
    and denote by $T_{\cA}  \coloneqq  T_{\cA}(X_1,Y_1,X_2,Y_2,\ldots)$ the time it made its last mistake. 
    We have that $\Pr_{\{(X_i,Y_i)\}_{i \in \N}\sim \cP^{\infty}}[T_\cA \in \N] = 1.$
    Thus, 
    \[
        \lim_{t \rightarrow \infty}\Pr_{\{(X_i,Y_i)\}_{i \in \N}\sim \cP^{\infty}}[T_\cA(X_1,Y_1,X_2,Y_2\ldots) \geq t] = 0\,.
    \]
    Thus, as required, there exists some $t^* \in \N$ such that
    \[
         \Pr_{\{(X_i,Y_i)\}_{i \in \N}\sim \cP^{\infty}}[T_\cA(X_1,Y_1,X_2,Y_2,\ldots) \geq t^*] \leq \frac{1}{7}\,.
    \]
\end{proof}
The problem is that this time $t^*$ depends on the algorithm $\cA$ and the unknown distribution $\cP.$ Suppose we knew a number $t^*$ so that for all $t \geq t^*$
\[
\E_{S \sim \cP^t}[\mathrm{er}(h_{t}(S)) > 0] = \Pr_{S \sim \cP^t}\left[L_{h_{t}(S)} \neq K\right] < \frac{1}{4}\,.
\]
Then, we
could design an identification algorithm $\{h_n\}_{n \in \N}$ with exponential rates as follows. First, we break up the data into batches, each of length $t^*$. Second, we run the identification algorithm from the online setting separately for each batch.
Finally, we aggregate these algorithms by taking a majority vote.
Now, by the definition of $t^*$ and
Hoeffding’s inequality, the probability that more than one-third of the classifiers have not
identified the language is exponentially small.

There are two issues with this approach: 
\begin{itemize}
    \item In our language setting, it can be the case that two identification algorithms are correct but output different indices for the true language. This was also an issue for the positive examples case and we can resolve it in a similar way using \Cref{lem:post-processing-same-index}.
    \item The second issue is that $t^*$ depends on the distribution $\cP.$ In the previous case of positive examples, the solution was to \emph{guess} the number $t^*$ using some very slowly increasing function $g(n)$. This was the reason why we did not manage to get exponential rates. Interestingly, in the setting where we have access to positive and negative examples, we can \emph{estimate} $t^*$ from samples, as we see below.
\end{itemize}
The main lemma behind this result is the following and corresponds to an adaptation of Lemma 4.4 from \citet{bousquet2021theory} with a different loss function.
This lemma will allow us to get exactly exponential rates, compared to the rates obtained in the positive examples regime.
\begin{lemma}
[Estimation of Stopping Time]
Let $S_n$ be $n$ \iid{} examples observed from $\cP$ which is valid with respect to some language $K \in \cL.$
There exists universally measurable $t_n = t_n(S_n)$, whose definition does not depend on $\cP$, so that the following holds. Given $t^*$
such that
\[
\Pr_{S_{t^*}}\left[L_{h_{t^*}(S_{t^*})} \neq K\right] < \frac{1}{8}\,,
\]
there exist 
$C, c > 0$ independent of $n$ (but depending on $\cP, t^*$) so that
\[
\Pr[t_n \in  \cT] \geq 1 - C e^{-cn}
\]
where
\[
\cT = \inbrace{
        1 \leq t \leq t^* \colon  \Pr_{S_{t}}\left[L_{h_{t}(S_{t})} \neq K\right] < \frac{3}{8}
    }\,.
\]
\label{lemma:stopping-time-estimation}
\end{lemma}
Given the above lemma, we are now ready to complete the proof of \Cref{thm:identification-positive-negative}.

\begin{proof}[Proof of \Cref{thm:identification-positive-negative}]
First, the exponential rate lower bound follows
immediately from \Cref{lem:exp-rates-lower-bound-ident-pos-neg}.

    The output of our identification algorithm $h_n$ is the majority of the $h^i_{t_n}$ for $i \in I_n = \{1,2, \dots, \lfloor \nfrac{n}{(2 t_n)} \rfloor\},$
after we apply to them the post-processing computation
described in \Cref{lem:post-processing-same-index} to map
them to the same index of $K.$ Let $z \in \N$ 
be the smallest number for which $L_z = K$. 
It remains to show that
\[
\E_{S_n \sim \cP^n} [\mathrm{er}(h_n)] = \Pr_{S_n \sim \cP^n}\left[L_{h_n(S_n)} \neq K\right] \leq C e^{-cn} \,,
\]
for some constants $C,c > 0$ depending on $\cP.$

Let us consider our predictors $\{h^i_{t_n}\}_{i \in I_n}$ that are obtained by running
the identification algorithm on independent
parts of the dataset of size $t_n.$
Since these predictors might be outputting different indices (descriptions) of the same language, we find the smallest indexed language the output of each classifier can be mapped to. By \Cref{lem:post-processing-same-index}, 
for $n$ sufficiently large,
all the indices from the outputs of the classifiers $h^i_{t_n}, i \in I_n,$ that correspond
to $K$ will be mapped to $z.$  

Let us fix $t \in \cT,$ where $\cT$ is defined in \Cref{lemma:stopping-time-estimation}, and consider the predictors $\{h^i_{t}\}$ for $i \in I = \{1,2,\dots,\lfloor \nfrac{n}{(2t)}\rfloor\}$. 
We also denote by  $\{\hat h^i_{t}\}$ the output of
predictor  $\{h^i_{t}\}$ after applying the post-processing
result from \Cref{lemma:stopping-time-estimation}.
For $n$ sufficiently large, standard concentration bounds give that
\[
\Pr \left[ \frac{1}{\abs{I_n}} \sum_{i \in I_n} 
\ind\{ \hat h^i_{t} \neq z \} > \frac{7}{16}
\right] < e^{- \lfloor n/2 t^* \rfloor / 128}\,.
\]
This means that, except on an event of exponentially small probability, we have that $\hat h^i_t$ outputs index $z$.
Recall that $L_z = K.$

Now we employ our estimation $t_n$ in the above calculation. In particular,
\[
\Pr\insquare{ 
    \hat h^i_{t_n} \neq z  ~~\text{for at least half of $i \in I_n$}
    }
\leq p_1 + p_2\,,
\]
where
\[
p_1 \coloneqq \Pr[t_n \notin \cT] < Ce^{-c n}\,,
\]
and 
\[
p_2 \coloneqq \Pr\insquare{\exists t \in \cT \colon h^i_{t}~~\text{does not predict $K$ for at least half indices}} \leq t^* e^{- \lfloor n/2 t^* \rfloor / 128}\,.
\]
This gives the desired result.
\end{proof}

\noindent We conclude this section with the proof of \Cref{lemma:stopping-time-estimation}.

\begin{proof}[Proof of \Cref{lemma:stopping-time-estimation}]
We split the training set $S_n$ into two sets and we further split the sets into batches. The idea is to use the first set to train multiple independent instances of the online learning algorithm and the second set to estimate its identification error.

More concretely, let $\cA$ the algorithm that identifies in the limit. For each batch size $1 \leq t \leq \lfloor \nfrac{n}{2} \rfloor$ and batch index $1 \leq i \leq \lfloor \nfrac{n}{(2t)} \rfloor$, we let
\[
I_t^i = h_t\inparen{X_{(i-1)t+1}, Y_{(i-1)t+1}, X_{it}, Y_{it}}
\]
be the index of the output of the learning algorithm that is trained on batch $i$ of a subset of the dataset that has size $t$. This is a mapping from the training samples to indices of the predicted language.

For every fixed $t$, the outputs $\left\{I_t^i\right\}_{i \leq \lfloor n/2t \rfloor}$ are trained on different parts of the first half of the training set and they are independent of the whole second half of the training set. This means that we can view every $\left\{I_t^i\right\}_{i \leq \lfloor n/2t \rfloor}$ as an independent draw of the distribution of $h_t(\cdot)$. To estimate the identification error of $h_t(\cdot)$, we will make use of the second half of the training set. We define 
\[
\hat{e}_t = \frac{1}{\lfloor n/2t \rfloor} \sum_{i=1}^{\lfloor n/2t \rfloor}  \ind\inbrace{\ind\inbrace{X_s \in L_{I_t^i}} \neq Y_s \text{ for some } \frac{n}{2} \leq s \leq n }\,.
\]
We underline that this can be computed using just
membership calls to the predicted languages.
Now observe that, almost surely,
\[
\hat{e}_t \leq e_t = \frac{1}{\lfloor n/2t \rfloor} \sum_{i=1}^{\lfloor n/2t \rfloor} \ind\left\{ \Pr_{(X,Y) \sim \cP}\insquare{\ind\inbrace{X \in L_{I_t^i}} \neq Y} > 0 \right\} =  \frac{1}{\lfloor n/2t \rfloor} \sum_{i=1}^{\lfloor n/2t \rfloor} \ind \inbrace{L_{I_t^i} \neq K}\,.
\]
We define $\wh{t}_n = \inf\{t \leq \lfloor \nfrac{n}{2} \rfloor\colon \wh{e}_t < \nfrac{1}{4}\}$, where we assume that $\inf \emptyset = \infty$.

We now want to bound the probability that $\wh{t}_n > t^{\star}$. Using Hoeffding's inequality we get that
\begin{align*}
\Pr\left[\wh{t}_n > t^{\star}\right]
\leq \Pr\left[\wh{e}_{t^{\star}} \geq \frac{1}{4}\right] \leq \Pr\left[e_{t^{\star}} \geq \frac{1}{4}\right]  
& = \Pr\left[e_{t^{\star}} - \frac{1}{8} \geq \frac{1}{8}\right] = \Pr\left[e_{t^{\star}} - \E[e_{t^{\star}}] \geq \frac{1}{8}\right] \leq e^{-\lfloor n/2t^{\star} \rfloor/32}\,.
\end{align*}
This implies that $\wh{t}_n \leq t^{\star}$ except for an event with exponentially small probability.

Moreover,  there is some $\eps > 0$ such that for all $1\leq t \leq t^{\star}$ with 
\[
    \Pr_{S_t \sim \cP^t}\left[\Pr_{(X,Y)\sim\cP}\left[ \ind\inbrace{X \in L_{{h}_t(S_t)}} \neq Y\right] > 0 \right] > \frac{3}{8} \,,
\]
we have that $\Pr_{S_t \sim \cP^t}\left[\Pr_{(X,Y)\sim\cP}\left[ \ind\inbrace{X \in L_{{h}_t(S_t)}} \neq Y\right] > \eps \right] > \nfrac{1}{4} + \nfrac{1}{16}$ (this holds by continuity).
Now fix some $1\leq t \leq t^{\star}$ such that $$\Pr_{S_t \sim \cP^t}\left[\Pr_{(X,Y)\sim\cP}\left[ \ind\inbrace{X \in L_{{h}_t(S_t)}} \neq Y\right] > 0 \right] > \frac{3}{8}$$ (if it exists).
Then, using Hoeffding's inequality again we get that
\[
\Pr\left[\frac{1}{\lfloor n/2t \rfloor} \sum_{i=1}^{\lfloor n/2t \rfloor} \ind \left\{ \Pr_{(X,Y)\sim \cP} \left[\ind\inbrace{X \in L_{I_t^i}} \neq Y\right] > \eps \right\} < \frac{1}{4}\right] \leq e^{\lfloor n/2t^{\star} \rfloor/128}\,.
\]
For any language $L$ such that $\Pr_{(X,Y) \sim \cP}\left[\ind\inbrace{X \in L} \neq Y\right] > \eps$, then
\[
\Pr\left[ \ind\inbrace{X_s \in L} \neq Y_s \text{ for some } \nfrac{n}{2} \leq s \leq n\right] \geq 1 - \left(1-\eps\right)^{n/2}\,.
\]
As we mentioned before, $\{I_t^i\}_{i \leq \lfloor n/2t \rfloor}$ are independent of $(X_s,Y_s)_{s > n/2}$. Thus, applying a union bound we get that the probability that all $I_t^i$ that have $\Pr_{(X,Y)\sim \cP}\left[\ind\inbrace{X \in L_{I_t^i}} \neq Y\right] > \eps$ make at least one error on the second half of the training set is
\begin{align*}
    &\Pr\left[ 
            \ind\left\{\Pr_{(X,Y) \sim \cP}\left[\ind\inbrace{X \in L_{I_i^t}} \neq Y\right] > \eps\right\} 
        \leq 
            \ind \left\{\ind\inbrace{X_s \in L_{I_i^t}} \neq Y_s\ \text{ for some } \nfrac{n}{2} < s \leq n\right\} \text{ for all } i \in \insquare{\lfloor n/2t \rfloor}
    \right]\\
    &\qquad\geq 1 - \left\lfloor\frac{n}{2t}\right\rfloor\left(1-\eps\right)^{n/2}\,.
\end{align*} 
Thus, we get that
\[
    \Pr[\wh{t}_n = t] \leq \Pr\left[\wh{e}_t < \frac{1}{4}\right] \leq \left\lfloor \frac{n}{2}\right\rfloor\left(1-\eps \right)^{n/2} + e^{-\lfloor\frac{n}{2t^{\star}}\rfloor/128}\,.
\]
Using the previous estimates and applying a union bound, we get that
\[
    \Pr[\wh{t}_n \notin \cT] \leq e^{-\lfloor n/2t^{\star} \rfloor/32} + t^{\star}\left\lfloor\frac{n}{2}\right\rfloor \left(1-\eps \right)^{n/2} + t^{\star}e^{-\lfloor n/2t^{\star}\rfloor/128} \leq Ce^{-cn}\,,
\]
for some constants $C,c > 0$. Note that $C = C(\cP, t^*)$ and $c = c(\cP,t^*)$.
\end{proof}

\section*{Acknowledgments}
    We thank Ahmad Beirami and Manolis Zampetakis for helpful discussions and references after the original draft.
    {We thank Dylan McKay for discussions and references about the theory of computation during the preparation of this paper, and specifically for informing us about \cref{lem:tm:consecutiveReads,lem:tm:finiteBitsRead}.}
    We thank Kyriakos Lotidis for useful discussions about the lower bound of \citecustom{angluin1988identifying}.
    We also thank Yuan Deng, Sid Mitra, Ansong Ni, Argyris Oikonomou, Xizhi Tan, and Manolis Zampetakis for their feedback on a draft of this paper.
    Alkis Kalavasis was supported by the Institute for Foundations of Data Science at Yale.
    Grigoris Velegkas was supported by the AI Institute for Learning-Enabled Optimization at Scale (TILOS).

\newpage
\printbibliography
\appendix  
\newpage

\section{{Undecidability} of \texorpdfstring{$\mop(\cdot)$}{MOP}}\label{sec:supportOracle}\label{sec:undeciable}
    In this section, we present a generating algorithm $\generator=(\generator_n)$ for which MOP is undecidable.
    The corresponding generating algorithm is static in the sense that $\generator_1=\generator_2=\dots$.
    For each $t$, $\generator_t$ is the following randomized Turing machine.
    
        \begin{mdframed}
            {\rm \textbf{Input:} None}
            \smallskip

            \noindent {\rm \textbf{Setup:} The internal random tape contains symbols from $\inbrace{0,1,2}$ }
            \smallskip
            
            \noindent \textbf{Description:}
            \begin{enumerate}[itemsep=1pt]
                \item Read bits $r=r_1r_2\dots$ from the random tape until the first 2 is found on the tape 
                \item Let $b=1$ if $r$ is of the form $\inangle{M}w$ where $\inangle{M}$ is a valid representation of a Turing machine and $w$ is any (possibly) empty string
                \item \textbf{If}~~$r=1$~~\textbf{then:}~~
                    Execute $M$ on input $w$ and \textbf{return} $\inangle{M}w1$ if $M$ halts
                \item \textbf{Else:}~~ \textbf{return} 0
            \end{enumerate}
        \end{mdframed}
        
        \bigskip
        
        \begin{proposition}
            $\mop{}(\cdot)$ is undecidable for the above Turing machine.
        \end{proposition}
        \begin{proof}
            The proof is a simple reduction to the halting problem, which is well-known to be undecidable.
            To see this, observe that to decide whether $M$ halts on input $w$, it suffices to check if $\inangle{M}w1$ is in the support of the above machine. %
        \end{proof}

\section{Results with Subset Oracles}
    In this section, we design algorithms that have access to a subset oracle that, given indices $i$ and $j$, answers whether ``$L_i\subseteq L_j$?''
    We give two algorithms (1) an algorithm that identifies in the limit without requiring tell-tale oracles and (2) a best-of-both-words algorithm that generates consistently and achieves breadth whenever possible.
    
    \subsection{Identification in the Limit without Tell-Tale Oracle via Subset Oracles}\label{sec:identification:subsetOracles}

    \citecustom{angluin1980inductive} showed that a collection of (recursive) languages $\cL$ is identifiable in the limit if and only if each $L_i\in \cL$ has a finite ``tell-tale'' set $T_i$ that, roughly speaking, enables one to eliminate $L_i$ if it is not the target.
    If one has access to an oracle that, given an index $i$, outputs the tell-tale $T_i$, then one can identify $\cL$ using an algorithm by \citecustom{angluin1980inductive}.
    Our next result shows that, if one has access to queries of the form ``$L_i\subseteq L_j$?'', then one can identify any identifiable language collection without access to a tell-tale oracle.

    \begin{theorem}\label{thm:identificationWithoutTellTale}
        Let $\algo{S}$ be an oracle that, given indices $i$ and $j$, outputs $\textsf{Yes}$ if $L_i\subseteq L_j$ and outputs \textsf{No} otherwise. 
        Fix any identifiable collection of languages $\cL=\inbrace{L_1,L_2,\dots}$.
        There is an algorithm $\algo{A}$ that given, 
            an enumeration of a target language $K\in \cL$ (for any $K$) and access to $\algo{S}$,
        identifies $K$ in the limit.
        
        Importantly, $\algo{A}$ does not need to be provided the tell-tale of $K$ or any other language in $\cL$.
    \end{theorem}
    Interestingly, the algorithm in \cref{thm:identificationWithoutTellTale} is the same as an algorithm proposed by \citecustom{kleinberg2024language}: the algorithm defines a certain notion of critical languages and selects the last critical language, say $L_j$.
    Naturally, since we want to identify the language, instead of outputting an element of $L_j$ the algorithm will guess the index.
    {This algorithm identifies $K$ in the sense that after some finite time $t^*$, it outputs an index $z$ such that $L_z = K$. 
    The specific index $z$ outputted, however, may change infinitely often.
    This can also be avoided by using the post-processing routine in \Cref{lem:post-processing-same-index}.}
    \begin{proof}
        The algorithm to identify $K$ is simple:%
        \begin{mdframed}
            
            \noindent \textbf{For} $t\in \inbrace{1,2,\dots}$ \textbf{do:}
            \begin{enumerate}
                \item Observe element $x_t$ and let $S_t$ be the set of all elements observed so far
                \item Construct a \textit{version space} $V_t$ consisting of all languages in $L_{\leq t}$ consistent with $S_t$, \ie{}, 
                \[
                    V_t\colon \inbrace{L_j\colon 1\leq j\leq t\,,~~L_j\supseteq S_t}\,.
                \]
                \item[] \textit{$\#$ Define an language $L_i\in V_t$ to be \underline{critical} if $L_i$ is the smallest-index language in $V_t$ or $L_i$ is a subset of all languages preceding it in $V_t$, \ie{}, if $L_i\subseteq L_j$ for all $1\leq j < i$}

                \item {If $V_t = \emptyset$, \textbf{output} a random element of $\cX$ and \textbf{go} to the next iteration.}
                
                \item Construct the set $C_t\subseteq V_t$ of all critical languages
                \item \textbf{output} $i$ where $L_i$ is the largest-indexed language in the set of critical languages $C_t$
            \end{enumerate}
        \end{mdframed} 
        Let $i$ be the first index such that $K=L_i$.
        The above algorithm {identifies $K$} when $K$ is the last element of the set of critical languages $C_t$.
        This condition is implied by the following two conditions.
        \begin{enumerate}
            \item[(A)] $K$ is in the set of critical languages $C_t$.
            \item[(B)] All the languages $L_j$ with $j > i$ that are included in $C_t$ satisfy $L_j = K$.
        \end{enumerate}
        Result (4.3) of \citecustom{kleinberg2024language} shows that there is a finite time $t_a$ after which Condition (A) holds.
        We will show that there is also a finite time $t_b$ after which Condition (B) holds.
        This shows that, {for any $t\geq \max\inbrace{t_a,t_b}$, $\algo{A}$ outputs an index $z(t)$ such $L_{z(t)}=K$}.
        {As mentioned before one can convert this into an algorithm that outputs a fixed index (after some finite time) using the post-processing routine in \cref{lem:post-processing-same-index}.}

        \paragraph{Condition (B) holds after a finite time.}
            Since $\cL$ is identifiable, it must satisfy Angluin's tell-tale criteria (\Cref{def:angliun-criterion}) and, hence, $K=L_i$ has a finite tell-tale set $T_i$.
            (Recall that $T_i$ is not known to us; our proof will not need this.)
            Fix any $j>i$ and any time $t\geq t_a$ (after which $K$ is guaranteed to be a critical language).
            If $L_j$ is a critical language, then by the definition of critical languages, the fact that $K=L_i$ is critical, {and that $j>i$} (P1) $L_j\subseteq K$ and (P2) $L_j\in V_t$.
            Further, if $L_j\subsetneq K$, then $L_j$ cannot contain the tell-tale $T_i$ of $K$ (by the properties of tell-tales; \cref{def:angliun-criterion}).
            Therefore, if $S_t$ (the set of samples seen until step $t$) contains $T_i$, then $L_j$ cannot be in the version space as otherwise it would need to contain $T_i$.
            It follows that, if $S_t\supseteq T_i$ and $t\geq t_a$, then either P2 will be violated or P1 will imply that $L_j = K$.
            Finally, since $T_i$ is finite and, hence, there is a finite time $t_b'$ when all elements of $T_i$ have been observed and, the result follows by letting $t_b\coloneqq \max\inbrace{t_a,t_b'}$.
            
    \end{proof}

    \subsection{Best-Of-Both Worlds: Generating with Breadth When Possible}\label{sec:subsetOracle:bestOfBothworlds}
        Consider the variant of the algorithm in the previous section which instead of outputting the index of the last critical language outputs an unseen sample from the language.
        {This is precisely the algorithm of \citecustom{kleinberg2024language}.\footnote{Note that since we only output an element from the last critical language and not its index, we do not need to perform the post-processing used in the previous section, so this is \citecustom{kleinberg2024language}'s algorithm.}}
        \citecustom{kleinberg2024language} showed that this algorithm consistently generates in the limit.
        An immediate corollary of the identification result in the last section is that this algorithm also achieves breadth for any identifiable collection $\cL$.
        \begin{corollary}
            
            Let $\algo{S}$ be an oracle that, given indices $i$ and $j$, outputs $\textsf{Yes}$ if $L_i\subseteq L_j$ and outputs \textsf{No} otherwise. 
            Fix any collection of countably many languages $\cL=\inbrace{L_1, L_2, \dots}$.

            There is a generating algorithm $\generator$ that given,  an enumeration of a target language $K\in \cL$ (for any $K$) and access to $\algo{S}$,
            consistently generates from $K$ in the limit.
            Moreover, whenever $\cL$ is identifiable, this algorithm generates consistently with breadth in the limit.
            
            Importantly, $\generator$ does not need to be provided the tell-tale of $K$ or any other language in $\cL$.
        \end{corollary}
        Finally, we recall that for any non-identifiable collection $\cL$, generation with breadth is impossible whenever a $\mop{}(\generator)$ can be implemented.

\section{Further Comparison with Online Learning}
\label{sec:comparison}
In this section, we provide some comparisons between the Gold--Angluin (GA) model  \citep{gold1967language,angluin1979finding} and the online game of \citet{bousquet2021theory} (that extends the standard online learning model of \citecustom{littlestone1988learning}), which at first sight share a lot of similarities. However, there are important differences between the two models, which we believe are worth highlighting.

Let us first recall the setting of \citet{bousquet2021theory}, appropriately rephrased
to the context of our work. There is a domain $\cX,$
a collection of languages $\cL \subseteq \inbrace{0,1}^\cX$,
and two players, the learner and the adversary,
play a game over an infinite sequence of discrete
rounds. In every round $t \in \N,$ the adversary
presents an example $x_t \in \cX$ to the learner
and the learner guesses its label $\hat y_t.$
Subsequently, the adversary reveals the true label
$y_t$ to the learner. We say that the learner
makes a mistake if $y_t \neq \hat y_t.$
This can be thought of as the learner
trying to guess whether the example
belongs to the target language.
The adversary has to satisfy the following
constraint. For every round $t \in \N$
there must be some $L_t \in \cL$ such that
$\ind\inbrace{x_\tau \in L_t} = y_\tau, \forall \tau \leq t.$
The goal of the learner is to make finitely many
mispredictions and the goal of
the adversary is to force infinitely many
such mispredictions.

\begin{itemize}
    \item In the GA model, the goal of the learner is to identify the true language in a finite number of steps. In the online game of \citet{bousquet2021theory}, the goal of the learner is to make a finite number of mistakes in its predictions.
    Moreover, in the GA model, the learner observes only positive examples while in the standard online setting, the adversary can provide both positive and negative examples.

    \item The set of languages identifiable in the limit in Gold's model is characterized by Angluin's criterion (\Cref{def:angliun-criterion}). The collections that are online learnable with a finite number of mistakes is characterized by the absence of infinite Littlestone trees \citep{bousquet2021theory}. If there is a uniform bound on the number of mistakes, then this corresponds to the standard finiteness of the Littlestone dimension \cite{littlestone1988learning}.
    
    \item 
        In the GA model, the adversary fixes the true language in advance. In (realizable) online learning, the only thing that matters is \emph{consistency} of the hypothesis class on the given examples, \ie{}, for every sequence of examples, along
        with their labels, there should exist some hypothesis in the hypothesis class that perfectly labels the given sequence (which can change as the online game progresses, see also Section 3 of \citet{bousquet2021theory}). 
        This subtle difference leads to starkly different learning landscapes.
        
    \item 
        In the GA model, the adversary must include all elements of $K$ in the enumeration (the domain is countable). In Littlestone's online setting, there is no such restriction (the feature space can even be uncountable).
         \item Another crucial difference is that the algorithm does not receive feedback about its guess in the GA model. This is in contrast to the standard online setting where, at each round, the learner gets feedback about its prediction. However, it is important to stress that the incentives of the adversaries in Gold's model and in the model of \citet{bousquet2021theory} are different. In the GA model, the goal is not to maximize the number of mistakes that the learner does, but to essentially not allow the learner to identify the true target language. 
   
\end{itemize}
To further show separations between the online setting of \citecustom{gold1967language} and the online game of \citet{bousquet2021theory}, we will need two important definitions coming from the work of \citet{bousquet2021theory}.

\begin{definition}
[Littlestone Tree \citep{bousquet2021theory}]
\label{defn:litt}
A \emph{Littlestone tree} for $\cL$ is a complete binary tree of depth $d\le\infty$ whose internal nodes are labeled by $\cX$, and whose two edges connecting a node to its children are labeled $0$ and~$1$, such that every finite path emanating from the root is consistent with a language $L \in \cL$.

More precisely, a Littlestone tree is a collection 
$$\{x_{\mathbf{u}}:0\le k<d, 
\mathbf{u}\in\{0,1\}^k\}\subseteq\mathcal{X}$$ 
such that for every $\mathbf{y}\in\{0,1\}^d$ and $n<d$, 
there exists
$L \in \cL$ so that $\ind\inbrace{x_{\mathbf{y}_{\le k}}\in L} = y_{k+1}$ for  
$0\le k\le n$.
We say that $\cL$ has an \emph{infinite Littlestone tree} if there is 
a Littlestone tree for $\cL$ of depth $d=\infty$.
\end{definition}
For instance, thresholds over $\N$ do not have an infinite Littlestone tree (yet, they have Littlestone trees of arbitrary length). On the other side, thresholds over the reals have an infinite Littlestone tree. Given this definition, we can show a separation between the online setting of \citet{bousquet2021theory} and the GA model.

First, we note that a language collection is not online learnable in the online setting of \citet{bousquet2021theory} if and only if it has an infinite Littlestone tree. We will show that 
there exists a countable language collection
 $\cL$ over a countable domain that has an infinite Littlestone
tree but it is identifiable in the limit with positive examples. 

\begin{example}[Infinite Littlestone Tree but Identifiable with Positive Examples]\label{ex:infinite-littlestone-tree-identifiable}
    Consider a countable domain $\cX$ and fix an enumeration $x_{\emptyset}, x_{0}, x_1, x_{00}, x_{01}, x_{10}, x_{11}, \dots$ of its elements. We use this enumeration to create the nodes of a Littlestone tree, \ie{} the root consists of $x_{\emptyset},$
    its left, right child is $x_0, x_1,$ respectively etc.
    Then, for each level $d \in \N$ and any path $y \in \{0,1\}^d$ we create a \emph{finite} language $L_{y}$ with index $y$, whose elements
    are the elements of $\cX$ that appear on the path
    with label 1.
    We consider the language collection $\cL = \{L_{y} \colon y \in \{0,1\}^d, d \in \N\}.$ This means that for any finite level $d < \infty$, we add all the $2^d$ languages in the collection (where for any path $y,$ $L_y$ contains the elements $x$ that are labeled with 1 in the path).
Hence, $\cL$ contains all the finite prefixes of the paths of the infinite Littlestone tree.
The language collection $\cL$ is countably infinite since the collection of all finite paths of the binary tree admits an enumeration.
By construction, this class induces an infinite Littlestone tree and hence is not online learnable in the game of \citet{bousquet2021theory}. However, since
it is a countable collection of finite languages,
it is identifiable in the limit with positive examples in the GA model (see also \cref{sec:exact-exp-rates-id-finite-languages} and \ref{sec:proof-exp-rates-id-finite-lang}).
\end{example}

\section{Borel--Cantelli Lemmas}\label{sec:borel-cantelli}
In this section, we present two well-known results due to Borel and Cantelli which are useful for our derivations.

\begin{lemma}[First Borel--Cantelli Lemma]\label{lem:first-borel-cantelli}
    Let $\inbrace{\cE_n}_{n \in \N}$ be a sequence of events. If
    \[
        \sum_{n \in \N} \Pr[\cE_n] < \infty \,,
    \]
    then the probability that infinitely many of them occur is 0, that is
    \[
        \Pr\insquare{\limsup_{n \rightarrow \infty} \cE_n} = 0\,.
    \]
\end{lemma}
The previous result has a \textit{partial} converse, which we state below.
\begin{lemma}[Second Borel--Cantelli Lemma]\label{lem:second-borel-cantelli}
     Let $\inbrace{\cE_n}_{n \in \N}$ be a sequence of independent events. If
    \[
        \sum_{n \in \N} \Pr[\cE_n] = \infty \,,
    \]
    then the probability that infinitely many of them occur is 1, that is
    \[
        \Pr\insquare{\limsup_{n \rightarrow \infty} \cE_n} = 1\,.
    \]
\end{lemma}
Notice that, unlike the first Borel--Cantelli lemma, the second
one requires that the events are \textit{independent}.

\end{document}